\documentclass{article}
\PassOptionsToPackage{numbers,sort&compress}{natbib}
 \usepackage[preprint]{neurips_2026}

\usepackage[utf8]{inputenc} 
\usepackage[T1]{fontenc}    
\usepackage{hyperref}       
\usepackage{url}            
\usepackage{booktabs}       
\usepackage{amsfonts}       
\usepackage{nicefrac}       
\usepackage{microtype}      
\usepackage{xcolor}         

\usepackage{booktabs}
\usepackage{color, colortbl}
\definecolor{greyC}{RGB}{180,180,180}
\definecolor{greyL}{RGB}{235,235,235}
\definecolor{shadecolor}{rgb}{0.92,0.92,0.92}
\definecolor{color1}{RGB}{255,190,122}
\definecolor{color2}{RGB}{142,207,201}
\definecolor{color3}{RGB}{190,184,220}
\definecolor{color4}{RGB}{130,176,210}
\usepackage{multirow}
\usepackage{multicol}
\usepackage{makecell}
\usepackage{graphicx}
\usepackage{subcaption}
\usepackage{amsmath}
\usepackage{algorithmic,algorithm}
\usepackage{caption}
\usepackage{soul}
\usepackage{threeparttable}
\usepackage{dsfont}
\usepackage{enumerate}
\usepackage{amsthm,amssymb}
\usepackage{wrapfig}
\usepackage{framed}
\usepackage{ulem}
\usepackage{rotating}
\usepackage{tablefootnote}
\usepackage{bbm}
\usepackage{pythonhighlight}

\newtheorem{remark}{Remark}
\newtheorem{theorem}{Theorem}
\newtheorem{assumption}{Assumption}
\newcommand{\colorb}[1]{\textcolor{blue}{#1}}

\newcommand{\colorr}[1]{\textcolor{red}{#1}}

\title{Task Specialization Fine-Tuning\\ for Contextual Reinforcement Learning}

\author{
\textbf{Jianan Zhou$^{1}$} \quad
\textbf{Jung-Hoon Cho}$^{2}$ \quad
\textbf{Tianyue Zhou$^{2}$}\quad
\textbf{Han Zheng$^{2}$}\\ 
\textbf{Jie Zhang$^{1}$} \quad
\textbf{Roy Dong$^{3}$} \quad
\textbf{Yining Ma}$^{2}\thanks{Corresponding author.}$ \quad
\textbf{Cathy Wu}$^{2}$ \vspace{2.5mm}\\
$^1$Nanyang Technological University\quad
$^2$MIT \quad
$^3$UIUC \vspace{2.5mm}\\
\texttt{jianan004@e.ntu.edu.sg},\ \texttt{zhangj@ntu.edu.sg},\ \texttt{roydong@illinois.edu} \\ 
\texttt{\{jhooncho, tianyuez, hanzheng, yiningma, cathywu\}@mit.edu}
}

\begin{document}

\maketitle

\begin{abstract}
Contextual Reinforcement Learning (CRL) seeks to generalize classical RL by maximizing task coverage across a context space of related tasks. While prior works often train from scratch and rely on either multi-task learning for a single policy or strategically training multiple policies, we advocate for a unified alternative: pretraining a single policy with good initial performance, followed by fine-tuning multiple policies for task specialization. This new paradigm, however, introduces unique challenges, such as heterogeneous marginal returns and sample inefficiency. This raises a critical research question: given a pretrained policy and a constrained budget, \emph{how much} fine-tuning should each task region receive to enable sample-efficient CRL? To this end, we propose \textit{Task Specialization Fine-Tuning (TSFT)}, an online framework that predicts fine-tuning performance with a simple parametric model and exactly solves the resulting discrete budget allocation problem via integer linear programming. Extensive experiments across diverse decision domains, including combinatorial optimization, continuous control, and LLM fine-tuning, demonstrate that TSFT significantly outperforms baselines in task coverage and approaches oracle performance. Our work charts a new direction for model-based CRL, aligning with the modern pretrain-finetune era.
\end{abstract}

\section{Introduction}
\label{intro}
Reinforcement learning (RL) has achieved remarkable success across various domains \cite{mnih2015human,silver2016mastering,fawzi2022discovering,mankowitz2023faster,zheng2026learning}, yet it remains fragile when applied to families of related tasks that differ only in a few key environment parameters \cite{benjaminscontextualize,degrave2022magnetic}. Contextual RL (CRL)~\cite{hallak2015contextual,modi2018markov,benjaminscontextualize} explicitly formalizes such families as Contextual Markov Decision Processes (CMDPs), where individual tasks are parameterized by vectors within a unified context space, with the objective of achieving broad coverage across that space.

Existing paradigms for solving CMDPs fall into three categories (see also  Table~\ref{tab:paradigm-comparison}): 1) \emph{independent training}~\cite{mnih2015human,rusu2016policy,yu2020meta,benjaminscontextualize}, which learns a separate policy for each task, straightforward yet expensive for covering high-dimensional context spaces;
2) \emph{multi-task training}~\cite{caruana1997multitask,wilson2007multi,teh2017distral,sodhani2021multi}, which learns a single unified policy across all tasks but is constrained by limited model capacity and susceptible to negative transfer;
3) \emph{multi-policy training}~\cite{cho2024model,ivanov2024personalized,ge2025learning,zhou2026structure,cho2026temporal}, which trains multiple policies, each on a single task, and relies on their generalization to cover the unseen tasks. 
Despite their differences, these paradigms share a common assumption: each policy is trained from scratch until convergence. Such an approach stands at odds with the pretrain-finetune convention of modern machine learning~\cite{bommasani2021opportunities,achiam2023gpt,yang2023foundation,taiga2023investigating,sun2023smart}, limiting the potential scalability of CRL in expansive, high-dimensional context space~\cite{benjaminscontextualize,cho2024model,zhou2026structure}.

In this paper, we advocate for a unified alternative: pretrain a \emph{single} policy with good initial performance across the context space, then fine-tune the trained policy with \textit{multiple} specialized variants under a constrained budget. A conceptual overview of the studied problem is illustrated in Fig.~\ref{fig:conceptual_overview}. This pretrain-finetune paradigm unifies multi-task and multi-policy training and holds potential advantages for three reasons. First, the pretrained policy may offer a good initialization across the context space. Second, fine-tuning enables the specialization of multiple policies, mitigating negative transfer without the computational expense of training from scratch. Finally, this approach integrates CRL into the scalable frameworks characteristic of modern machine learning.

\begin{table}[t]
\centering
\caption{Comparison of CRL paradigms along three orthogonal axes.}
\vspace{5pt}
\label{tab:paradigm-comparison}
\begin{small}
\begin{tabular}{l|ccc}
\toprule
Paradigm & \# Policy & Training Method & Training Mode \\
\midrule
Independent Training & Multiple $\mathcal{N}$ & Single-task & From scratch \\
Multi-Task Training  & Single   & Multi-task  & From scratch \\
Multi-Policy Training & Multiple $N\ll\mathcal{N}$ & Single-task & From scratch \\
\midrule
\textbf{TSFT (Ours)} & \textbf{Multiple $N\ll\mathcal{N}$} & \textbf{Multi-task} & \textbf{Pretrain + Fine-tune} \\
\bottomrule
\end{tabular}
\end{small}
\vspace{-3mm}
\end{table}

However, this paradigm introduces unique challenges: fine-tuning gains are \emph{heterogeneous} across tasks, where some regions improve rapidly while others quickly saturate or even degrade. Consequently, a uniform allocation of fine-tuning compute is often sample-inefficient, prompting a key research question: \emph{given a pretrained policy and a constrained budget, how much fine-tuning should each task region receive to enable sample-efficient CRL?} Notably, existing CRL literature has primarily investigated \emph{where} to train, selecting source tasks to maximize coverage under the implicit assumption of a fixed per-policy budget. The orthogonal question of \emph{how much} fine-tuning each region should receive remains largely open. To this end, this paper introduces a budget-aware perspective on CRL that complements existing source-task selection.

We introduce \emph{Task Specialization Fine-Tuning} (TSFT), an online framework for model-based budget allocation in CRL. Specifically, TSFT leverages a simple yet effective parametric model to predict how task performance, and consequently the induced coverage set, evolves as additional budget is assigned. Based on these predictions, each planning step reduces to a maximum coverage problem (MCP) variant over allocation vectors, which can be solved exactly through integer linear programming (ILP).
While ILP yields an optimal solution with respect to the model-based allocation, 
its theoretical optimality may be compromised by modeling errors. 
To mitigate this issue, we further embed ILP solving into an online framework with periodic model re-estimation. At each step, TSFT solves the current ILP, executes the derived allocation policy for a limited horizon, collects new data, and updates the model accordingly.
Empirical results demonstrate the promise of online budget allocation in CRL, enabling the coverage landscape to be substantially expanded under a constrained budget.

\textbf{Contributions}: 1) \textit{Conceptually}, we formulate a \emph{budget-constrained fine-tuning} problem in CRL, shifting the focus from \emph{where} to train to \emph{how much} fine-tuning to allocate under a pretrain-finetune paradigm; 2) {\textit{Methodologically}}, we introduce \textit{TSFT}, an online framework that optimizes model-based budget allocation via ILP, accompanied by a theoretical error analysis.
3) \textit{Experimentally}, we validate TSFT across combinatorial optimization, continuous control, and LLM fine-tuning, with up to 2--3$\times$ improvements in task coverage over simple strategies, and 2$\times$ gains in sample efficiency over multi-task training baselines, while performing comparably to an oracle policy under diverse settings.

\section{Related Work}
\textbf{Multi-Task Learning and Contextual RL.}
Multi-task learning improves generalization by jointly learning related tasks, but often suffers from negative transfer when tasks are insufficiently related \cite{caruana1997multitask, standley2020tasks, zhang2021survey}. 
In reinforcement learning, this challenge is further exacerbated by diverse task dynamics, sparse feedback, and unstable optimization. 
Contextual RL can be viewed as a structured form of multi-task RL, where a family of related tasks is parameterized by a context variable that affects the environment dynamics, rewards, or initial-state distributions \cite{hallak2015contextual,modi2018markov,benjaminscontextualize}.
A common approach in both multi-task RL and CRL is to train a single policy that generalizes across tasks, often through shared task representations, context-conditioned policies, or task-conditioned policy and value heads \cite{sodhani2021multi, yu2020meta, nauman2025bigger, grooten2026out}. 
Although recent work shows that larger and better-regularized value functions can improve generalization across diverse tasks \cite{nauman2025bigger}, single-policy approaches remain limited by model capacity and may suffer from negative transfer as task diversity increases.

Another line of work addresses this through transfer learning \cite{teh2017distral} or policy composition \cite{sun2022paco}, which alleviates task interference by allowing different components to specialize.
Recent multi-policy CRL methods also address contextual generalization by strategically selecting source tasks and relying on zero-shot transfer to cover the remaining context space \cite{cho2024model, zhou2026structure, cho2026temporal}. 
For example, MBTL models source task performance and transfer gaps to guide source task selection, while SD-MBTL further detects the underlying generalization structure of the CMDP and switches between suitable task selection strategies. 
These methods primarily address \emph{where} to train and typically assume that each selected policy is trained under a fixed or converged budget. Our work follows this multi-policy paradigm to mitigate negative transfer, while addressing their lack of budget allocation mechanism.

\textbf{RL Fine-Tuning.} 
Fine-tuning has become a common way to reuse experience in RL. Early related work focused on learning an initialization or task representation that can adapt with a small amount of new data, such as MAML and PEARL~\cite{finn2017model,rakelly2019efficient}. Another line learns reusable behaviors before downstream training, for example through unsupervised skill discovery~\cite{eysenbach2018diversity}. More recent work studies pretraining followed by fine-tuning more directly: multi-task pretraining with task-specific fine-tuning has been shown to be a strong and simple alternative to meta-RL~\cite{zhao2022effectiveness}, and multi-task pretraining on Atari variants has been shown to improve generalization to unseen variants even after substantial fine-tuning~\cite{taiga2023investigating}. Self-supervised multi-task pretraining has also been explored for sequential decision-making models, improving downstream fine-tuning efficiency across seen and unseen control tasks~\cite{sun2023smart}. Offline-to-online RL methods use prior datasets to accelerate later online improvement~\cite{nair2021awac,ball2023efficient,nakamoto2023cal}; and parameter-efficient methods such as L2M reduce forgetting when adapting pretrained decision-making models~\cite{schmied2023learning}. These works mainly address how to obtain a useful initialization, stabilize one fine-tuning run, or adapt a single agent to a new task. In contrast, TSFT treats fine-tuning progress itself as a planning signal: it models heterogeneous marginal gains from specialization and coordinates compute across a portfolio of fine-tuned policies.

\section{Problem Statement}
\textbf{Contextual MDP.}
Let $M=(S, A, P, R, \rho)$ denote a standard MDP, where $S$ is the state space, $A$ the action space, $P$ the transition dynamics, $R$ the reward function, and $\rho$ the initial state distribution. A contextual MDP, denoted by $\{M_x\}_{x\in \mathcal{X}}$, is a family of context-specific MDPs $M_x=(S, A, P_x, R_x, \rho_x)$ parameterized by a context vector $x$ drawn from a finite and bounded context set $\mathcal{X}$,
which can influence transition dynamics, reward function, and initial state distribution \cite{hallak2015contextual,modi2018markov,benjaminscontextualize}. Hereafter, we simplify notation by using $x$ to refer to a specific task.

\textbf{Task Specialization Fine-Tuning.} Given a pretrained policy with parameters $\theta$, it can be fine-tuned on a source task set $X_\mathcal{S} \subset \mathcal{X}$ with $k$ budget units to achieve task specialization ($\theta \to \theta_{X_\mathcal{S}}^k$) on a context subspace. 
The performance of the policy on a task $x \in \mathcal{X}$ is denoted by $J(\theta_{X_\mathcal{S}}^k, x)$.
A task $x$ is considered \emph{covered} by the policy if its performance satisfies a predefined threshold, $J(\theta_{X_\mathcal{S}}^k, x) \leq \epsilon$, assuming a minimization objective. 
The coverage set of the policy is defined as the union of all such covered tasks, $\mathcal{C}(\theta_{X_\mathcal{S}}^k)=\{x\in \mathcal{X}| J(\theta_{X_\mathcal{S}}^k, x) \leq \epsilon\}$. This convention is without loss of generality: for a maximization objective, the inequality is simply reversed.


Prior studies typically assume full training to convergence, which can be sample inefficient. This paper instead considers a more practical budget-constrained setting and formulates task specialization as a budget allocation problem.
Formally, given $N$ distinct source task sets $\{X_{\mathcal{S}_1}, X_{\mathcal{S}_2}, \dots, X_{\mathcal{S}_N}\}$ and a total budget of $K$ units, the goal is to distribute this constrained budget to maximize the global coverage of the context space (i.e., the union of converage sets). Starting from a pretrained policy $\theta$, we independently fine-tune it on each source task set $X_{\mathcal{S}_n}$, resulting in $N$ specialized policies.
The allocation objective is thus formulated as follows:
\begin{equation}
\begin{aligned}
\label{eq_obj}
    \max_{\mathcal{K}} \quad & G(\mathcal{K}),\\
     \text{s.t.} \quad & \|\mathcal{K}\|_1 \leq K,
\end{aligned}
\end{equation}
where $\mathcal{K} = [k_1, k_2, \dots, k_N]^T \in \mathbb{Z}_{\geq0}^N$ denotes the allocation vector, and $k_n$ denotes the number of budget units allocated to fine-tuning the $n$-th policy on its source task set $X_{\mathcal{S}_n}$. 
Given that specialized policies may exhibit overlapping coverage across the context space, we define the \emph{global coverage} as the measure of the union of individual coverage sets: 
\begin{equation}
\label{eq_g}
    G(\mathcal{K}) = \sum_{x \in \mathcal{X}} \mathbb{I} \left( x\in \bigcup_{n=1}^N\mathcal{C}(\theta_{X_{S_n}}^{k_n})\right)
    = \sum_{x \in \mathcal{X}} \mathbb{I} \left( \min_{n} J(\theta_{X_{\mathcal{S}_n}}^{k_n}, x) \leq \epsilon \right).
\end{equation}
In this paper, we define a \emph{budget unit} as the number of data samples consumed over training epochs or gradient steps, depending on the specific domain. This formulation induces a \emph{combinatorial} search space and leads to a \emph{non-convex} optimization problem, whose solution strategy is discussed next.


\begin{figure}[!t]
    \centering
    \includegraphics[width=0.99\textwidth]{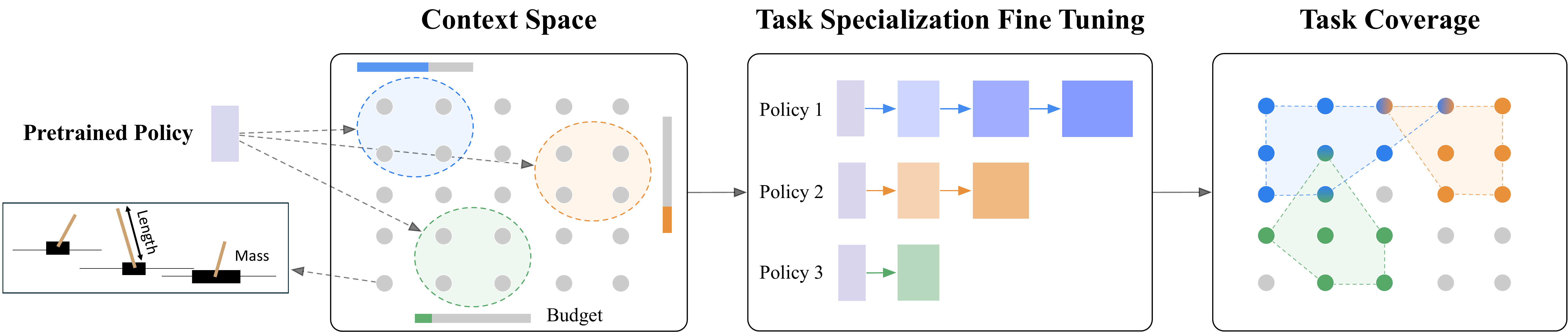}
    \vspace{1.5mm}
    \caption{Conceptual Overview of Task Specialization Fine-Tuning for Contextual RL. We illustrate the framework using a two-dimensional CartPole context space, where a pretrained policy is fine-tuned on three source task sets under different budget allocations to maximize final task coverage.}
    \label{fig:conceptual_overview}
    \vspace{-2mm}
\end{figure}

\section{Methodology}
In this section, we formalize the problem as a model-based allocation problem and solve it exactly via ILP.
Then, we introduce the overall framework TSFT, which embeds ILP solving into an online setting to enable effective task specialization in CRL.

\subsection{Problem Formulation}
Given a pretrained policy, a policy class, an RL training algorithm, a set of source task distributions $\{X_{\mathcal{S}_1},\ldots,X_{\mathcal{S}_N}\}$, and a total budget $K$, the goal is to decide \emph{how much} budget should be assigned to each policy. 
Each allocation decision assigns one fixed budget unit to a policy, e.g., by running the given RL algorithm for a fixed number of samples or epochs on the corresponding source task set. Although the total budget induces a finite horizon of $K$ allocation decisions, the final objective depends only on the resulting allocation vector, rather than on the order of decisions. We therefore directly optimize over feasible allocations subject to the budget constraint. An alternative MDP perspective is provided in Appendix~\ref{app_mdp}.

\textbf{ILP Formulation.}
We first consider the model-free setting in which the coverage set $\mathcal{C}(\theta_{X_{\mathcal{S}_n}}^{k})$ is known for every policy $n$ and allocation level $k$. Under this assumption, Eq.~\eqref{eq_obj} can be reformulated exactly as an ILP. Specifically, we introduce a binary decision variable $z_{n,k}$ to indicate whether policy $n$ is allocated $k$ budget units, and a binary decision variable $q_x$ to indicate whether task $x$ is covered by at least one policy. The resulting ILP formulation is defined as follows:
\begin{align}
    \max_{\mathbf{z},\mathbf{q}} \quad
    & \sum_{x\in\mathcal X}q_x,
    \label{eq:ilp_obj}\\
    \text{s.t.} \quad
    & \sum_{k=0}^{K} z_{n,k} = 1,
    && \forall n \in \{1,\ldots,N\},
    \label{eq:ilp_one_choice}\\
    & \sum_{n=1}^{N}\sum_{k=0}^{K} k\,z_{n,k}
    \leq K,
    \label{eq:ilp_budget}\\
    & q_x \leq
    \sum_{n=1}^{N}
    \sum_{k:\,x \in \mathcal{C}(\theta_{X_{S_n}}^{k})}
    z_{n,k},
    && \forall x \in \mathcal{X},
    \label{eq:ilp_coverage}\\
    & z_{n,k}\in\{0,1\}, \quad q_x\in\{0,1\}.
    \label{eq:ilp_binary}
\end{align}
Constraint~\eqref{eq:ilp_one_choice} selects exactly one allocation level for each policy. Constraint~\eqref{eq:ilp_budget} ensures that the total allocated budget does not exceed $K$. Constraint~\eqref{eq:ilp_coverage} enforces that $q_x$ can be set to one only if task $x$ is covered by at least one policy.
The optimal solution to Eq.~\eqref{eq_obj} is recovered as $\mathcal{K}^*=[k_1^*, k_2^*, \dots, k_N^*]$, where $k_n^\star = \sum_{k=0}^{K} k z_{n,k}^\star$.
Note that this formulation follows the standard ILP formulation of the maximum coverage problem (MCP), as our problem can be viewed as a multiple-choice budgeted MCP (see alternative formulations in Appendix~\ref{app_formulation}). In specific, each policy is associated with $K+1$ candidate coverage sets, among which exactly one must be selected. The formulation contains $N(K+1)+|\mathcal{X}|$ decision variables and $N+|\mathcal{X}|+1$ constraints. The ILP solver (e.g., CBC) operates on the precomputed coverage sets and leverages mature branch-and-bound and cutting-plane methods for efficient optimization. When a time limit is imposed, the solver returns the best feasible incumbent along with its optimality status. 
However, in practice, the coverage set $\mathcal{C}(\theta_{X_{\mathcal{S}_n}}^{k})$ is not known a priori. We therefore introduce a surrogate performance model to approximate the induced coverage sets.


\textbf{Surrogate Performance Model.} We model the evolution of each policy's task-wise performance and thereby implicitly estimate its coverage set. For each policy-task pair $(n,x)$, we fit a parametric performance model of the form:
\begin{equation}
    \mathcal{F}(y) = L \pm B e^{-d y},
\end{equation}
where $L$, $B \ge 0$, and $d > 0$ are learnable parameters, and $y$ denotes the consumed budget. The sign of $B$ controls the direction of the performance shift, capturing both improvement and degradation during fine-tuning. We fit $\mathcal{F}_{n}^{x}(y)$ via non-linear least squares (NLLS), which predicts the performance of policy $n$ on task $x$ after consuming a given amount of budget $y$. Empirically, this fitting process can be performed in parallel over thousands of tasks within seconds. We refer the reader to Appendix~\ref{app_model_study} for a comprehensive study of alternative models.

The true optimal solution (i.e., Oracle) could, in principle, be obtained by training each policy to completion using the full budget $K$, and then solving Eq.~(\ref{eq_obj}) exactly in a model-free manner. However, in practice, when computing $G$ via Eq.~(\ref{eq_g}), $J(\theta_{X_{\mathcal{S}_n}}^{k_n}, x)$ is unknown during the planning phase. We therefore replace it with our model prediction $\mathcal{F}_n^x(\cdot)$, yielding a surrogate objective $\widehat G$ that we optimize within our framework. 
To optimize $\widehat{G}$ via ILP, we replace $\mathcal{C}(\theta_{X_{\mathcal{S}_n}}^{k})$ in constraint~\eqref{eq:ilp_coverage} with the estimated coverage set $\widehat{\mathcal{C}}(\theta_{X_{\mathcal{S}_n}}^{k})=\left\{x\in\mathcal{X}| \mathcal{F}_n^x(k)\leq \epsilon \right\}$. The resulting ILP can be solved exactly with respect to $\widehat G$, although its solution may differ from that of $G$ because of modeling error.

\subsection{Overall Framework}
The above subsections detail how we solve the ILP to obtain the solution (i.e., allocation policy) at a single decision-making step (line 7). We now embed this procedure into our online framework TSFT for task specialization in contextual RL. 
The detailed algorithmic workflow is presented in Alg.~\ref{alg}.

Concretely, the framework begins with a \emph{warmup stage} (lines 1-4), in which the warmup budget $W$ is uniformly allocated across $N$ policies. Each policy $n$ is then trained on its corresponding source task set $X_{\mathcal{S}_n}$ according to this allocation. 
During training, we save checkpoints for every budget unit.
Subsequently, all checkpoints $\Theta = \{\theta_{X_{\mathcal{S}_n}}^{0}, \theta_{X_{\mathcal{S}_n}}^{1}, \dots, \theta_{X_{\mathcal{S}_n}}^{k_n}\}_{n=1}^N$ are evaluated on each task $x\in \mathcal{X}$ to collect training data $\mathcal{Z}$. In practice, this evaluation can be performed online if checkpoint storage becomes prohibitive.
This warmup phase ensures that all policies receive equal initial training and provides sufficient data to enable reliable model learning.

Next (lines 5-12), the framework iteratively allocates an \emph{execution budget} $E$ at each decision-making step until the remaining budget $b$ is exhausted. In specific, at each step, for each policy-task pair $(n, x)$, we first fit a model $\mathcal{F}_{n}^x$ to the available training data $\mathcal{Z}$. By calculating the surrogate objective with these models, we solve the model-based allocation via ILP to derive its solution $\widehat{\mathcal{K}}$. However, due to inherent modeling errors, we do not execute $\widehat{\mathcal{K}}$ to completion. Instead, we allocate the execution budget $E$ proportionally according to $\widehat{\mathcal{K}}$, which may prevent over-committing to an imperfect model and ensure balanced exploration. We then train the policies under this allocation and collect new data through evaluation. This newly acquired data $\mathcal{Z}'$ is appended to the existing training dataset $\mathcal{Z}$, which may progressively reduce modeling error in subsequent iterations. Furthermore, if the ILP solver terminates early without allocating the entire execution budget, any unused budget is carried forward and reconsidered in subsequent TSFT iterations.

To analyze TSFT, we introduce two oracle allocation policies. The \emph{Oracle} directly optimizes the true coverage objective over the entire budget and represents the best possible allocation in hindsight. The \emph{Oracle-Warmup} follows the same warmup stage as TSFT but optimizes the remaining allocation using the true coverage objective. We then present the following error bound and decomposition.

\begin{theorem}[Oracle Gap Decomposition]
\label{thm:oracle_gap}
Let $G(\mathcal{K})$ denote the global coverage achieved by $\mathcal{K}$, $\mathcal{K}^*$ denote the Oracle allocation, $\mathcal{K}_{\mathrm{OW}}^*$ denote the Oracle-Warmup allocation, and $\mathcal{K}_{\mathrm{TSFT}}$ denote the allocation returned by TSFT. Under Assumptions~\ref{assum:bounded_error} and~\ref{assum:planning_accuracy}, the optimality gap of TSFT satisfies
\begin{equation}
    G(\mathcal{K}^*) - G(\mathcal{K}_{\mathrm{TSFT}})
    \le
    G(\mathcal{K}^*) - G(\mathcal{K}_{\mathrm{OW}}^*)
    +
    2\delta_m
    +
    \eta_{\mathrm{alg}}.
\end{equation}
\end{theorem}
Theorem~\ref{thm:oracle_gap} shows that the gap between TSFT and the Oracle decomposes into three terms: the warmup error $G(\mathcal{K}^*) - G(\mathcal{K}_{\mathrm{OW}}^*)$, the surrogate-model error $2\delta_m$, and the algorithmic planning error $\eta_{\mathrm{alg}}$. 
If the post-warmup state lies on an optimal trajectory, the warmup error becomes zero.
Moreover, this provides a conditional post-warmup guarantee with respect to a fixed surrogate objective. 
For a one-shot exact full-horizon planner that optimizes this fixed surrogate and executes the returned allocation exactly, $\eta_{\mathrm{alg}}=0$. 
We refer to Appendix~\ref{app_theory} for proofs and further discussion.

In summary, our proposed online framework solves a model-based allocation via ILP at each decision-making step, while allowing for periodic model re-estimation to mitigate error accumulation. This enables adaptive and sample-efficient policy training, facilitating effective task specialization in CRL.

\begin{algorithm}[!t]
    \caption{Task Specialization Fine-Tuning (TSFT)}
    \label{alg}
    \textbf{Input}: Context space $\mathcal{X}$, policy count $N$, total budget $K$, warmup budget $W$, execution budget $E$, Source task sets $\{X_{\mathcal{S}_n}\}_{n=1}^N$;\\
    \textbf{Output}: $N$ trained policies;
    \begin{algorithmic}[1] 
        \STATE \textbf{Initialize:} $\mathcal{K} = \{k_n\}_{n=1}^N$, $k_n = \frac{W}{N}$ 
        \STATE $\Theta \gets$ Train $N$ policies on $\{X_{\mathcal{S}_n}\}_{n=1}^N$ following $\mathcal{K}$ 
        \STATE $\mathcal{Z} \gets$ Evaluate $\Theta$ on each task $x \in \mathcal{X}$ to construct dataset
        \STATE Remaining budget $b \gets K - W$
        \WHILE{$b > 0$}
            \STATE $\forall n, x: \mathcal{F}_{n}^x \gets$ Fit models using dataset $\mathcal{Z}$
            \STATE $\widehat{\mathcal{K}} \gets$ Solve ILP with models $\{\mathcal{F}_{n}^x\}_{n,x}$
            \STATE $\Theta \gets$ Distribute budget $\min(b, E)$ based on $\widehat{\mathcal{K}}$; Train $\mathcal{N}$ policies
            \STATE $\mathcal{Z}' \gets$ Evaluate $\Theta$ on each task $x \in \mathcal{X}$
            \STATE Update dataset $\mathcal{Z} \gets \mathcal{Z} \cup \mathcal{Z}'$
            \STATE Update remaining budget $b = b - \min(b, E)$
        \ENDWHILE
    \end{algorithmic}
\end{algorithm}

\section{Experiment}
\label{exp}
In this section, we empirically validate our framework across continuous control, combinatorial optimization, and LLM fine-tuning. The source code will be publicly released upon publication.

\textbf{Baseline.} 
1) \emph{Oracle}: We first train each policy to completion (i.e., using the maximal budget $K$), and then solve the ILP defined in Eqs.~\eqref{eq:ilp_obj}-\eqref{eq:ilp_binary}.
This serves as a standard oracle reference, representing an upper bound on achievable performance.
2) \emph{Oracle-Warmup}: Similar to the Oracle, but isolates the warmup phase by solving the ILP after warmup. Specifically, it treats the post-warmup allocation $\mathcal K^0$ as the initial state and optimizes the true coverage objective over $\Omega_W =\{\mathcal K\in\mathbb Z_{\ge0}^N:\mathcal K\succeq \mathcal K^0,\;\|\mathcal K\|_1\le K\}$. Therefore, Oracle-Warmup shares the same initial decision-making state as our proposed framework, serving as an additional oracle reference for validating model error.
3) \emph{Pretrained}: The given pretrained policy without further adaptation. Depending on the decision domain, the pretrained policy may be trained outside the context space.
4) \emph{MTL}: This approach fine-tunes the pretrained policy over the entire context space in a multi-task learning manner, using the same computation budget as the other methods.
5) \emph{Random}: This approach randomly allocates the compute budget across the $N$ policies.
6) \emph{Uniform}: This approach evenly allocates the compute budget across the $N$ policies.
7) \emph{Adaptive}: At each decision-making step, this approach adaptively allocates the decision budget $D$ according to a probability distribution derived from the observed improvement rate of each policy's coverage set.
8) \emph{LinUCB} \cite{li2010contextual,chu2011contextual}: This approach formulates the compute allocation as a contextual multi-armed bandit problem. At each decision step, it utilizes the Linear Upper Confidence Bound (LinUCB) algorithm to select a policy for training. The context vector encodes the normalized training progress and coverage momentum, while the reward is defined as the monotonic improvement in the policy's coverage set to account for temporary performance dips.
We refer to Appendix~\ref{app_baseline} for additional details on the baselines.

We report average results in the main paper (see Appendix \ref{app_exp_result} for full results). During inference, we evaluate all $N$ policies and select the best one for each task, following the convention in CRL.

\subsection{Combinatorial Optimization}
\textbf{Environment and Context Space.} We consider the capacitated vehicle routing problem (CVRP) and its variant with time window constraints (CVRPTW).
1) \emph{CVRP}: We construct a 2D context space defined by customer node distribution and vehicle capacity. The distribution parameter ranges from [0.01, 0.25] with a step size of 0.01, while the vehicle capacity ranges from [10, 400] with a step size of 10. This results in a total of 1{,}000 tasks in the context space.
2) \emph{CVRPTW}: We construct a 2D context space defined by time window tightness and vehicle capacity. The time window parameter ranges from [0.04, 1.00] with a step size of 0.04, while the vehicle capacity ranges from [10, 400] with a step size of 100. This also yields a total of 1{,}000 tasks in the context space. Each source task set corresponds to a square region in the context space, comprising 49 tasks centered around a reference task.
More details are presented in Appendix \ref{app_context_space}.

\begin{table*}[!t]
  \vskip -0.1in
  \caption{Performance Comparison in Combinatorial Optimization and Continuous Control.}
  \label{table_nco_control}
  \begin{center}
  \begin{small}
  \renewcommand\arraystretch{1.0}  
  \resizebox{0.99\textwidth}{!}{ 
  \begin{tabular}{l|ccc|ccc|ccc|ccc}
    \toprule
     & \multicolumn{3}{c|}{\textbf{CVRP}} & \multicolumn{3}{c|}{\textbf{CVRPTW}} & \multicolumn{3}{c|}{\textbf{CartPole}} & \multicolumn{3}{c}{\textbf{Ant}} \\
     $N/K$ & 3/100 & 3/150 & 4/150 & 3/100 & 3/150 & 4/150 & 3/50 & 3/100 & 4/100 & 3/50 & 3/100 & 4/100 \\
    \midrule
    Oracle & 26.7\% & 29.3\% & 32.5\% & 49.4\% & 50.0\% & 50.0\% & 89.2\% & 89.7\% & 90.5\% & 100.0\% & 100.0\% & 100.0\% \\
    Oracle-Warmup & 21.4\% & 24.5\% & 28.0\% & 46.6\% & 47.5\% & 47.5\% & 86.8\% & 87.5\% & 90.4\% & 100.0\% & 100.0\% & 100.0\% \\
    \midrule
    Pretrained & 6.8\% & 6.8\% & 6.8\% & 40.9\% & 40.9\% & 40.9\% & 66.6\% & 66.6\% & 66.6\% & 31.1\% & 31.1\% & 31.1\% \\
    \midrule
    MTL & 9.5\% & 11.0\% & 11.0\% & 44.8\% & 46.0\% & 46.0\% & 71.9\% & 74.5\% & 74.5\% & 31.1\% & 31.1\% & 31.1\% \\
    Random & 10.1\% & 11.0\% & 12.9\% & 42.9\% & 43.2\% & 43.3\% & 75.4\% & 75.0\% & 82.4\% & 15.8\% & 42.5\% & 45.7\% \\
    Uniform & 7.9\% & 15.4\% & 12.4\% & 42.8\% & 41.8\% & 42.0\% & 72.8\% & 65.4\% & 82.8\% & 25.7\% & 53.0\% & 64.7\% \\
    Adaptive & 10.4\% & 16.1\% & 14.5\% & 43.2\% & 44.0\% & 44.0\% & 75.1\% & 76.9\% & 87.3\% & 54.1\% & 88.0\% & 91.9\% \\
    LinUCB & 12.6\% & 20.5\% & 20.8\% & 42.8\% & 44.3\% & 44.3\% & 77.7\% & 82.6\% & 82.6\% & 64.3\% & 88.2\% & 76.9\% \\
    \midrule
    TSFT & \textbf{18.7\%} & \textbf{23.0\%} & \textbf{26.4\%} & \textbf{45.8\%} & \textbf{46.2\%} & \textbf{46.9\%} & \textbf{82.2\%} & \textbf{84.0\%} & \textbf{87.6\%} & \textbf{88.4\%} & \textbf{97.7\%} & \textbf{99.5\%} \\
    \bottomrule
  \end{tabular}}
  \end{small}
  \end{center}
  \vskip -0.2in
\end{table*}

\textbf{Setup.} We adopt POMO \cite{kwon2020pomo} as the policy network, consisting of 1.27M parameters, and train it using the REINFORCE algorithm \cite{williams1992simple}, following the training configurations in \cite{kwon2020pomo} (see Appendix \ref{app_training}). In this domain, one budget unit corresponds to 100 training epochs, with each epoch processing 10{,}000 data samples.
Starting from a policy pretrained over the entire context space for 5{,}000 epochs, we perform task specialization via multi-task fine-tuning on the corresponding source task sets.
Performance is evaluated in terms of optimality gap with respect to HGS \cite{vidal2022hybrid}. 
We set the performance threshold $\epsilon$ to 1.25\% for CVRP and 3\% for CVRPTW.
We consider a range of budget configurations. For example, 3/100 denotes training $N=3$ policies with a total budget of $K=100$ units, as shown in Table \ref{table_nco_control}.
In our framework, the warmup budget is set to $W = 5\times N$ units. The execution budget is initially set to $E = 5$ units, and is increased to $10$ after a cumulative budget of $10\times N$ has been consumed. This reflects increased confidence in the model as more data becomes available and reduces the frequency of ILP solving. 

\textbf{Result.} We report the best coverage rate achieved by all approaches on CVRP and CVRPTW in Table~\ref{table_nco_control}, where the coverage rate is defined as the global coverage divided by the total number of tasks in the context space. TSFT achieves the best non-oracle performance across all six settings. On CVRP, TSFT brings substantial improvements over the strongest baseline, especially as the budget or number of policies increases, and closely approaches Oracle-Warmup. On CVRPTW, the gains are more moderate because several baselines, particularly MTL, already achieve strong coverage, possibly due to reduced task interference in the considered context space. Nevertheless, TSFT still consistently improves over all non-oracle baselines and remains close to Oracle-Warmup.

\subsection{Continuous Control}
\textbf{Environment and Context Space.}
We evaluate on three continuous control suites:
1) \emph{CartPole 3D} from CARL~\cite{benjaminscontextualize}, where the context varies pole length, cart mass, and pole mass. Each axis is discretized into 10 values evenly spaced over $[0.1, 10]\times$ the CARL defaults, yielding a $10^3$-task grid.
2) \emph{Ant 2D} from CARL~\cite{benjaminscontextualize}, where the context varies gravity and friction over $[0.2, 2.0]\times$ their defaults $(g_0, \mu_0) = (9.8, 1.0)$ on a $25\!\times\!40$-task grid.
3) \emph{Meta-World MT50}~\cite{yu2020meta}, a benchmark of $50$ robotic manipulation tasks. We refer to Appendix \ref{app_context_space} for more details.

\textbf{Setup for CartPole and Ant.} 
We use PPO~\cite{schulman2017proximal} as the base RL algorithm, with an MLP policy/value network of hidden sizes $[64,64]$ for CartPole and $[256,256]$ for Ant. The pretrained policy $\theta$ is obtained by multi-task PPO on the full grid for $5$M (CartPole) and $1$M (Ant) environment steps.
One budget unit corresponds to 10{,}000 steps for CartPole and 40{,}000 steps for Ant.
Source task sets correspond to $N$ axis-aligned regions of the grid, each containing $3^3=27$ contexts for CartPole and $5\!\times\!4=20$ contexts for Ant. 
We set the performance threshold $\epsilon$ to 500 for CartPole and 6.9 for Ant.
Note that, in contrast to combinatorial optimization, larger values indicate better performance in continuous control. We consider budget settings with $N\in{3,4}$ and $K\in{50,100}$ for both environments. For TSFT, we set $E=50$ for CartPole and $E=10$ for Ant, while keeping the other hyperparameters the same as those used for combinatorial optimization. We report the average performance across seeds in the main paper, with per-seed results provided in Appendix~\ref{app_exp_result}.

\textbf{Setup for Meta-World.} 
We use MOORE~\cite{hendawy2023multi} as the base multi-task RL algorithm. Specifically, we pretrain MOORE on the full MT50 benchmark for 50M steps to obtain the pretrained policy. The remaining 50M steps are used for task specialization, where each budget unit corresponds to 1M steps. To construct the source task sets, we extract the task-specific expert weights learned by pretrained MOORE and apply K-means clustering to partition the 50 tasks into 10 groups, each containing 5 tasks.
We set the performance threshold to $\epsilon=1.0$, corresponding to a success rate of 100\% over five evaluation episodes. We consider the budget setting with $N=10$ policies and $K=50$ units. For TSFT, we set the warm-up budget to $W=3\times N$. 
The remaining hyperparameters are the same as those used for combinatorial optimization. To evaluate robustness, we additionally report results on a smaller setting with 5 random seeds in Appendix~\ref{app_exp_result}.

\begin{wraptable}{r}{0.5\textwidth}
  \caption{Results for Meta-World and LLM-FT.}
  \label{table_metaworld_llm}
  \begin{center}
  \vspace{-5pt}
  \renewcommand\arraystretch{1.0}  
  \resizebox{0.5\textwidth}{!}{ 
  \begin{tabular}{l|c|c}
  \toprule
     & Meta-World (10/50) & LLM-FT (4/100) \\
    \midrule
    Oracle & 86.0\% & 77.8\% \\
    Oracle-Warmup & 84.0\% & 77.8\% \\
    \midrule
    Pretrained & 58.0\% & 0.0\% \\
    \midrule
    MOORE & 60.0\% & / \\
    Random & 71.2\% & 26.7\% \\
    Uniform & 70.0\% & 22.2\% \\
    Adaptive & 71.2\% & 28.9\% \\
    LinUCB & 68.0\% & 33.3\% \\
    \midrule
    TSFT & \textbf{74.0\%} & \textbf{55.6\%} \\
    \bottomrule
  \end{tabular}}
  \end{center}
  \vskip -0.1in
\end{wraptable}

\textbf{Result.} Table~\ref{table_nco_control} reports global coverage on CartPole and Ant. Gains over adaptive and bandit baselines are moderate, but TSFT consistently improves coverage and narrows the gap to Oracle-Warmup as the budget grows. MTL can be unstable here, e.g., on Ant, additional training does not necessarily yield higher task coverage. In contrast, Meta-World results (Table~\ref{table_metaworld_llm}) show strong MTL performance, yet TSFT still achieves better coverage.
Specifically, MOORE achieves an average success rate of 0.68 with 60\% task coverage, whereas TSFT improves these to 0.72 and 74\%, respectively. 
We also note that task grouping plays an important role, yet it is not deliberately optimized in this paper. Further improving the task grouping in Meta-World could yield even better performance and higher sample efficiency.
Overall, these results show that model-based budget allocation is effective for continuous-control CRL, especially in higher-dimensional settings where naive multi-task fine-tuning is sample-inefficient.

\subsection{LLM Fine-Tuning}
\textbf{Environment and Context Space.}
We evaluate TSFT on reinforcement fine-tuning (RFT) of LLMs. The discrete 1D context space consists of 9 reasoning tasks (benchmarks). The in-distribution (ID) set comprises the test splits of our four fine-tuning corpora: DAPO-17K~\cite{yu2025dapo}, MATH-500~\cite{hendrycks2021measuring,lightman2023lets}, GSM8K~\cite{cobbe2021training}, and CodeContests+~\cite{wang2025codecontests}. The out-of-distribution (OOD) set comprises five held-out benchmarks: AIME 2024/2025, Minerva Math~\cite{lewkowycz2022solving}, MBPP~\cite{austin2021program}, and BigCodeBench~\cite{zhuo2024bigcodebench}.

\textbf{Setup.}
We adopt Qwen3-4B-Base~\cite{yang2025qwen3} as the pretrained policy and apply GRPO~\cite{shao2024deepseekmath} with verifiable rule-based rewards. We construct $N=4$ specialized policies, each fine-tuned on one of \{DAPO-17K, MATH, GSM8K, CodeContests+\}, under a total budget of $K=100$, where one budget unit corresponds to a single GRPO update step. Per Eq.~(\ref{eq_obj}), each task is dispatched to all $N$ policies and the best pass@1 (greedy decoding) is taken. A task is covered if pass@1 exceeds a per-benchmark threshold. For TSFT, we use the same hyperparameters as in the combinatorial optimization setting. Full training, decoding, and threshold details are presented in Appendices~\ref{app_training} and~\ref{app_context_space}.

\textbf{Result.}
Table \ref{table_metaworld_llm} reports the global coverage on the LLM suite. TSFT consistently outperforms all non-oracle baselines and approach Oracle-Warmup, whereas
heuristic allocation baselines recover only part of the specialization benefit. 
These results suggest that model-based budget allocation can scale to LLM RFT with billions of parameters.

\begin{figure*}[!t]
    \centering
    \centerline{
    \raisebox{1.2mm}{\includegraphics[width=0.29\columnwidth]{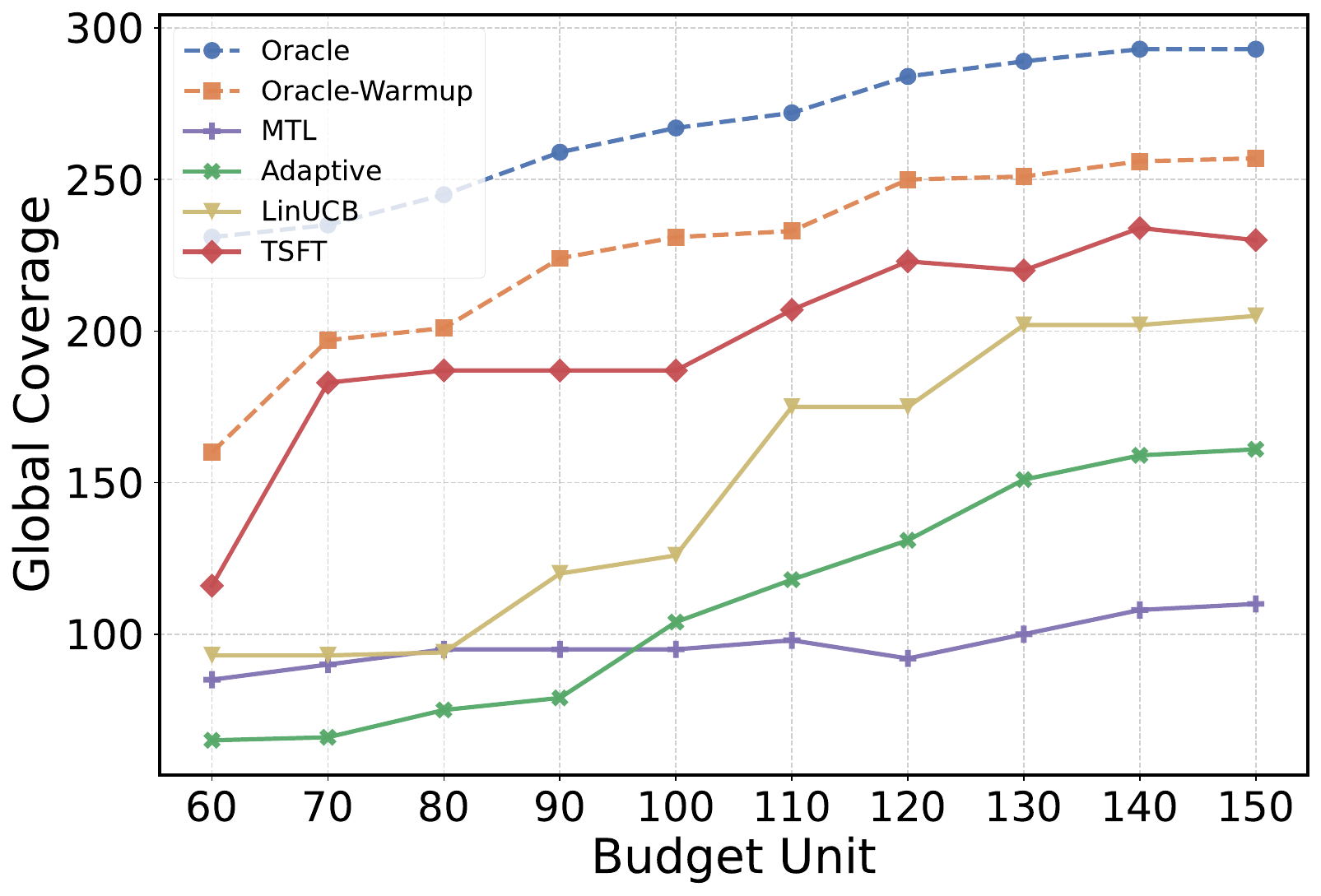}}
    \includegraphics[width=0.22\columnwidth]{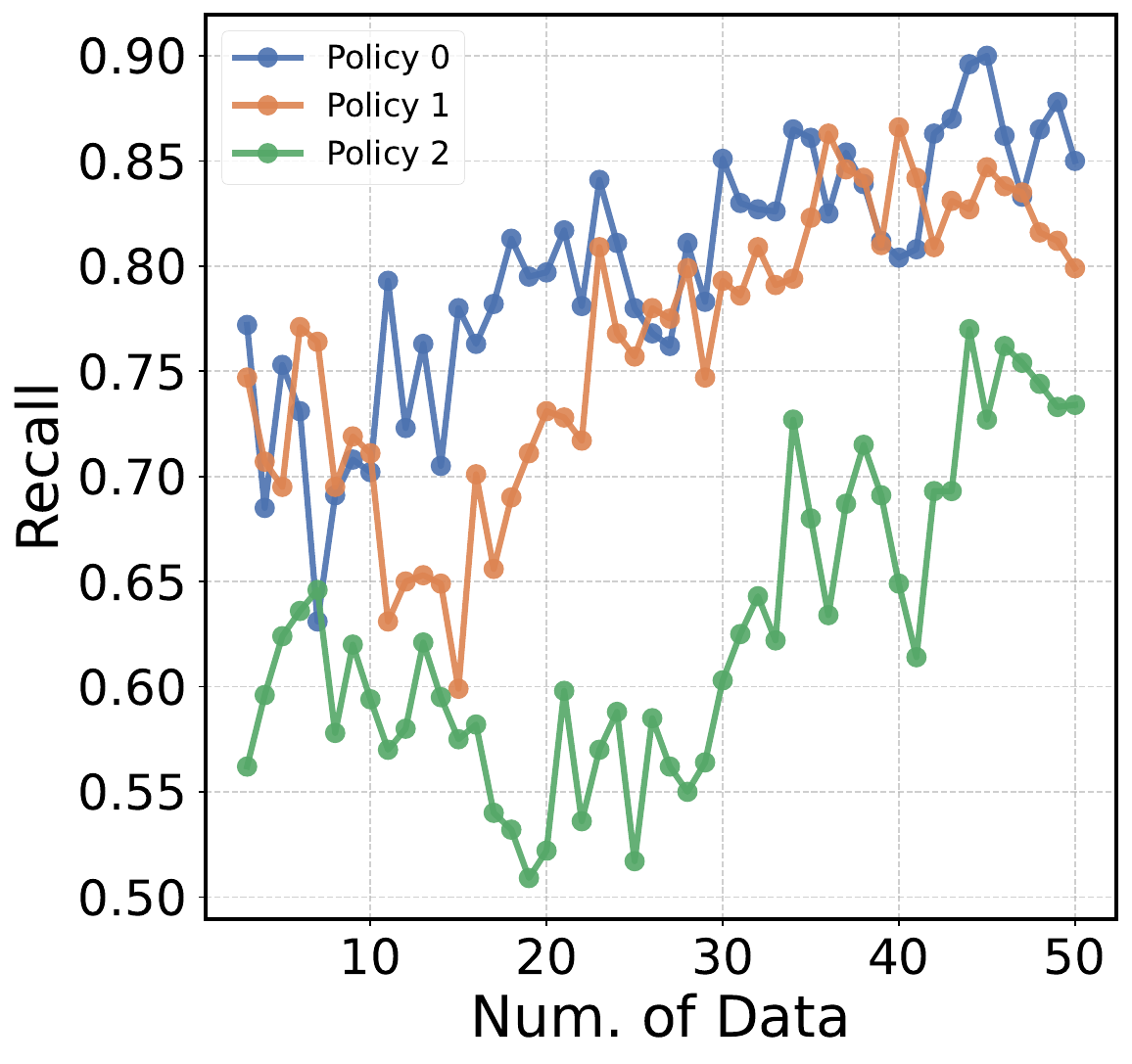}
    \raisebox{2mm}{\includegraphics[width=0.24\columnwidth]{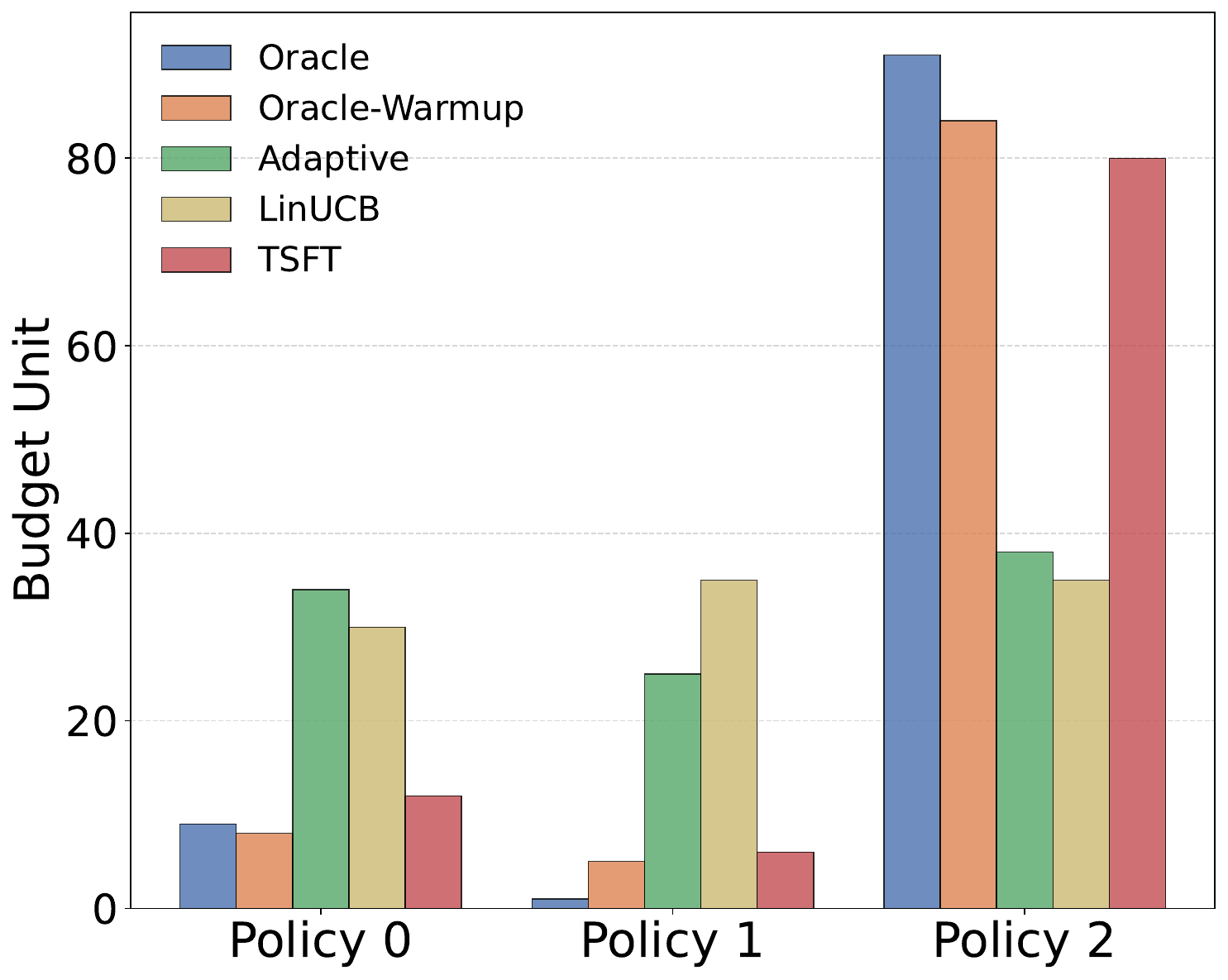}}
    \includegraphics[width=0.22\columnwidth]{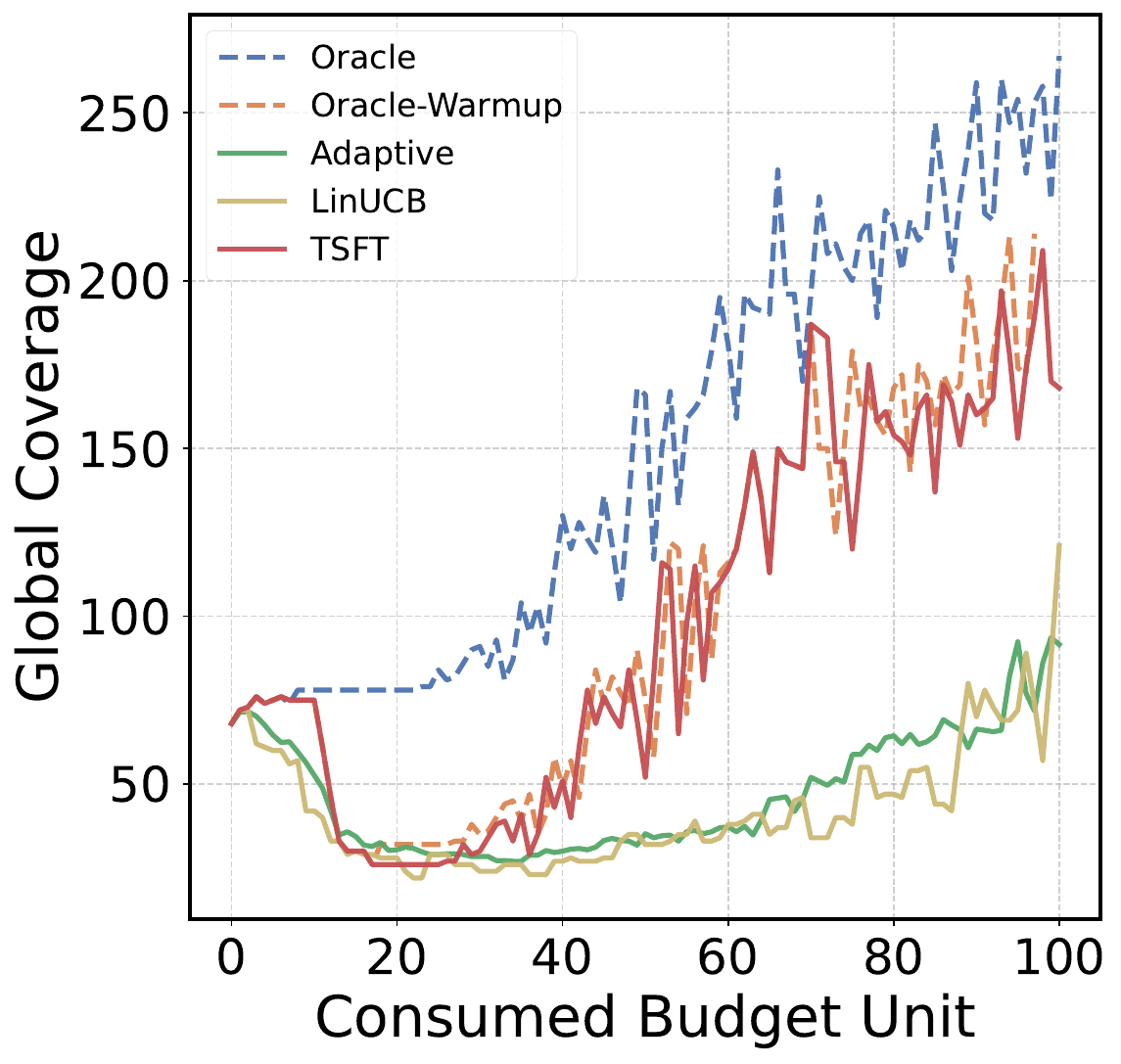}
    }
    \caption{\emph{From left to right:} Performance comparison under diverse budget settings; Average recall for predicting coverage sets over the future 50 budget units, given varying amounts of data for model fitting; An example run illustrating the allocation policies and the global coverage trajectories over the training process across different methods for the case with $N=3$ and $K=100$.
    }
    \label{analysis}
    \vskip -0.1in
\end{figure*}

\subsection{Analysis}
We further present detailed analyses using CVRP as a standard testbed. See Appendix~\ref{app_exp_result} for additional empirical results and analyses (e.g., full results, performance model studies, sensitivity analyses, low-cost evaluation, computational cost, and visualizations).

\textbf{Diverse Budget.} We provide a detailed comparison of performance across varying budget settings for the case of $N=3$ in Fig.~\ref{analysis}.
The results consistently demonstrate the superiority of TSFT in terms of global coverage across both low- and high-budget regimes. In contrast, LinUCB performs well only in high-budget settings, as it requires sufficient exploration to identify which bandit (policy) has the potential for improvement through trial-and-error, highlighting its relative sample inefficiency. 

\textbf{Computational Cost.}
TSFT introduces only modest overhead, as both surrogate performance modeling and ILP optimization are efficient. Thousands of surrogate models can be fitted in parallel within seconds, and an ILP with $N=10$ and $K=150$ can be solved to optimality in under one second. This cost is minor compared with policy training, which takes about one week for CVRP.

\textbf{Model Error.} For each policy-task pair $(n, x)$, TSFT collects one data point per budget unit, which is used to fit the predictive model $\mathcal{F}_x^n$ at each decision-making step. We report the recall rate of the coverage set estimated by the fitted model as more data becomes available. 
The results in Fig.~\ref{analysis} demonstrate the effectiveness of our simple parametric model and justify the design of online model re-estimation with newly collected data. More detailed results are provided in Appendix~\ref{app_model_study}.

\textbf{Sample Efficiency.} We compare TSFT with the MTL baseline, which performs further training over the entire context space without task specialization. As shown in Fig.~\ref{app_hypara}, MTL typically requires roughly $2\times$ more samples to achieve global coverage comparable to that of our method.

\textbf{Allocation Policy.}
In the right two panels of Fig.~\ref{analysis}, we present the allocation policies derived by each approach and the corresponding global coverage in the CVRP setting with $N=3$ and $K=100$. The coverage set may initially shrink due to the distribution shift between pretraining and fine-tuning. 

\textbf{Comparison with MTL.} Since TSFT fine-tunes multiple policies during the specialization stage, we provide a comprehensive comparison with the MTL baseline on CVRP under matched training compute. Let $\Phi$ denote the number of parameters in the pretrained model, and let $\hat{K}=200$ denote the total budget (pretraining + fine-tuning). We compare TSFT against two MTL settings: (1) MTL with a total compute budget of $\Phi \times \hat{K}$, corresponding to the standard setup; and (2) MTL with a matched compute budget of $5\Phi \times \frac{\hat{K}}{5}$, which uses a model with five times the parameters but one-fifth of the budget so that the overall number of parameter updates is comparable to TSFT.
These two MTL baselines achieve only 11.0\% and 7.8\% task coverage, respectively, despite being trained on a larger set of source tasks (i.e., the entire context space) than TSFT (see Fig.~\ref{app_fig_context_space}).

\section{Conclusion}
\label{conclusion}
This paper introduces Task Specialization Fine-Tuning (TSFT), a framework that intelligently guides task specialization in contextual reinforcement learning (CRL) by explicitly optimizing how much budget should be allocated to each task region under a constrained budget. Specifically, TSFT formulates the process as a model-based allocation problem using a simple parametric model, and solves it exactly via integer linear programming (ILP). We further embed the process into an online framework with periodic model re-estimation to mitigate modeling errors. Experiments on combinatorial optimization, continuous control, and large language model fine-tuning demonstrate that TSFT significantly improves task coverage, highlighting the promise of intelligent task specialization as a principled mechanism for sample-efficient policy training in complex CRL environments. 

The limitations of this work are fourfold.
First, the surrogate objective used for model-based allocation may be misaligned with the original objective. As a result, solutions that are optimal under the surrogate objective may not achieve strong performance under the original objective.
Second, exhaustive evaluation can become computationally prohibitive as the context space grows or the number of specialized policies increases, potentially limiting the scalability of TSFT.
Third, we assume a fixed task grouping and do not optimize the partitioning of the context space, which may significantly affect the final coverage, particularly in challenging domains such as Meta-World.
Finally, the exponential performance model is not universally applicable and may be inadequate for domains with strongly non-monotonic fine-tuning dynamics.
Developing better-aligned optimization objectives, more efficient evaluation strategies, joint optimization of task grouping and budget allocation, and more expressive performance models are promising directions for future work.


{
\small
\bibliographystyle{unsrtnat}
\bibliography{main}
}

\newpage
\appendix

\vbox{
\hsize\textwidth
\hrule height 4pt
\vskip 0.25in%
\vskip -\parskip%
\centering
{\LARGE\bf Appendix \par}
\vskip 0.29in
\vskip -\parskip
\hrule height 1pt
\vskip 0.3in
}

\section{Discussion}
\textbf{Relationship to Existing Literature.} We provide a brief discussion to better contextualize the studied problem and the proposed framework. 1) At the formulation level, Eqs.~\eqref{eq:ilp_obj}-\eqref{eq:ilp_binary} constitute a multiple-choice budgeted MCP \cite{khuller1999budgeted}. Specifically, the \(K+1\) allocation levels of each policy form a group of candidate sets: selecting level \(k\) consumes \(k\) budget units and induces a corresponding coverage set. The objective is to select exactly one level for each policy so as to maximize the cardinality of the union of the selected coverage sets, subject to the total budget constraint.
This structure resembles the multiple-choice knapsack problem (KP) \cite{sinha1979multiple}, as both problems involve selecting one item from each group under a shared budget.
The key distinction is that KP assumes additive item values, whereas our coverage objective is non-additive: overlapping coverage sets create dependencies among policies. If the policies' coverage sets were disjoint, or, more generally, their contributions were additive, the formulation would reduce to a multiple-choice KP.
2) At the learning-optimization interface, TSFT follows the predict-then-optimize paradigm \cite{elmachtoub2022smart}: it first predicts task-wise performance using parametric models, converts these predictions into estimated coverage sets, and then optimizes the resulting surrogate problem.
In contrast to decision-focused learning~\cite{mandi2024decision}, our performance models are trained with prediction-level losses rather than an end-to-end objective based on allocation quality or coverage regret. This distinction is particularly relevant for the thresholded coverage objective: prediction errors near the performance threshold $\epsilon$, or on tasks that distinguish between competing allocations, can have a much larger impact than errors elsewhere. Decision-focused learning may therefore improve objective alignment, although differentiating through the discrete thresholding operation and the ILP remains challenging.
3) At the framework level, TSFT resembles the feedback structure of model predictive control (MPC)~\cite{garcia1989model}: it estimates a surrogate model, optimizes an allocation, executes only a limited step, observes the resulting performance, and replans. In this view, the allocation vector plays the role of the planning state, while budget assignments serve as controls. Unlike conventional MPC, which typically controls a dynamical system with continuous states and actions, TSFT addresses a discrete allocation problem over policies and tasks. Accordingly, feedback is primarily used to correct uncertainty in the learned fine-tuning response, rather than to compensate for disturbances in physical dynamics.

\textbf{Assumption.} 1) \emph{Quality of the pretrained policy:} TSFT assumes a pretrained policy that provides a meaningful foundation for adaptation. This assumption is consistent with the intended practical setting, where fine-tuning is used to adapt an already capable pretrained model, such as a foundation model in many domains. If the pretrained policy lacks capabilities relevant to the target contexts, adaptation becomes closer to learning each target task from scratch and may require substantially more data and computation than the fixed fine-tuning budget allows. Such scenarios therefore fall outside the intended scope of TSFT. Importantly, however, this assumption does not require the pretrained policy to exhibit uniformly strong performance across all target contexts. As shown in Table~\ref{table_nco_control}, TSFT can substantially expand limited target-context coverage (e.g., from 6.8\% to 26.4\% on CVRP), provided that the pretrained policy offers a sufficiently meaningful basis for adaptation.
2) \emph{Efficient evaluation:} TSFT periodically evaluates policies across the context space to collect the data required for fitting the performance prediction models. Consequently, the framework assumes that these evaluations are sufficiently efficient so that they do not introduce prohibitive computational overhead during online optimization. This limitation can be mitigated by using cheaper evaluation strategies for model fitting, as demonstrated in Appendix~\ref{app:cheap_eval}. More advanced methods, such as approximating the evaluation results through selective evaluation, constitute an interesting direction for future research.
3) \emph{Structure of the context space:} By definition, a CMDP comprises a family of related, context-specific MDPs parameterized by a context vector drawn from a context space. We primarily consider well-structured context spaces, following conventions established in prior CRL work~\citep{benjaminscontextualize,cho2024model,zhou2026structure}. In contrast to these studies, we also consider a discrete context space in the LLM domain, where nine heterogeneous reasoning and coding benchmarks serve as contexts. In this setting, relationships among contexts are less explicit and less smoothly varying, allowing us to evaluate TSFT beyond conventional, strongly structured context spaces. When the context space lacks clear structure, methods that learn where to train through source-task grouping (e.g., based on task embeddings, or gradient-based task-affinity grouping methods) could complement our scope on how much to train and provide an orthogonal and important extension to TSFT in the future work.

\section{Methodology Detail}
\subsection{MDP Interpretation}
\label{app_mdp}
In this section, we present an alternative MDP interpretation of the studied problem. Under this view, the problem can be formulated as a model-based MDP and solved optimally via dynamic programming (DP). Although this approach can be incorporated into TSFT, it is less computationally efficient than the ILP-based formulation used in our main method, as demonstrated in Appendix~\ref{app_computation}.

\subsubsection{MDP Formulation}
We frame the budget allocation in Eq.~(\ref{eq_obj}) as a sequential decision-making process over the fine-tuning dynamics. Given a pretrained policy, a policy class, an RL training algorithm, a set of source task distributions $\{X_{\mathcal{S}_1},\ldots,X_{\mathcal{S}_N}\}$, and a total budget $K$, the goal is to decide \emph{how much} budget should be assigned to each policy. 
Each decision step allocates one fixed budget unit to one policy, e.g., by running the given RL algorithm for a fixed number of samples or epochs on the corresponding source task set. 
Given the total budget $K$, the allocation naturally induces a finite horizon of $K$ decision steps.
Although the final objective depends only on the resulting allocation vector rather than the action order, this sequential formulation provides a convenient way to search over feasible allocations under the budget constraint. 
Formally, we formulate this allocation process as a model-based MDP, specifically an undiscounted finite-horizon MDP with horizon $K$ and discount factor $\gamma=1$.

\textbf{State Space:} A state $s_t \in S$ at step $t$ is defined by the allocation vector across policies, $s_t=(k_1^t, k_2^t, ..., k_N^t)$, where $k_n^t$ denotes the cumulative number of budget units allocated to policy $n$ up to step $t$. Note that the state space is \emph{discrete}, and the initial state corresponds to a zero allocation vector.

\textbf{Action Space:} An action $a_t \in [1, 2, ..., N, \emptyset]$ specifies which policy receives an additional budget unit at step $t$. Choosing $\emptyset$ corresponds to termination (i.e., early stopping), capturing cases where allocating further budget does not improve the final objective. 

\textbf{Transition Dynamics and Model:} The state transition $P$ is \emph{deterministic} with respect to the allocation state. 
Specifically, selecting policy $n$ at step $t$ updates the allocation state by assigning one additional budget unit to that policy. 
For example, if the first policy is selected, the state transitions from $s_t=(k_1^t, k_2^t, ..., k_N^t)$ to 
$s_{t+1}=(k_1^t+1, k_2^t, ..., k_N^t)$. 
The size of the budget unit determines the decision and transition granularity.
This deterministic transition is defined over the abstract allocation state, not over the underlying stochastic training process. In practice, the parameter trajectory and realized performance of each policy after receiving additional training budget can be stochastic due to random initialization, data sampling, environment interaction, and the RL optimization procedure. We therefore use a model-based MDP: the allocation transition is deterministic, while the effect of training on task performance and coverage is estimated by a learned performance model.

Concretely, since the global coverage is not known a priori, we model the evolution of each policy's task-wise performance and thereby implicitly estimate its coverage set. 
For each policy-task pair $(n,x)$, we fit a parametric performance model of the form:
\begin{equation}
    \mathcal{F}(y) = L \pm B e^{-d y},
\end{equation}
where $L$, $B$, and $d$ are learnable parameters, and $y$ denotes the consumed training budget. 
The sign of $B$ controls the direction of the performance shift, capturing both improvement and degradation during fine-tuning. 
We fit $\mathcal{F}_{n}^{x}(y)$ via non-linear least squares (NLLS), which predicts the performance of policy $n$ on task $x$ after consuming a given amount of budget $y$. 
Empirically, this fitting process can be performed in parallel over thousands of tasks within seconds. 

\textbf{Reward Function:} The reward $r_t$ is defined as the marginal improvement in global coverage. Given that specialized policies may exhibit overlapping coverage across the context space, we define the \emph{global coverage} at state $s_t$ as the measure of the union of individual coverage sets:
\begin{equation}
    G(s_t) = \left| \bigcup_{n=1}^N \mathcal{C}(\theta_{X_{\mathcal{S}_n}}^{k_n^t}) \right|.
\end{equation}
The reward at step $t$ is then the marginal gain resulting from action $a_t$: $r_t = r(s_t, a_t) = G(s_{t+1}) - G(s_t)$. This formulation ensures that the reward is coupled across policies, naturally penalizing redundant optimization and incentivizing the expansion of the collective coverage boundary.

\subsubsection{MDP Solving}
\label{method_mdp}
The true optimal solution (i.e., Oracle) could, in principle, be obtained by training each policy to completion (using the full budget $K$) and solving Eq.~(\ref{eq_obj}) exactly via DP in a model-free MDP. However, in practice, $J(\theta_{X_{\mathcal{S}_n}}^{k_n}, x)$ is unknown during the planning phase. We therefore replace it with our model prediction $\mathcal{F}_n^x(\cdot)$, yielding a surrogate objective that we optimize within our framework. 
For clarity, we denote the resulting model-induced surrogate coverage objective by $\widehat G$. That is, $\widehat G(\cdot)$ is computed by replacing the unknown performance with the model prediction.

Based on the above model-based MDP formulation, we develop a discrete DP solver to derive its optimal allocation policy. To optimize the final objective, the solver applies \emph{value iteration} on the surrogate objective, searching for the allocation trajectory that maximizes the global coverage upon exhaustion of the total budget $K$.
Let $V^*(s_t)$ represent the optimal state value, defined as the maximum achievable global coverage from state $s_t$ with remaining budget $b_t$. Our DP solver recursively computes this value following the Bellman optimality principle:
\begin{equation}
    \label{eq_bellman}
    V^*(s_t) = \max \left(\widehat G(s_t), \max_{a_t} V^*(s_{t+1}) \right),
\end{equation}
where $s_{t+1}$ is the deterministic next state induced by taking action $a_t$ in state $s_t$. We omit a discount factor (i.e., $\gamma=1$), as the objective is to maximize the final global coverage over a finite horizon. The optimal policy is obtained via backward induction over the state space. An implementation code snippet is provided in Appendix \ref{app_dp_detail}.
Our DP solver achieves global optimality and permutation invariance, as detailed in Remark~\ref{remark_dp}. Note that this optimality is defined w.r.t. the model-based MDP under the surrogate objective, and does not necessarily carry over to the original objective in Eq.~(\ref{eq_obj}).

\begin{remark}[DP Property]
\label{remark_dp}
1) Global Optimality: The DP solver guarantees global optimality strictly for the full budget $K$. Intermediate allocations along the optimal trajectory do not necessarily constitute optimal solutions for smaller budget constraints.
2) Permutation Invariance: 
The final global coverage depends exclusively on the allocation vector $\mathcal{K}$. The specific sequence of actions taken to reach this distribution is commutative and irrelevant to the final objective.
\end{remark}

The computational complexity of the DP solver is dominated by the size of the reachable state space $|S|$. Due to the combinatorial structure of the state representation and the total budget constraint $K$, the number of reachable states is mathematically upper bounded by the number of weak compositions of at most $K$ units into $N$ policies, i.e., $|S|\leq \binom{K+N}{N}$. For fixed $N$, this bound scales as $\mathcal{O}(K^N)$.

\begin{remark}[Computational Complexity]
\label{remark_complexity}
The time and space complexities of the DP solver scale exponentially with the number of policies $N$, and polynomially with the total budget $K$. 
\end{remark}

\begin{wrapfigure}{r}{0.313\columnwidth}
    \vspace{-3mm}
    \begin{center}
    \includegraphics[width=0.34\columnwidth]{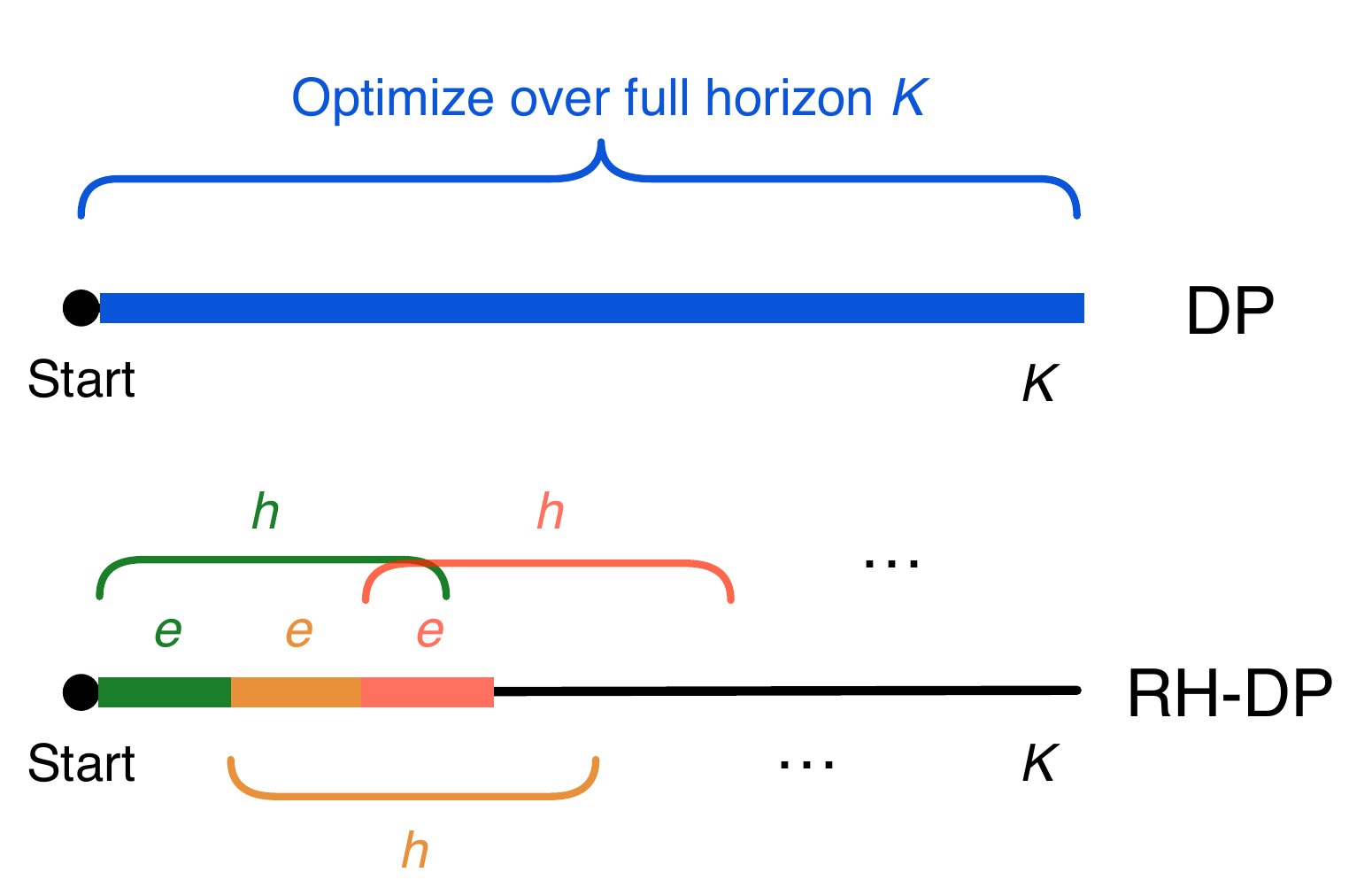}
    \caption{Illustration of RH-DP.}
    \label{dp_rhdp}
    \end{center}
\end{wrapfigure}
This curse of dimensionality motivates the exploration of more efficient strategies for solving the MDP. Inspired by the receding-horizon principle in model predictive control, we introduce \emph{receding-horizon DP} (RH-DP), an $h$-step lookahead strategy that trades global optimality for improved tractability. RH-DP is characterized by a planning horizon $h$ and an execution horizon $e$. Specifically, instead of solving the full-horizon DP over the entire remaining budget, RH-DP solves an $h$-step limited-horizon DP, executes the first $e \leq h$ allocation decisions, and then replans from the updated state. 
By bounding the complexity of each local planning phase to $\mathcal{O}(h^N)$, this approach reduces the overall computational burden while preserving adaptive replanning over the full budget.

\subsubsection{Dynamic Programming}
\label{app_dp_detail}
We provide a code snippet illustrating our DP implementation below.

\begin{python}
# An example of DP implementation
@lru_cache(maxsize=None)
def dp_value_iter(self, state, steps_remain):
    current_reward = self.get_coverage(state)
    if steps_remain == 0:
        return current_reward, []

    possible_actions = []
    for id, s in enumerate(state):
        if s < self.max_epochs:
            new_state = list(state)
            new_state[id] += 1
            val, path = self.dp_value_iter(tuple(new_state), steps_remain-1)
            possible_actions.append((val, id, path))

    # Optimization: Choose action with max future value
    best_future_val, best_action, best_path = max(possible_actions, key=lambda x: x[0])
    
    if best_future_val > current_reward:
        return best_future_val, [best_action] + best_path
    else:
        # No Action: early-stop as further training yields NO improvement
        return current_reward, []
\end{python}

\subsection{Alternative Formulation}
\label{app_formulation}
In addition to Eqs.~\eqref{eq:ilp_obj}-\eqref{eq:ilp_binary}, we present two alternative ILP formulations and a constraint-programming (CP) formulation. The original formulation uses \(K+1\) binary variables \(z_{n,k}\) to encode the allocation level of each policy. 
A natural way to potentially reduce the formulation size is to instead use a single integer variable for each policy. However, without additional structure (e.g., monotonicity), the coverage condition in Constraint~\eqref{eq:ilp_coverage} cannot be expressed linearly in these integer variables without introducing auxiliary variables.
We first consider the case in which coverage is monotone with respect to the allocated budget: $\mathcal{C}(\theta^{k}_{X_{S_n}}) \subseteq \mathcal{C}(\theta^{k+1}_{X_{S_n}}),\ \forall n, k\in\{0,\ldots,K-1\}$.
Under this assumption, we define the minimum budget required for policy \(n\) to cover task \(x\) as:
\begin{equation}
    T_{n,x}
    =
    \min\left\{
        k\in\{0,\ldots,K\}:
        x\in\mathcal{C}(\theta^{k}_{X_{S_n}})
    \right\},
\end{equation}
where \(T_{n,x}=+\infty\) if policy \(n\) cannot cover \(x\) within the available budget. In the model-based setting, \(\mathcal{C}\) and \(T_{n,x}\) are replaced by their surrogate counterparts induced by \(F_n^x\). Under monotone coverage, this leads to the following two ILP formulations.

\textbf{Threshold Formulation.} We directly optimize the allocation variables \(k_n\in\{0,\ldots,K\}\) and introduce \(c_{n,x}\in\{0,1\}\) to indicate whether policy \(n\) covers task \(x\):
\begin{align}
    \max_{\mathcal{K},\mathbf{c},\mathbf{q}}
    \quad & \sum_{x\in\mathcal{X}} q_x, \\
    \mathrm{s.t.}\quad
    & \sum_{n=1}^{N} k_n \le K, \\
    & T_{n,x}c_{n,x}
      \le k_n
      \le (T_{n,x}-1)+(K-T_{n,x}+1)c_{n,x},
      && \forall (n,x):T_{n,x}<+\infty, \\
    & q_x \le
      \sum_{n:T_{n,x}<+\infty} c_{n,x},
      && \forall x\in\mathcal{X}, \\
    & k_n\in\{0,\ldots,K\},\quad
      c_{n,x},q_x\in\{0,1\}.
\end{align}
The two-sided constraint enforces \(c_{n,x}=1\) if and only if \(k_n\ge T_{n,x}\). This formulation replaces the \(N(K+1)\) binary variables \(z_{n,k}\) with \(N\) integer variables. However, it introduces up to \(N|\mathcal{X}|\) policy-task variables and is exact only when coverage is monotone.

\textbf{Incremental Budget Formulation.} Let \(d_{n,k}=1\) indicate that policy \(n\) receives at least \(k\) budget units, for \(k\in\{1,\ldots,K\}\), and define \(d_{n,0}=1\). The allocation is then \(k_n=\sum_{k=1}^{K}d_{n,k}\), yielding:
\begin{align}
    \max_{\mathbf{d},\mathbf{q}}
    \quad & \sum_{x\in\mathcal{X}}q_x,\\
    \mathrm{s.t.}\quad
    & \sum_{n=1}^{N}\sum_{k=1}^{K}d_{n,k}\le K,\\
    & d_{n,k}\ge d_{n,k+1},
      && \forall n,\;k\in\{1,\ldots,K-1\},\\
    & q_x\le
      \sum_{n:T_{n,x}<+\infty}d_{n,T_{n,x}},
      && \forall x\in\mathcal{X},\\
    & d_{n,k},q_x\in\{0,1\}.
\end{align}
This cumulative encoding eliminates the policy-task variables \(c_{n,x}\) and naturally represents training as accumulated budget. Its main drawbacks are the \(NK\) binary variables and ordering constraints it introduces, as well as the same monotonicity requirement imposed by the threshold-based formulation.

\textbf{Constraint Programming.}
In addition to ILP, we formulate the problem using the \textsc{Element} constraint supported by OR-Tools CP-SAT. For every policy \(n\), allocation level \(k\), and task \(x\), define the constant $a_{n,k}^{x}=\mathbb{I}\!\left(x\in\mathcal{C}(\theta^{k}_{X_{S_n}})\right)$.
We introduce an integer variable \(k_n\in\{0,\ldots,K\}\) for the budget allocated to policy \(n\), and a binary variable \(c_{n,x}\) indicating whether policy \(n\) covers task \(x\) under its selected allocation. The formulation is as follows:
\begin{align}
    \max_{\mathcal{K},\mathbf{c},\mathbf{q}}
    \quad & \sum_{x\in\mathcal{X}}q_x,\\
    \mathrm{s.t.}\quad
    & \sum_{n=1}^{N}k_n\le K,\\
    & c_{n,x} =
      \operatorname{Element}
      \left(
          k_n;
          a_{n,0}^{x},\ldots,a_{n,K}^{x}
      \right),
      && \forall n\in\{1,\ldots,N\},\;x\in\mathcal{X},\\
    & q_x\le\sum_{n=1}^{N}c_{n,x},
      && \forall x\in\mathcal{X},\\
    & k_n\in\{0,\ldots,K\},\qquad
      c_{n,x},q_x\in\{0,1\}.
\end{align}
The \textsc{Element} constraint sets \(c_{n,x}\) to the \(k_n\)-th entry of the coverage table \((a_{n,0}^{x},\ldots,a_{n,K}^{x})\). Unlike the threshold-based formulations, this formulation supports arbitrary non-monotone coverage sets. 
Compared with the original ILP, it replaces the \(N(K+1)\) binary allocation variables with \(N\) integer variables, but introduces \(N|\mathcal{X}|\) binary coverage variables and element constraints. It may therefore be advantageous when \(K\) is large relative to \(|\mathcal{X}|\), but can become expensive for large context spaces, and its computational efficiency depends on CP-SAT's constraint propagation.

Overall, the threshold formulation can reduce the dependence on $K$ when the context space is small or the policy-task coverage relation is sparse, while the incremental formulation can yield sparser coverage constraints under monotonicity. In practice, however, monotone coverage is difficult to guarantee. Consequently, the original ILP and the CP formulation are more general. Among them, the original ILP offers a direct representation that benefits from mature mixed-integer optimization techniques and, in our empirical evaluation, exhibits greater computational efficiency.

\subsection{Alternative Model}
\label{app_model}
In addition to the parametric model, we further explore alternative modeling approaches.

\textbf{Gaussian Process (GP).} For each policy-task pair, we employ a GP model to implicitly represent a distribution over performance functions, providing a non-parametric alternative to the explicit exponential model. Specifically, for each policy-task pair $(n,x)$, given the observed performance trajectory $\mathcal{Z}_{n}^{x}=\{(y_i, r_i)\}_{i=1}^{m}$, where $y_i$ denotes the training step and $r_i$ denotes the observed performance, we assume $r_i = f_{n}^{x}(y_i) + \epsilon_i,\ \epsilon_i \sim \mathcal{N}(0,\sigma_\epsilon^2)$.
We place a GP prior over the latent function $f_{n}^{x}(\cdot) \sim \mathcal{GP}\big(0,k_\theta(\cdot,\cdot)\big)$,
where $k_\theta$ is the covariance kernel. Given $\mathcal{Z}_{n}^{x}$, the GP yields a posterior predictive distribution at any future training step $y$:
\begin{equation}
    p\big(f_{n}^{x}(y)\mid \mathcal{Z}_{n}^{x}\big) = \mathcal{N}\big(\mu_{n}^{x}(y),\sigma_{n}^{x}(y)^2\big).
\end{equation}
We use the posterior mean $\mu_{n}^{x}(y)$ as the predicted performance for policy $n$ on task $x$ at step $y$, thereby implicitly predicting the evolution of the corresponding coverage set.

\textbf{Parametric Function with GP (PFGP).} Although the parametric approach provides a smooth and interpretable estimate of the performance trajectory, fine-tuning, particularly in RL settings, often exhibits noise and instability. To account for this effect, we further combine the parametric model with a GP-based residual model. Specifically, the parametric function is used to model the base performance trend, while a GP is trained on the residuals to capture stochastic performance variations.
For each policy-task pair $(n,x)$, given the observed trajectory 
$\mathcal{Z}_{n}^{x}=\{(y_i,r_i)\}_{i=1}^{m}$, we first fit the parametric model $\mathcal{F}_{n}^{x}(y)$ and compute the residuals $\delta_i = r_i - \mathcal{F}_{n}^{x}(y_i)$.
We then fit a GP to the residual data $\{(y_i,\delta_i)\}_{i=1}^{m}$, which gives a posterior mean $\mu_{\delta,n}^{x}(y)$ for the residual at any future training step $y$. The final prediction is obtained by adding this residual correction to the parametric trend: $\widehat{\mathcal{F}}_{n}^{x}(y)=\mathcal{F}_{n}^{x}(y)+\mu_{\delta,n}^{x}(y)$.
In this way, PFGP preserves the smooth extrapolation structure of the parametric model, while using the GP residual component to capture local deviations and stochastic fluctuations around the fitted trend.

\textbf{Other Parametric Forms.} We additionally consider a power-law model, $\mathcal{F}(y)=L\pm B(y+1)^{-d}$ with $d>0$; a logarithmic model, $\mathcal{F}(y)=L\pm B\log(1+y)$; and polynomial models of degree $d$, $\mathcal{F}(y)=\sum_{j=0}^{d}c_jy^j$, for $d\in\{2,3,4\}$ (quadratic, cubic, and quartic). We also evaluate a piecewise-linear model consisting of two independently fitted linear segments, with the breakpoint selected using the Bayesian information criterion (BIC). The power-law and logarithmic models impose smooth monotonic trends, whereas the polynomial and piecewise-linear models offer greater flexibility for modeling non-monotonic trajectories. A comprehensive discussion is provided in Appendix~\ref{app_model_study}.

\newpage
\section{Theoretical Analysis}
\label{app_theory}
We provide a theoretical analysis of the allocation optimality gap introduced by our TSFT framework. Recall that the objective is to maximize the
global coverage under a total budget $K$:
\begin{equation}
    G(\mathcal{K})
    =
    \left|
    \bigcup_{n=1}^N
    \mathcal{C}\!\left(\theta_{\mathcal{X}_{S_n}}^{k_n}\right)
    \right|,
\end{equation}
where $\mathcal{K}=(k_1,\ldots,k_N)$ denotes the allocation vector and $k_n$ is the budget allocated to
the $n$-th policy. Since the true task performance trajectory $J(\theta,x)$ is unknown during
planning, TSFT relies on a predictive model $\mathcal{F}_n^x$ to construct a surrogate global
coverage objective, denoted by $\widehat{G}(\mathcal{K})$.

Let $\mathcal{K}^0$ denote the allocation after the warmup stage. For example, when the warmup
budget $W$ is uniformly allocated across $N$ policies and $W$ is divisible by $N$, we have
$\mathcal{K}^0=(W/N,\ldots,W/N)$. The post-warmup feasible allocation set is defined as
\begin{equation}
    \Omega_W
    =
    \left\{
    \mathcal{K}\in\mathbb{Z}_{\ge 0}^N:
    \mathcal{K}\succeq \mathcal{K}^0,\;
    \|\mathcal{K}\|_1\le K
    \right\},
\end{equation}
where $\mathcal{K}\succeq \mathcal{K}^0$ denotes element-wise inequality. Both TSFT and Oracle-Warmup are constrained to optimize over this same post-warmup feasible allocation set.

\begin{assumption}[Bounded Surrogate Coverage Error]
\label{assum:bounded_error}
There exists a constant $\delta_m\ge 0$ such that the surrogate coverage objective uniformly approximates the true coverage objective over the post-warmup feasible allocation set:
\begin{equation}
    \sup_{\mathcal{K}\in\Omega_W}
    \left|
    G(\mathcal{K})-\widehat{G}(\mathcal{K})
    \right|
    \le
    \delta_m .
\end{equation}
\end{assumption}

\begin{assumption}[Surrogate Planning Accuracy]
\label{assum:planning_accuracy}
Let $\widehat G$ be the fixed surrogate objective used to evaluate allocations over $\Omega_W$, and let
$\mathcal{K}_{\mathrm{TSFT}}\in\Omega_W$ be the allocation returned by TSFT. We assume that the executed allocation is $\eta_{\mathrm{alg}}$-optimal with respect to the surrogate objective:
\begin{equation}
    \widehat{G}(\mathcal{K}_{\mathrm{TSFT}})
    \ge
    \max_{\mathcal{K}\in\Omega_W}
    \widehat{G}(\mathcal{K})
    -
    \eta_{\mathrm{alg}} .
\end{equation}
\end{assumption}

\begin{theorem}[Performance Bound under Surrogate Model Error]
\label{thm:model_error_bound}
Let $\mathcal{K}_{\mathrm{OW}}^* \in \arg\max_{\mathcal{K}\in\Omega_W} G(\mathcal{K})$ be the Oracle-Warmup allocation, and let $\mathcal{K}_{\mathrm{TSFT}}$ be the allocation returned by TSFT. Under Assumptions~\ref{assum:bounded_error} and~\ref{assum:planning_accuracy}, we have
\begin{equation}
    G(\mathcal{K}_{\mathrm{OW}}^*)
    -
    G(\mathcal{K}_{\mathrm{TSFT}})
    \le
    2\delta_m+\eta_{\mathrm{alg}}.
\end{equation}
For a one-shot exact full-horizon planner that optimizes the same fixed surrogate objective and executes the resulting allocation exactly, $\eta_{\mathrm{alg}}=0$, and the post-warmup gap is bounded by $2\delta_m$.
\end{theorem}

\begin{proof}
By adding and subtracting the surrogate objective, we obtain
\begin{align}
    G(\mathcal{K}_{\mathrm{OW}}^*) - G(\mathcal{K}_{\mathrm{TSFT}})
    &=
    \left[
    G(\mathcal{K}_{\mathrm{OW}}^*)-\widehat{G}(\mathcal{K}_{\mathrm{OW}}^*)
    \right]
    +
    \left[
    \widehat{G}(\mathcal{K}_{\mathrm{OW}}^*)
    -
    \widehat{G}(\mathcal{K}_{\mathrm{TSFT}})
    \right] \notag\\
    &\quad+
    \left[
    \widehat{G}(\mathcal{K}_{\mathrm{TSFT}})
    -
    G(\mathcal{K}_{\mathrm{TSFT}})
    \right].
\end{align}
By Assumption~\ref{assum:bounded_error}, the first and third terms are each upper bounded by
$\delta_m$. Moreover, by Assumption~\ref{assum:planning_accuracy},
\begin{equation}
    \widehat{G}(\mathcal{K}_{\mathrm{OW}}^*)
    -
    \widehat{G}(\mathcal{K}_{\mathrm{TSFT}})
    \le
    \max_{\mathcal{K}\in\Omega_W}\widehat{G}(\mathcal{K})
    -
    \widehat{G}(\mathcal{K}_{\mathrm{TSFT}})
    \le
    \eta_{\mathrm{alg}} .
\end{equation}
Combining the three inequalities gives
\begin{equation}
    G(\mathcal{K}_{\mathrm{OW}}^*)
    -
    G(\mathcal{K}_{\mathrm{TSFT}})
    \le
    2\delta_m+\eta_{\mathrm{alg}} .
\end{equation}
\end{proof}

\textbf{Remark on Coverage Error Assumption.}
Assumption~\ref{assum:bounded_error} is stated directly at the level of the coverage objective. This
is because $G$ is a thresholded objective: small prediction errors in task performance can change the
coverage indicator for tasks whose true performance lies close to the coverage threshold. Therefore,
$\delta_m$ should be interpreted as the induced coverage-level error of the surrogate model. In
practice, online model re-estimation can reduce this error as more evaluation data are collected,
although a monotonic decrease of $\delta_m$ is not guaranteed without additional assumptions.

Next, we decompose the gap between TSFT and the full Oracle. Assume $\|\mathcal K^0\|_1\le K$, so that $\Omega_W\subseteq\Omega_0$. Let the full feasible allocation set be
\begin{equation}
    \Omega_0
    =
    \left\{
    \mathcal{K}\in\mathbb{Z}_{\ge 0}^N:
    \|\mathcal{K}\|_1\le K
    \right\}.
\end{equation}
The Oracle optimizes over $\Omega_0$, whereas Oracle-Warmup and TSFT optimize over the
restricted post-warmup feasible set $\Omega_W\subseteq\Omega_0$. Let $\mathcal{K}^*\in \arg\max_{\mathcal{K}\in\Omega_0} G(\mathcal{K})$
be the Oracle allocation.

\begin{theorem}[Oracle Gap Decomposition]
\label{thm:oracle_decomposition}
Define the warmup error as
\begin{equation}
    \mathcal{E}_{\mathrm{warmup}}
    =
    G(\mathcal{K}^*)-G(\mathcal{K}_{\mathrm{OW}}^*) .
\end{equation}
Then the total optimality gap of TSFT with respect to the full Oracle satisfies
\begin{equation}
    G(\mathcal{K}^*)-G(\mathcal{K}_{\mathrm{TSFT}})
    \le
    \mathcal{E}_{\mathrm{warmup}}
    +
    2\delta_m
    +
    \eta_{\mathrm{alg}} .
\end{equation}
\end{theorem}

\begin{proof}
By adding and subtracting the performance of Oracle-Warmup, we obtain the exact decomposition
\begin{align}
    G(\mathcal{K}^*)-G(\mathcal{K}_{\mathrm{TSFT}})
    &=
    \underbrace{
    \left[
    G(\mathcal{K}^*)-G(\mathcal{K}_{\mathrm{OW}}^*)
    \right]
    }_{\text{warmup error}}
    +
    \underbrace{
    \left[
    G(\mathcal{K}_{\mathrm{OW}}^*)-G(\mathcal{K}_{\mathrm{TSFT}})
    \right]
    }_{\text{post-warmup planning error}} .
\end{align}
The first term is $\mathcal{E}_{\mathrm{warmup}}$ by definition. The second term is bounded by
Theorem~\ref{thm:model_error_bound}. Therefore,
\begin{equation}
    G(\mathcal{K}^*)-G(\mathcal{K}_{\mathrm{TSFT}})
    \le
    \mathcal{E}_{\mathrm{warmup}}
    +
    2\delta_m
    +
    \eta_{\mathrm{alg}} .
\end{equation}
\end{proof}

\textbf{Interpretation.}
Theorem~\ref{thm:oracle_decomposition} decomposes the total gap to the full Oracle into two sources. The first term, $\mathcal{E}_{\mathrm{warmup}}$, captures the loss induced by committing the initial warmup budget before model-based planning begins.
Note that this term can be zero if the post-warmup state lies on an optimal trajectory (i.e., $\mathcal{K}^0\preceq \mathcal{K}^*$). We further ablate its effect on global coverage when it introduces non-zero error (see Fig.~\ref{app_hypara1} for the warmup sensitivity analysis).
The second term, bounded by $2\delta_m+\eta_{\mathrm{alg}}$, captures the post-warmup loss caused by surrogate model error and algorithmic approximation. Thus, TSFT approaches Oracle-Warmup when the surrogate coverage error is small and the planning procedure is accurate. It approaches the full Oracle when, in addition, the warmup allocation does not substantially restrict the optimal final allocation.

\textbf{Effect of Surrogate Error.}
Surrogate prediction error does not necessarily induce allocation error. If the surrogate objective preserves the ranking of high-quality feasible allocations, TSFT can still recover a strong post-warmup allocation even when $\widehat G\neq G$. Conversely, small prediction errors can affect the selected allocation when multiple feasible allocations have similar true coverage or when many tasks lie close to the threshold $\epsilon$.

\newpage
\section{Experiment Detail}
\label{app_exp}
We conduct experiments on a machine equipped with NVIDIA GeForce RTX 4090 GPUs and an AMD Ryzen Threadripper PRO 7975WX CPU for combinatorial optimization, on a machine with NVIDIA V100 GPUs and an Intel Xeon E5-2670 CPU for continuous control, and on a cluster with 16 NVIDIA H200 GPUs for LLM fine-tuning.

\subsection{Baseline}
\label{app_baseline}
Here, we detail the baseline implementations of Adaptive and LinUCB.
1) \emph{Adaptive:} We maintain a weight $w_n$ for each policy $n$, initialized uniformly. At each decision-making step, the weights are converted into allocation probabilities using a temperature-scaled softmax, and the execution budget $E$ is allocated across policies by sampling from the resulting multinomial distribution. After training, we evaluate the updated checkpoints and update each weight based on the coverage improvement per allocated budget unit. In this way, Adaptive prioritizes policies that have recently expanded their coverage set more efficiently. We set the temperature to 0.05 in all experiments.
2) \emph{LinUCB:} We formulate the budget allocation as a contextual multi-armed bandit problem \cite{li2010contextual,chu2011contextual}, where each policy corresponds to an arm. For each policy $n$, we maintain a ridge-regression estimator with a design matrix $A_n$ and a response vector $b_n$. At each decision-making step, LinUCB constructs a 5-dimensional context vector $z_n$ for each policy, which includes a bias term, normalized training progress, its squared value, the current coverage ratio, and recent coverage momentum. The estimated parameter for policy $n$ is given by $\hat{\delta}_n = A_n^{-1}b_n$. LinUCB then computes the upper-confidence score
\begin{equation}
    s_{n} = \hat{\delta}_n^\top z_{n} + \alpha \sqrt{z_{n}^\top A_n^{-1} z_{n}},
\end{equation}
where $\alpha=1.0$ controls the exploration strength. The policy $n^*$ with the largest score is selected and allocated one budget unit. After training, we evaluate the updated checkpoint and compute the reward as the monotonic coverage improvement over the historical best coverage of the selected policy, normalized by the allocated budget. This reward design avoids penalizing temporary performance drops during fine-tuning. Finally, the statistics $A_{n^*}$ and $b_{n^*}$ of the selected policy $n^*$ are updated using the observed context and reward.

\subsection{Policy Training}
\label{app_training}
\textbf{Combinatorial Optimization.} We adopt POMO~\cite{kwon2020pomo} as the policy network, a strong attention-based neural solver for routing problems. We largely follow the training setup in~\cite{kwon2020pomo}. Specifically, we use the Adam optimizer with a learning rate of $1\times 10^{-4}$, a weight decay of $1\times 10^{-6}$, and a batch size of $64$. The policy is pretrained for $5{,}000$ epochs over the entire context space, with each epoch containing $10{,}000$ training instances. Both the problem size and the POMO size are set to $100$.

\textbf{Continuous Control.} For CartPole and Ant experiments, we used PPO~\cite{schulman2017proximal} implemented by Stable Baseline3~\cite{stable-baselines3} and used the default hyperparameters, including a learning rate of $3\times 10^{-4}$, \texttt{n\_steps}$=2048$, batch size $64$, discount factor $0.99$, GAE parameter $0.95$, clipping parameter $0.2$, entropy coefficient $0$, and a value function loss coefficient $0.5$.
The policy and value networks share an MLP backbone whose hidden sizes are environment-specific: $[64,64]$ for CartPole 3D and $[256,256]$ for Ant 2D (the latter to accommodate Ant's $111$-dimensional observation). The pretraining run targets $5$M environment steps over the CartPole grid and $1$M over the Ant grid. Fine-tuning on each source region uses the same algorithm and hyperparameters across all $N$ regions; only the source task subset and the budget allocation $\mathcal{K}$ vary across runs.
For Meta-World experiments, we used MOORE~\cite{hendawy2023multi} as the base multi-task RL algorithm. We kept all MOORE hyperparameters identical to those used in the original implementation. We evaluate on the Meta-World MT50 benchmark, which contains $50$ manipulation tasks. The model is first pretrained for $50$M environment steps on all $50$ tasks, and then fine-tuned using a total budget of $50$M environment steps. During fine-tuning, the same algorithm and hyperparameter configuration are used across all task groups; only the task grouping and the budget allocation $\mathcal{K}$ vary across runs.

\textbf{LLM Fine-Tuning.} 
We perform reinforcement fine-tuning on Qwen3-4B-Base using GRPO~\cite{shao2024deepseekmath}, implemented in VeRL~\cite{sheng2025hybridflow} with a vLLM-based rollout backend \cite{kwon2023efficientmemorymanagementlarge}. All four specialized policies (DAPO-17K, MATH, GSM8K, CodeContests+) share an identical training configuration, only the training corpus changes across runs. The full hyperparameter list is given in Table~\ref{tab:llm_hparams}. After every GRPO update, the current checkpoint is evaluated on each benchmark in the context space (capped at 100 problems per benchmark to keep online evaluation tractable).

\begin{table*}[h]
\centering
\caption{GRPO hyperparameters used for LLM fine-tuning.}
\label{tab:llm_hparams}
\vspace{2pt}
\begin{small}
\renewcommand\arraystretch{0.9}
\begin{tabular}{lc}
\toprule
\textbf{Hyperparameter} & \textbf{Value} \\
\midrule
Parameter precision                   & BF16 (mixed precision) \\
\midrule
Group size (rollouts per prompt)      & 8 \\
KL loss coefficient                   & 0 (KL loss disabled) \\
KL penalty in reward                  & Disabled \\
Advantage normalization by group std  & Disabled \\
Entropy coefficient                   & 0 \\
\midrule
Optimizer                             & AdamW \\
Learning rate                         & $2\!\times\!10^{-6}$ \\
Train batch size                      & 128 \\
PPO mini-batch size                   & 64 \\
PPO micro-batch size per GPU          & 16 \\
\midrule
Max prompt length                     & 2{,}048 tokens \\
Max response length                   & 8{,}192 tokens \\
\bottomrule
\end{tabular}
\end{small}
\end{table*}

\subsection{Context Space}
\label{app_context_space}
Here, we detail the data generation, as well as the settings of the context space and source task sets.

\begin{figure*}[ht]
    \centering
    \centerline{
    \includegraphics[width=0.195\columnwidth]{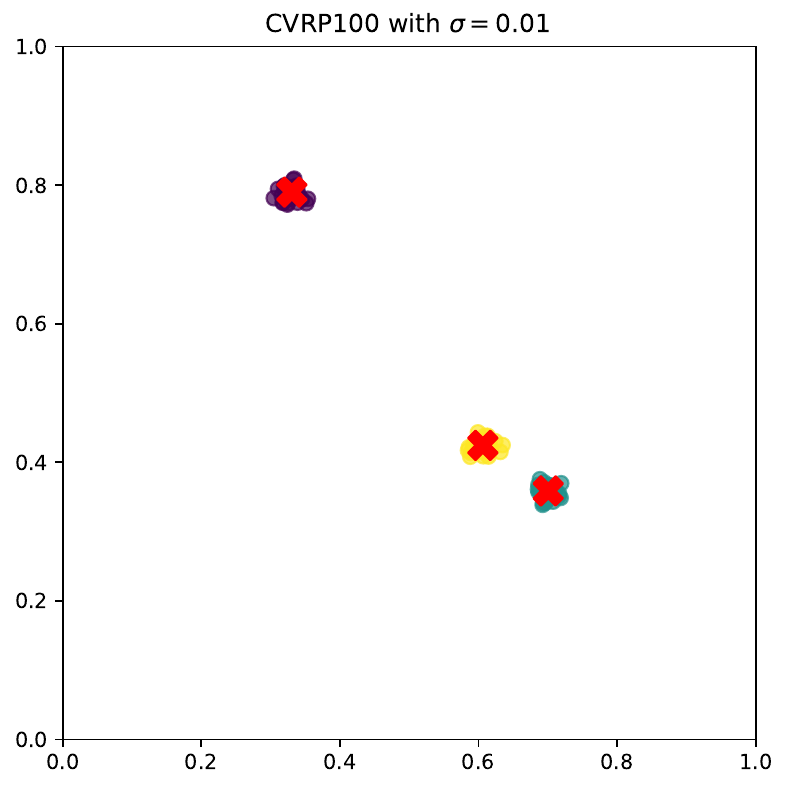}
    \includegraphics[width=0.195\columnwidth]{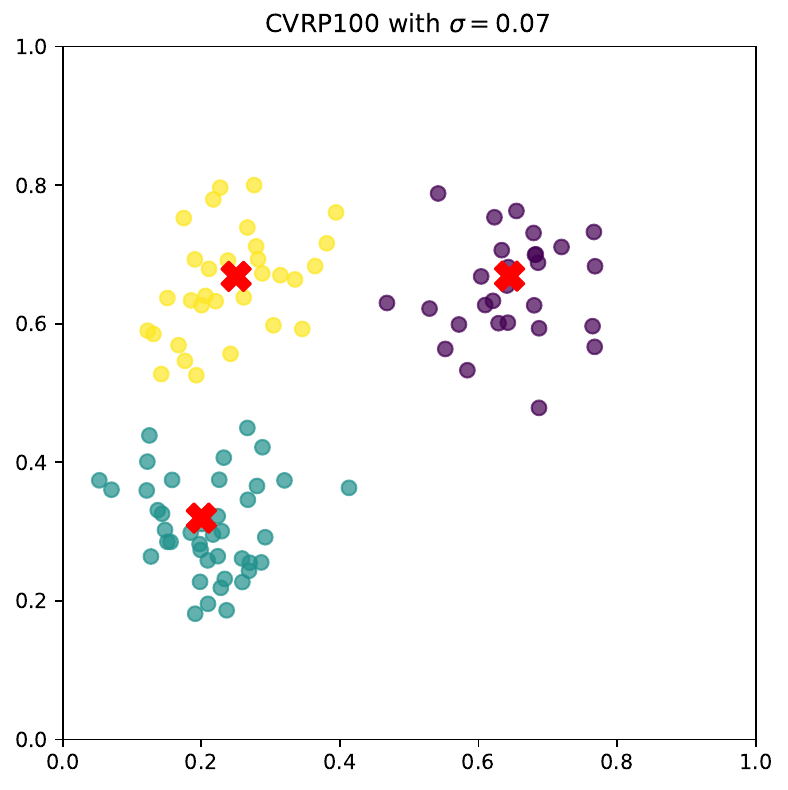}
    \includegraphics[width=0.195\columnwidth]{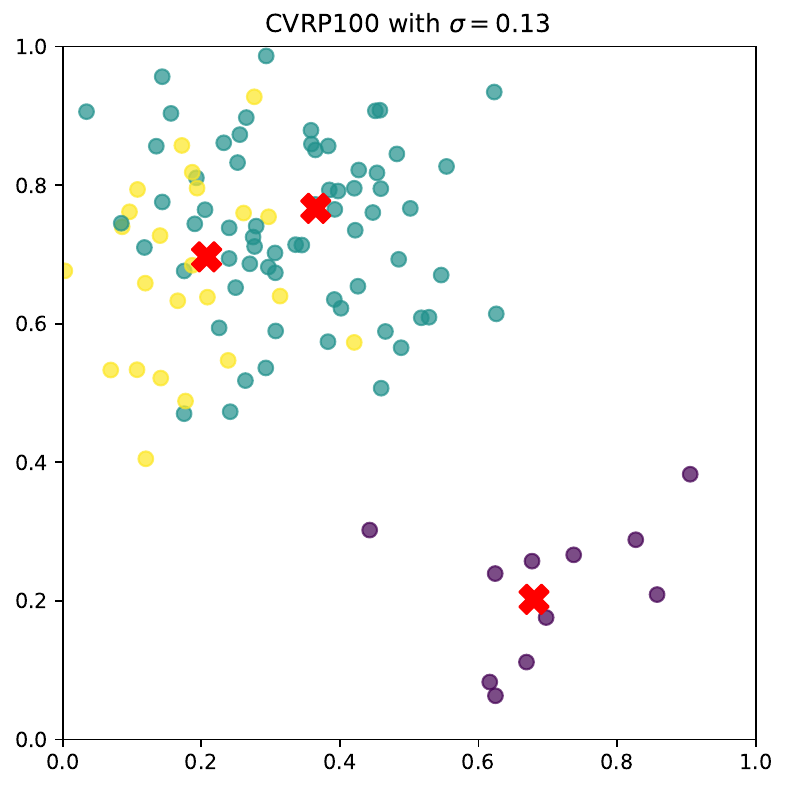}
    \includegraphics[width=0.195\columnwidth]{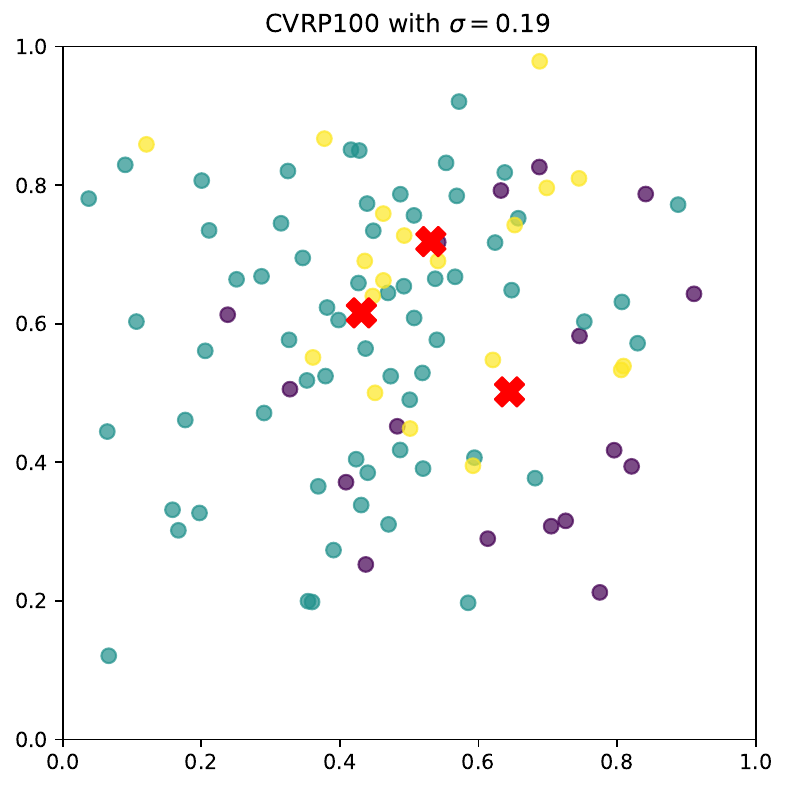}
    \includegraphics[width=0.195\columnwidth]{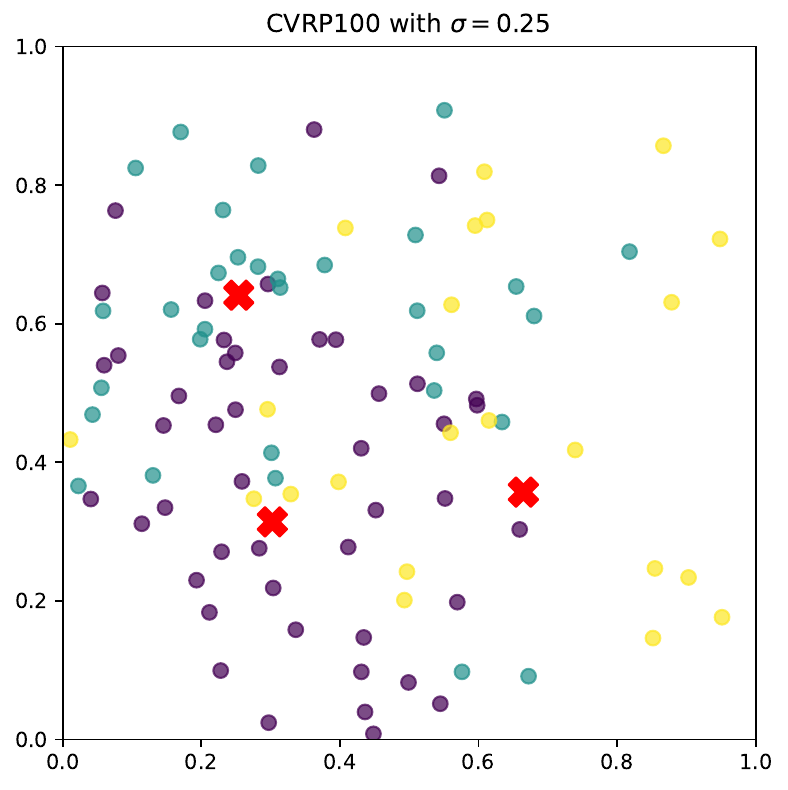}
    }
    \caption{Examples of node distributions obtained by varying the covariance $\sigma$ of a GMM.
    }
    \label{app_fig_cvrp}
\end{figure*}

\textbf{Combinatorial Optimization.} For CVRP, we follow the common settings in the literature~\cite{kool2018attention,kwon2020pomo}. The demand of each customer node is randomly sampled from a discrete uniform distribution over $\{1,2,\dots,9\}$. Before being fed into the network, each demand is normalized by the vehicle capacity. The context space is defined by two dimensions: vehicle capacity and node distribution. The vehicle capacity ranges from 10 to 400 with a step size of 10. We control the node distribution through the covariance of a Gaussian mixture model (GMM). Specifically, node locations are generated from a GMM with three clusters, where the cluster centers are sampled uniformly from $[0.2,0.8]^2$. Each cluster uses a diagonal covariance matrix $\Sigma=\operatorname{diag}(\sigma^2,\sigma^2)$, and we vary $\sigma$ to control the spatial dispersion of customer nodes along the two dimensions. Nodes sampled outside the unit square $[0,1]^2$ are rejected and resampled. A visualization example is provided in Fig.~\ref{app_fig_cvrp}. 
We construct five source task sets, each containing 49 tasks centered around a reference task (see Fig.~\ref{app_fig_context_space}). The reference tasks for $X_{S_1}, X_{S_2}, X_{S_3}, X_{S_4}, X_{S_5}$ are $(0.07,100)$, $(0.07,300)$, $(0.13,200)$, $(0.19,100)$, and $(0.19,300)$, respectively. We use $\{X_{S_1}, X_{S_3}, X_{S_5}\}$ and $\{X_{S_1}, X_{S_2}, X_{S_4}, X_{S_5}\}$ as the source task sets for the $N=3$ and $N=4$ cases in Table~\ref{table_nco_control}, respectively.

For CVRPTW, we extend CVRP by introducing time-window constraints. Specifically, each customer node $v_i$ is associated with a time window $[f_i, g_i]$ and a service time $m_i$. A vehicle must start serving customer $v_i$ within the interval $[f_i, g_i]$. If the vehicle arrives earlier than $f_i$, it must wait until $f_i$ before service can begin. All vehicles must return to the depot $v_0$ no later than $g_0$. We set the depot time window to $[f_0, g_0]=[0,4]$ and assign zero service time to the depot. The time window for each customer node $v_i$ is then generated as follows: 1) sample the time-window center $\gamma_i \sim U(f_0 + d_{0i}, g_0 - d_{i0} - m_i)$, where $d_{0i}=d_{i0}$ denotes the distance, or equivalently the travel time, between $v_0$ and $v_i$; 2) sample the time-window half-width $w_i$ uniformly at random from $[m_i/2, g_0/3]$; and 3) set the customer time window as $[f_i, g_i]=[\max(f_0, \gamma_i-w_i), \min(g_0, \gamma_i+w_i)]$. 
The context space is defined along two dimensions: vehicle capacity and time window (TW) tightness. Vehicle capacity ranges from 10 to 400 with a step size of 10. We control time window tightness by varying the service time $m_i$. Specifically, $m_i$ is varied from 0.04 to 1.00 with a step size of 0.04. Larger service times implicitly correspond to tighter time windows.
Similar to CVRP, we construct five source task sets, each containing 49 tasks centered around a reference task (see Fig. \ref{app_fig_context_space}). The reference tasks for $X_{S_1}, X_{S_2}, X_{S_3}, X_{S_4}, X_{S_5}$ are $(0.28,100)$, $(0.28,300)$, $(0.52,200)$, $(0.76,100)$, and $(0.76,300)$, respectively. We use $\{X_{S_1}, X_{S_3}, X_{S_5}\}$ and $\{X_{S_1}, X_{S_2}, X_{S_4}, X_{S_5}\}$ as the source task sets for the $N=3$ and $N=4$ cases in Table~\ref{table_nco_control}, respectively.

\textbf{Continuous Control.} 
For CartPole, we construct a $3$-dimensional context space along (pole length, cart mass, pole mass). Each axis is discretized into $10$ values evenly spaced over $[0.1, 10]\times$ the CARL defaults $(\ell_0,m_{c,0},m_{p,0})=(0.5,1.0,0.1)$, giving the explicit grids $\{0.05, 0.6, 1.15, 1.7, \dots, 5.0\}$ for pole length, $\{0.1, 1.2, 2.3, \dots, 10.0\}$ for cart mass, and $\{0.01, 0.12, 0.23, \dots, 1.0\}$ for pole mass. Source task sets correspond to $N=5$ axis-aligned $3^3 = 27$-context regions: four extremal corners and a geometric center. These regions are mutually disjoint and together cover $5\times 27 = 135$ of the $1{,}000$ context grid points; the remaining $865$ contexts serve as held-out targets that the fine-tuned policies must generalize to.

For Ant, we construct a $2$-dimensional context space along (gravity, friction). The two axes are discretized into $25$ values for gravity ($\{1.96, 2.69, \dots, 19.6\}\,\text{m/s}^2$, i.e.\ $[0.2,2.0]\times g_0$) and $40$ values for friction ($\{0.2, 0.246, \dots, 2.0\}$, i.e.\ $[0.2,2.0]\times \mu_0$), giving a $25\!\times\!40=1{,}000$-context grid.
Source task sets correspond to $N=5$ axis-aligned $5\!\times\!4=20$-context rectangles: four corners and a center. These five regions span $100$ of the $1{,}000$ grid points; the remaining $900$ contexts are held-out targets. The square and source regions are visualized in Fig.~\ref{app_fig_context_space}.

Meta-World MT50 contains 50 robotic manipulation tasks, each corresponding to a distinct task context. In MOORE, each task is initially represented by a one-hot encoding. We instead use the (K)-dimensional task-specific expert weights from the pretrained policy as the task context, since they provide a learned continuous representation of how MOORE combines experts for each task. We project the K-dimensional expert weights of the 50 tasks into a 2D space using PCA, and annotate each point with its task description, as shown in Fig. \ref{app_fig_context_space}.


\textbf{LLM Fine-Tuning.} We conduct experiments on a discrete context space comprising nine benchmarks across math reasoning, code generation, and general reasoning: DAPO-17K, MATH-500, GSM8K, CodeContests+, AIME 2024, AIME 2025, Minerva Math, MBPP, and BigCodeBench. We use benchmark-specific performance thresholds $\epsilon$, set to [0.75, 0.10, 0.10, 0.27, 0.36, 0.30, 0.70, 0.20, 0.90] for these benchmarks, respectively.

\begin{figure*}[ht]
    \centering
    \includegraphics[width=0.3\columnwidth]{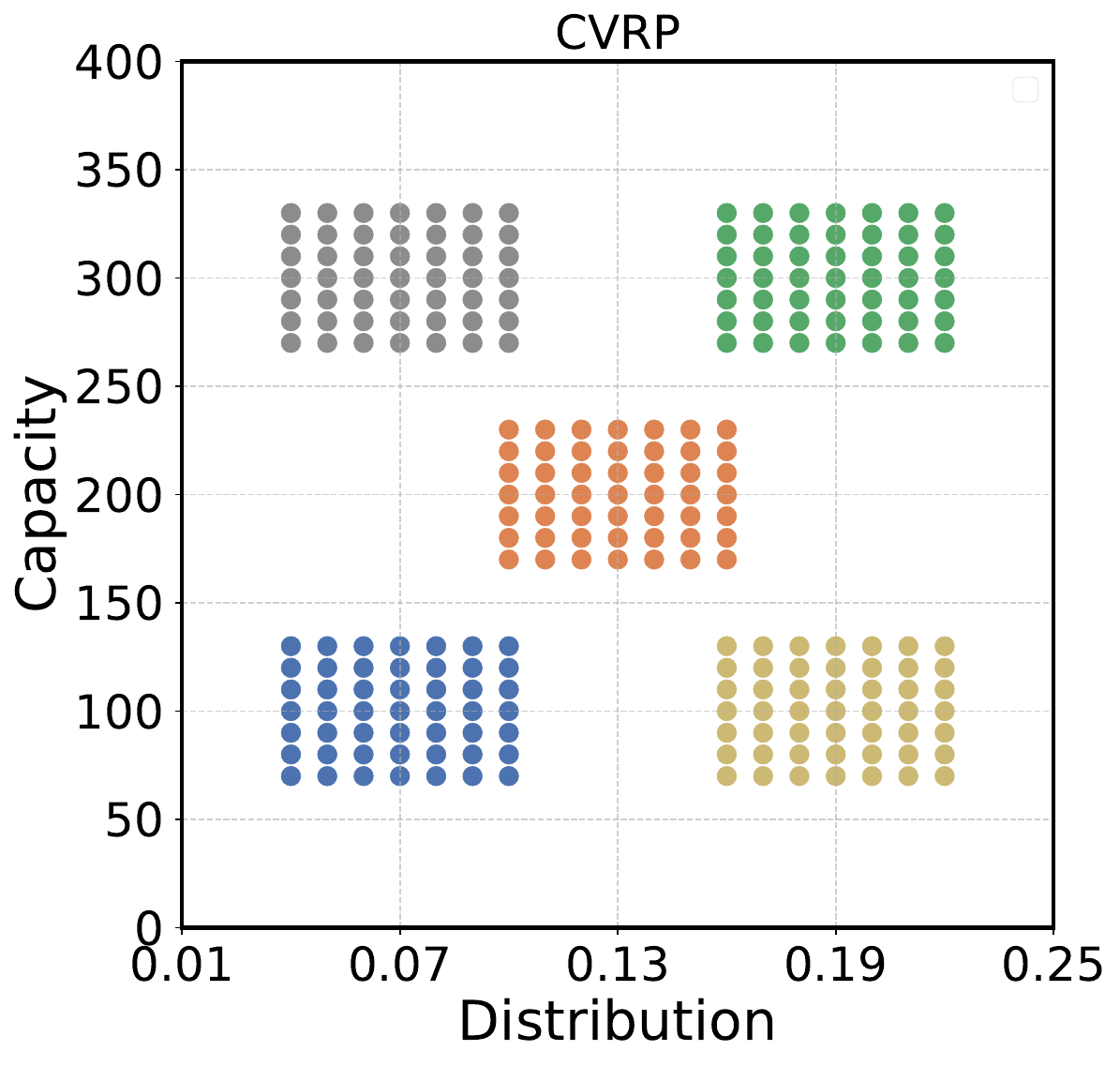}
    \includegraphics[width=0.3\columnwidth]{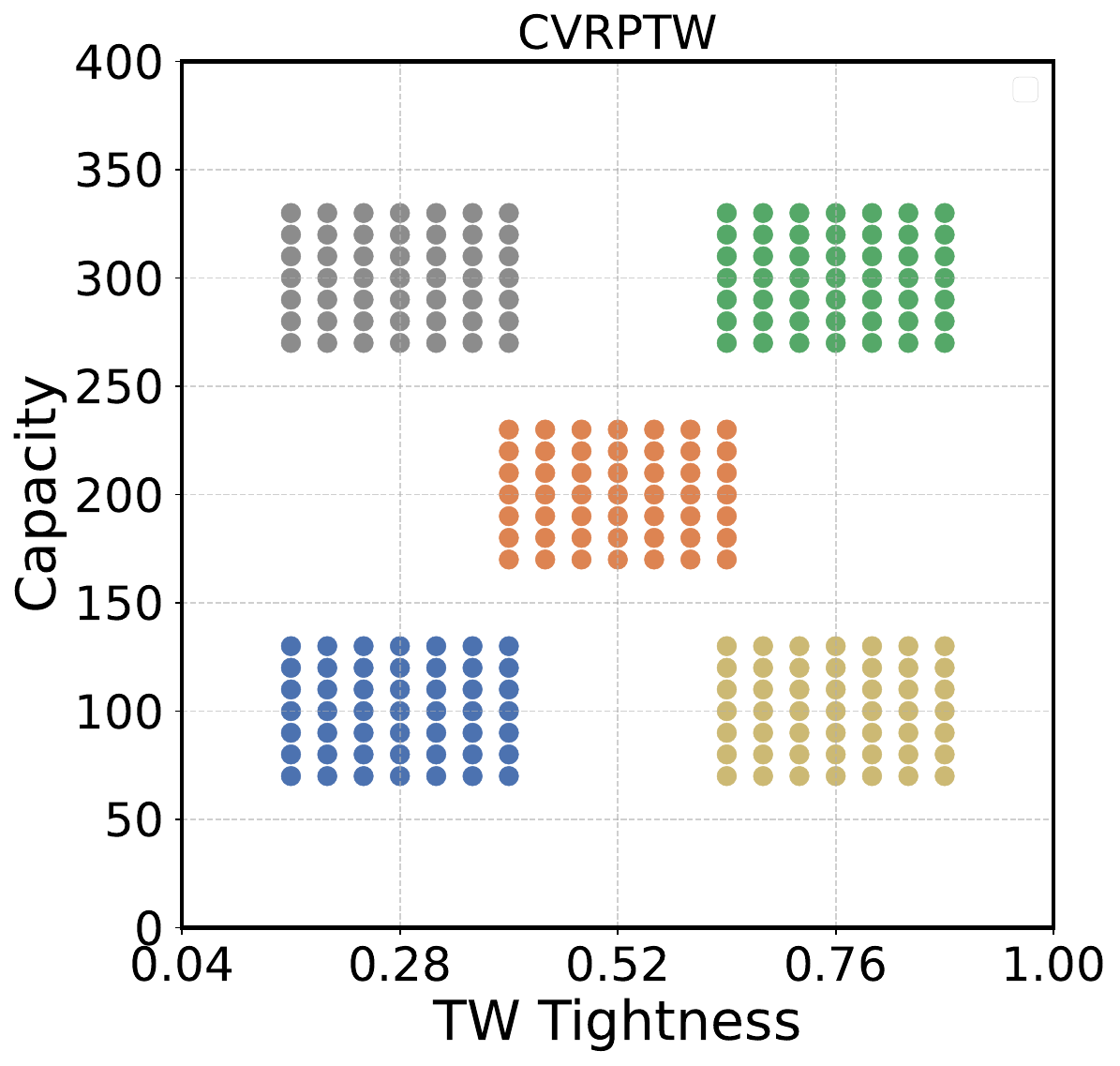}
    \includegraphics[width=0.3\columnwidth]{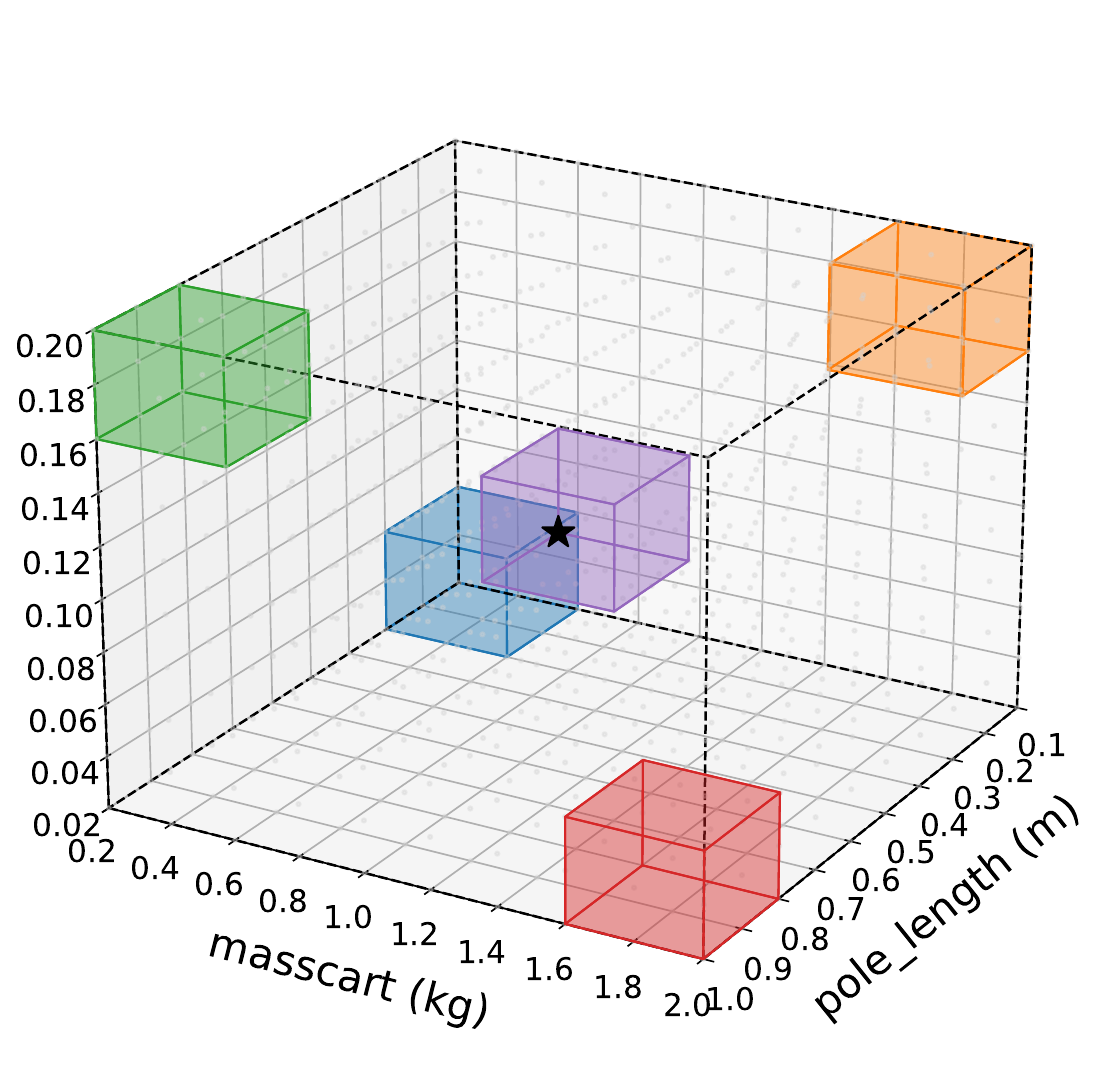}\\
    \includegraphics[width=0.32\columnwidth]{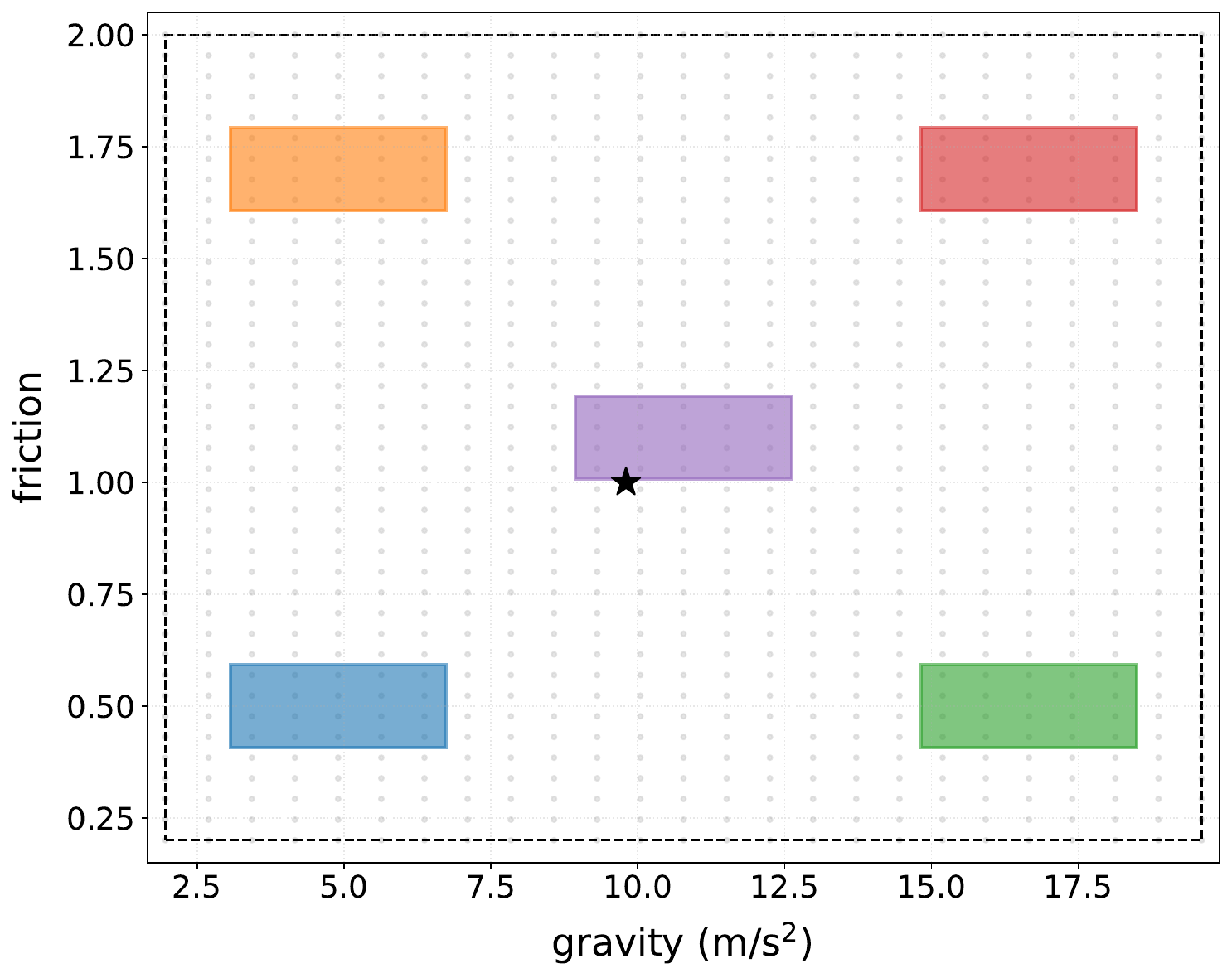}
    \raisebox{1.5mm}{\includegraphics[width=0.45\columnwidth]{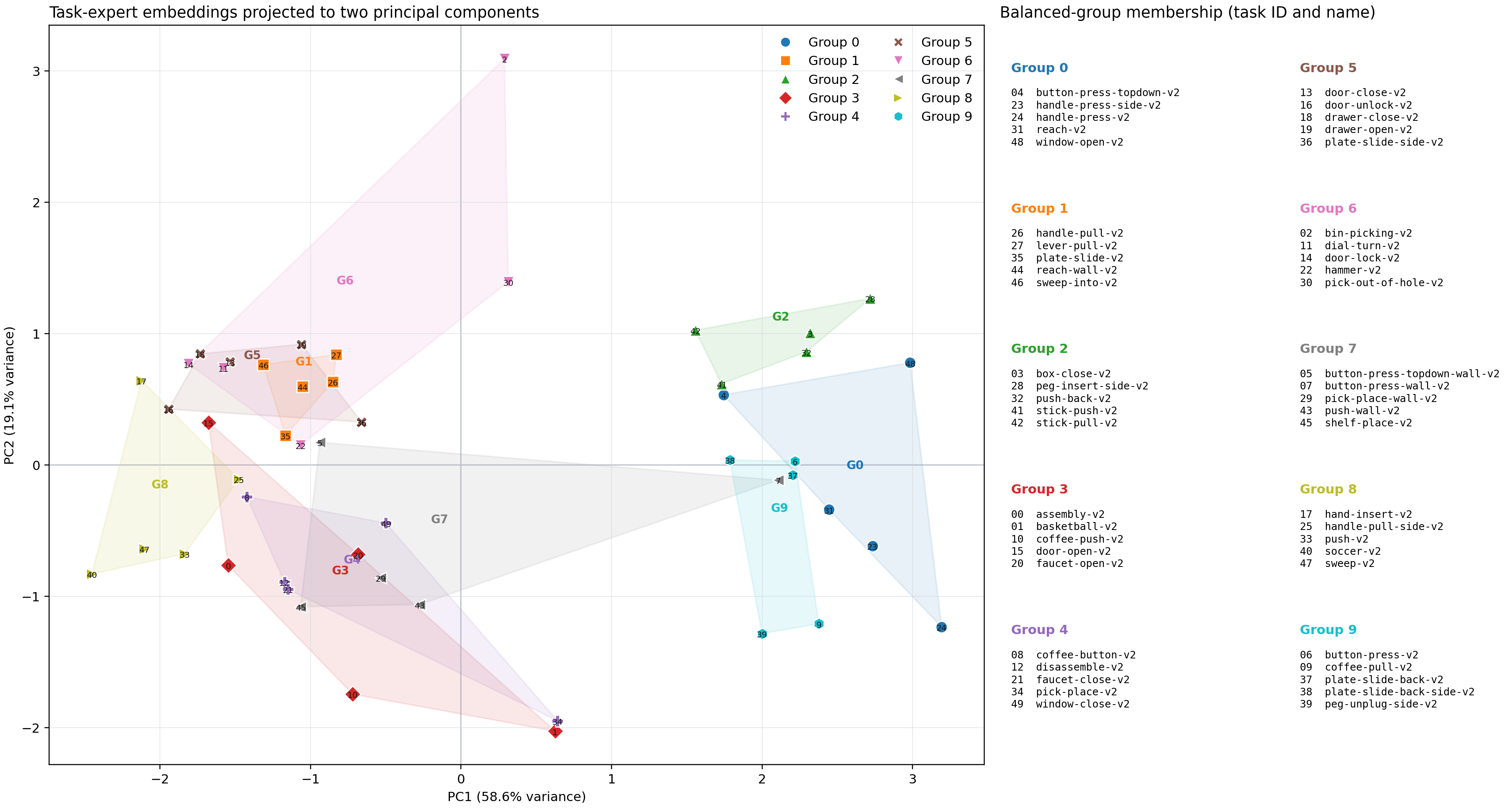}}
    \caption{Visualization of the context spaces and corresponding source task sets for CVRP, CVRPTW, CartPole, Ant, and Meta-World, respectively. The task embeddings in Meta-World are the $K$-dimensional expert weights associated with each task from the pretrained model. We use PCA to reduce them to two dimensions for visualization.
    }
    \label{app_fig_context_space}
    \vskip -0.1in
\end{figure*}

\clearpage
\newpage
\section{Additional Result}
\label{app_exp_result}
\subsection{Full Result}
We conduct experiments with 5 seeds for CartPole, and 3 seeds for Ant, with the full results reported in Tables \ref{table_control_full} and \ref{table_control_full_ant_metaworld}. For combinatorial optimization and LLM fine-tuning, we follow the common convention of reporting results from a single seed~\cite{kool2018attention,kwon2020pomo,yu2025dapo,limozin2026sft}, due to their relatively stable performance or high computational cost. 
For Meta-World, we use a single seed in the main experiments and additionally evaluate robustness by running 5 random seeds on a smaller setting with $N=5$ and $K=25$.
We summarize the results in Fig.~\ref{app_figure_bar}, where error bars denote the standard deviation across seeds. We observe that MTL can be unstable when trained over a large context space, unless equipped with sufficient model capacity and state-of-the-art MTL algorithms. For example, in Ant, allocating additional budgets to MTL does not necessarily improve task coverage. In contrast, TSFT generally performs robustly across these diverse settings. 
Moreover, we note that the large discrepancy between 93.4\% and 0.0\% coverage across Ant trials is primarily an artifact of the hard coverage threshold ($\epsilon = 6.9$). Specifically, the pretrained policies from the three seeds achieve mean returns of 6.917, 6.825, and 6.858, respectively. Therefore, 0\% coverage does not indicate a failure to learn an effective policy, but rather that the policy narrowly misses the predefined threshold.

\begin{figure*}[h]
    \centering
    \includegraphics[width=0.325\columnwidth]{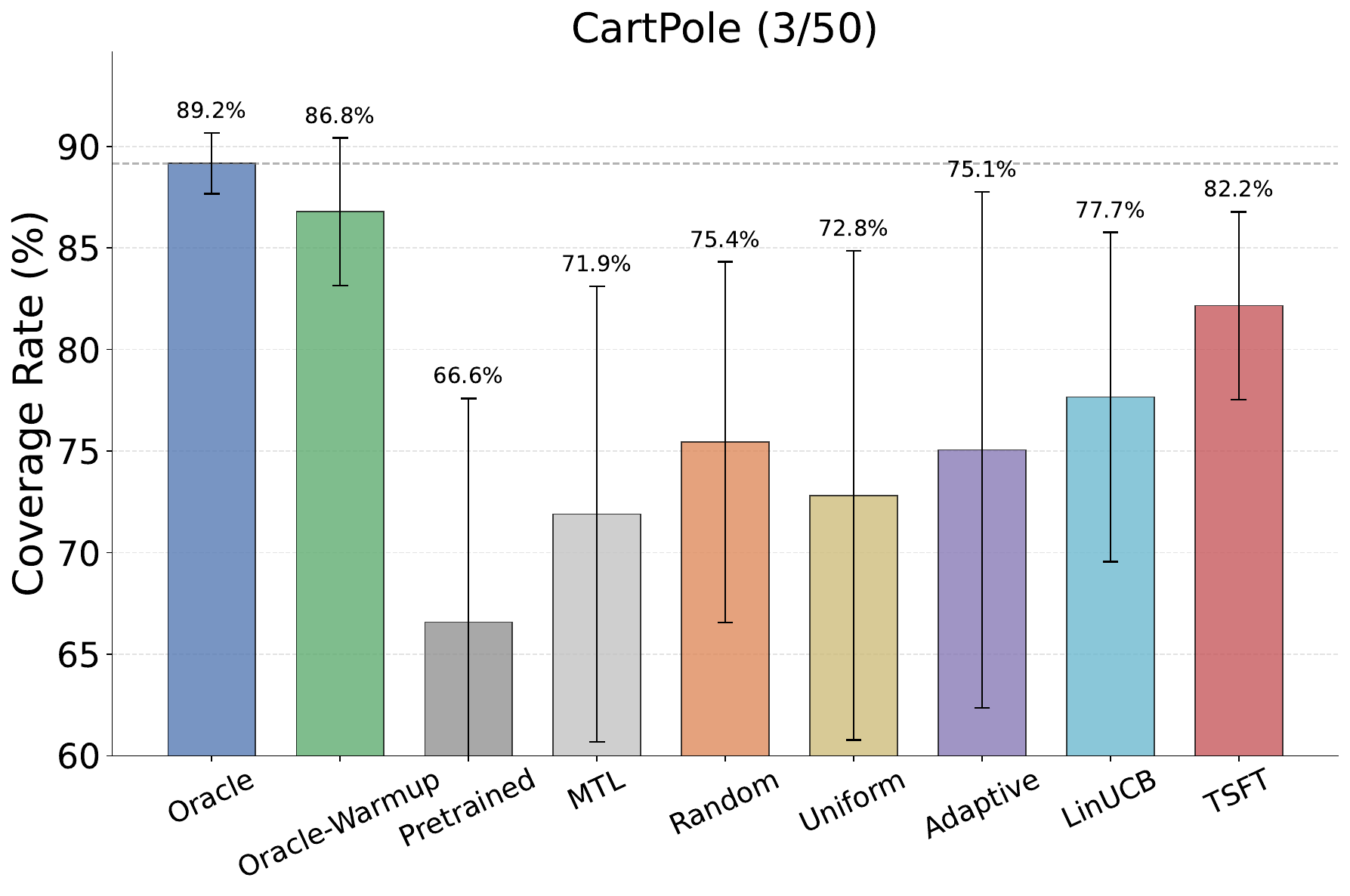}
    \includegraphics[width=0.325\columnwidth]{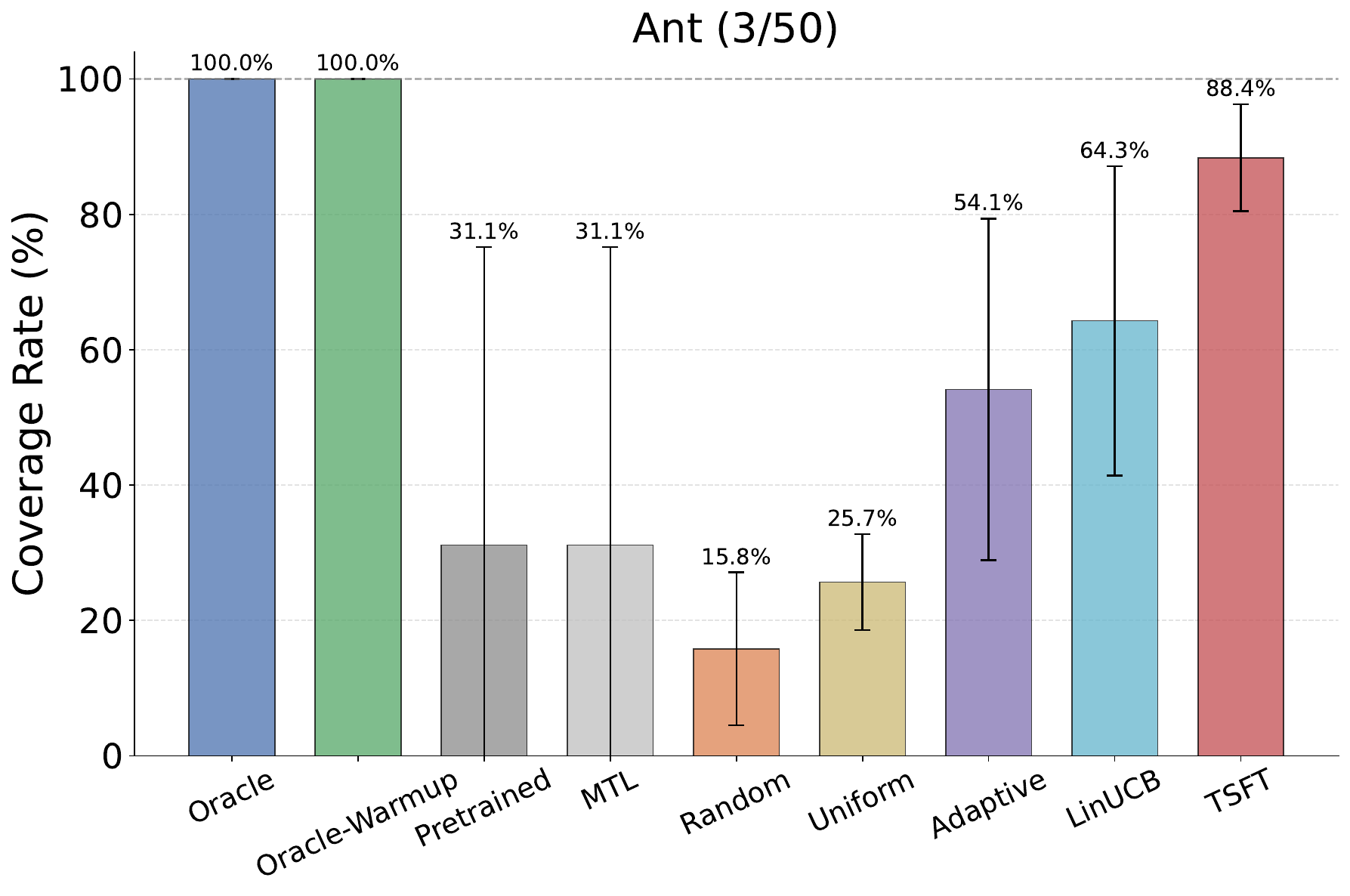}
    \includegraphics[width=0.325\columnwidth]{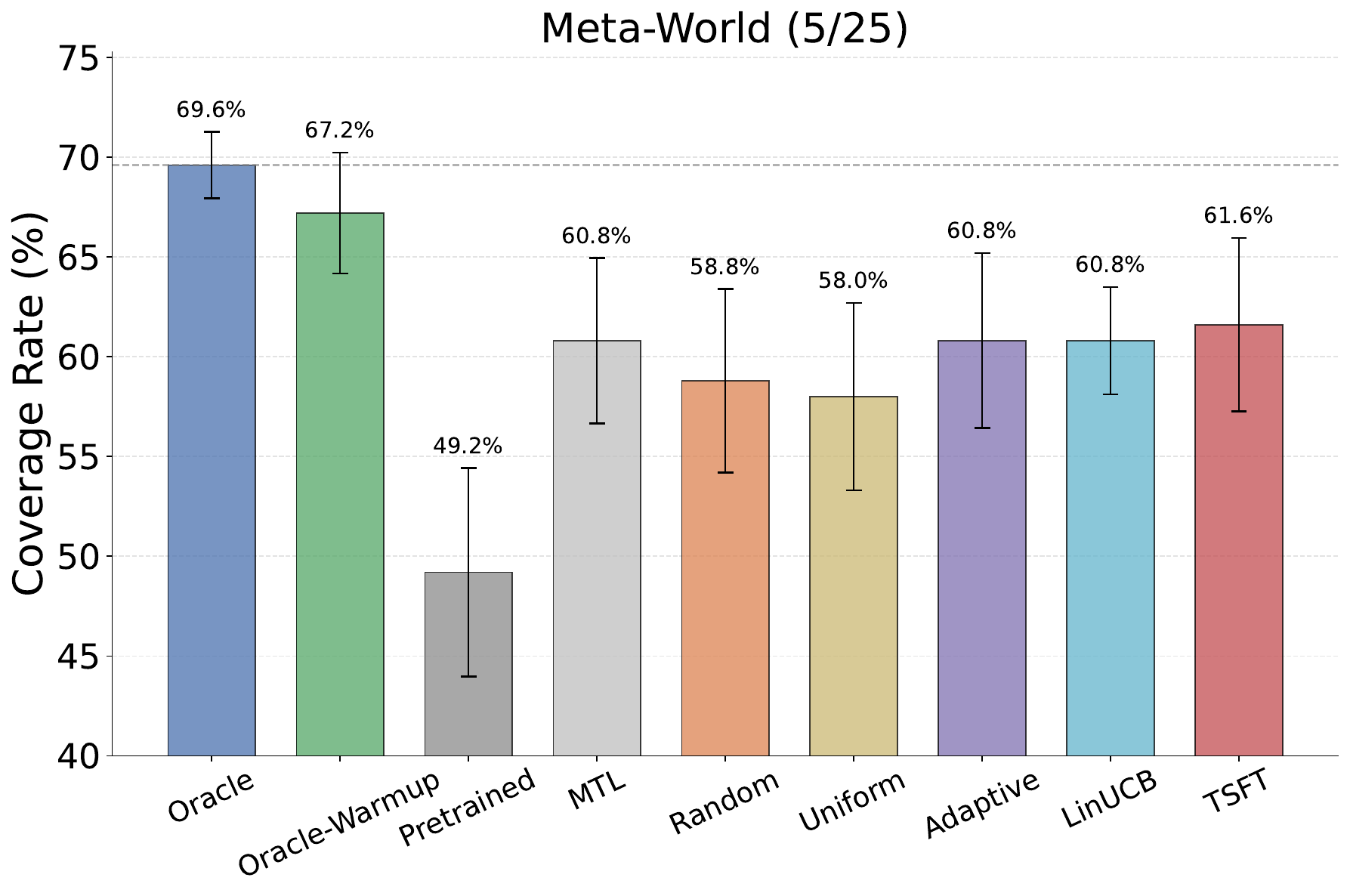}
    \caption{Coverage rate comparison on continuous control. Bars report mean coverage, and error bars denote standard deviation.}
    \label{app_figure_bar}
\end{figure*}

\begin{table*}[h]
  \caption{Full Results for CartPole.}
  \label{table_control_full}
  \begin{center}
  \begin{small}
  \renewcommand\arraystretch{1.25}  
  \resizebox{0.99\textwidth}{!}{ 
  \begin{tabular}{l|ccc|ccc|ccc|ccc|ccc}
    \toprule
     & \multicolumn{3}{c|}{\textbf{Trial 1}} 
     & \multicolumn{3}{c|}{\textbf{Trial 2}} 
     & \multicolumn{3}{c|}{\textbf{Trial 3}} 
     & \multicolumn{3}{c|}{\textbf{Trial 4}} 
     & \multicolumn{3}{c}{\textbf{Trial 5}} \\
     & 3/50 & 3/100 & 4/100 
     & 3/50 & 3/100 & 4/100 
     & 3/50 & 3/100 & 4/100 
     & 3/50 & 3/100 & 4/100 
     & 3/50 & 3/100 & 4/100 \\
    \midrule
    Oracle & 86.5\% & 86.8\% & 90.0\% & 89.7\% & 91.2\% & 91.5\% & 89.9\% & 90.5\% & 90.6\% & 90.0\% & 90.0\% & 90.2\% & 89.7\% & 90.0\% & 90.1\% \\
    Oracle-Warmup & 82.1\% & 82.4\% & 89.7\% & 83.7\% & 85.5\% & 91.5\% & 89.9\% & 90.5\% & 90.6\% & 88.5\% & 89.0\% & 90.0\% & 89.7\% & 89.9\% & 90.0\% \\
    Pretrained & 53.5\% & 53.5\% & 53.5\% & 83.1\% & 83.1\% & 83.1\% & 68.2\% & 68.2\% & 68.2\% & 60.5\% & 60.5\% & 60.5\% & 67.6\% & 67.6\% & 67.6\% \\
    \midrule
    MTL & 63.7\% & 73.0\% & 73.0\% & 83.1\% & 83.1\% & 83.1\% & 84.6\% & 88.4\% & 88.4\% & 60.5\% & 60.5\% & 60.5\% & 67.6\% & 67.6\% & 67.6\% \\
    Random & 60.2\% & 65.7\% & 79.5\% & 78.4\% & 77.5\% & \textbf{90.2\%} & 82.0\% & 87.6\% & 77.7\% & 75.6\% & 64.8\% & 76.9\% & 81.0\% & 79.4\% & 87.5\% \\
    Uniform & 55.9\% & 62.7\% & 80.4\% & 78.3\% & 63.9\% & 78.6\% & \textbf{86.4\%} & 87.4\% & \textbf{89.6\%} & 65.6\% & 40.0\% & 76.4\% & 77.9\% & 73.1\% & \textbf{88.9\%} \\
    Adaptive & 52.9\% & 59.8\% & 83.7\% & 78.4\% & 78.4\% & 89.0\% & 81.4\% & 83.5\% & 88.5\% & 77.8\% & 77.8\% & 86.6\% & 84.8\% & 84.8\% & 88.5\% \\
    LinUCB & 68.4\% & 70.3\% & 70.3\% & \textbf{88.8\%} & \textbf{89.2\%} & 89.2\% & 74.0\% & 81.6\% & 81.6\% & 74.2\% & \textbf{87.8\%} & \textbf{87.8\%} & 82.9\% & 84.2\% & 84.2\% \\
    \midrule
    TSFT & \textbf{74.3\%} & \textbf{74.5\%} & \textbf{86.2\%} & 83.3\% & 83.3\% & 89.7\% & 82.2\% & \textbf{89.8\%} & 85.3\% & \textbf{85.3\%} & 85.3\% & 87.7\% & \textbf{85.7\%} & \textbf{87.0\%} & \textbf{88.9\%} \\
    \bottomrule
  \end{tabular}}
  \end{small}
  \end{center}
\end{table*}

\begin{table*}[h]
  \caption{Full Results for Ant and Meta-World.}
  \label{table_control_full_ant_metaworld}
  \begin{center}
  \begin{small}
  \renewcommand\arraystretch{1.25}
  \resizebox{0.99\textwidth}{!}{
  \begin{tabular}{l|ccc|ccc|ccc|ccccc}
    \toprule
     & \multicolumn{3}{c|}{\textbf{Ant Trial 1}}
     & \multicolumn{3}{c|}{\textbf{Ant Trial 2}}
     & \multicolumn{3}{c|}{\textbf{Ant Trial 3}}
     & \multicolumn{5}{c}{\textbf{Meta-World 5/25}} \\
     & 3/50 & 3/100 & 4/100
     & 3/50 & 3/100 & 4/100
     & 3/50 & 3/100 & 4/100
     & Trial 1 & Trial 2 & Trial 3 & Trial 4 & Trial 5 \\
    \midrule
    Oracle & 100.0\% & 100.0\% & 100.0\% & 100.0\% & 100.0\% & 100.0\% & 100.0\% & 100.0\% & 100.0\% & 68.0\% & 72.0\% & 70.0\% & 70.0\% & 68.0\% \\
    Oracle-Warmup & 100.0\% & 100.0\% & 100.0\% & 100.0\% & 100.0\% & 100.0\% & 100.0\% & 100.0\% & 100.0\% & 62.0\% & 70.0\% & 68.0\% & 68.0\% & 68.0\% \\
    Pretrained & 93.4\% & 93.4\% & 93.4\% & 0.0\% & 0.0\% & 0.0\% & 0.0\% & 0.0\% & 0.0\% & 48.0\% & 48.0\% & 52.0\% & 56.0\% & 42.0\% \\
    \midrule
    MTL & \textbf{93.4\%} & 93.4\% & 93.4\% & 0.0\% & 0.0\% & 0.0\% & 0.0\% & 0.0\% & 0.0\% & \textbf{62.0\%} & 60.0\% & \textbf{64.0\%} & 64.0\% & 54.0\% \\
    Random & 21.6\% & 33.1\% & 60.2\% & 0.0\% & 74.2\% & 73.4\% & 25.8\% & 20.3\% & 3.5\% & 54.0\% & 58.0\% & 60.0\% & \textbf{66.0\%} & 56.0\% \\
    Uniform & 21.1\% & 28.0\% & 51.3\% & 20.2\% & \textbf{100.0\%} & 84.2\% & 35.7\% & 30.9\% & 58.7\% & 50.0\% & 58.0\% & 60.0\% & 60.0\% & \textbf{62.0\%} \\
    Adaptive & 33.1\% & 86.9\% & 94.6\% & 89.6\% & 89.6\% & \textbf{100.0\%} & 39.7\% & 87.4\% & 81.2\% & 54.0\% & 62.0\% & 60.0\% & \textbf{66.0\%} & \textbf{62.0\%} \\
    LinUCB & \textbf{93.4\%} & \textbf{96.7\%} & 96.7\% & 61.8\% & 71.9\% & 68.7\% & 37.6\% & 96.0\% & 65.4\% & 58.0\% & 58.0\% & 62.0\% & 64.0\% & \textbf{62.0\%} \\
    \midrule
    TSFT (DP) & 77.5\% & 93.0\% & \textbf{98.5\%} & \textbf{91.6\%} & \textbf{100.0\%} & \textbf{100.0\%} & \textbf{96.0\%} & \textbf{100.0\%} & \textbf{100.0\%} & 58.0\% & \textbf{68.0\%} & 60.0\% & 64.0\% & 58.0\% \\
    \bottomrule
  \end{tabular}}
  \end{small}
  \end{center}
\end{table*}

\subsection{A Comprehensive Study of Performance Model}
\label{app_model_study}
TSFT employs an exponential model to predict performance, which implicitly assumes that fine-tuning trajectories are approximately monotonic. While RL fine-tuning is inherently stochastic and no surrogate model can perfectly predict the entire training trajectory, effective budget allocation in TSFT does not require highly accurate performance prediction. Instead, it is often sufficient to capture the coarse trend of the learning curve. This observation motivates our monotonicity assumption and the use of a simple parametric model that emphasizes the global trajectory rather than short-term fluctuations. Below, we provide empirical evidence and discussion supporting this design choice.

\begin{wraptable}{r}{0.49\textwidth}
  \caption{Approximate Monotonicity Rate.}
  \label{table_monotonicity}
  \begin{center}
  \vspace{-5pt}
  \renewcommand\arraystretch{1.0}  
  \resizebox{0.49\textwidth}{!}{ 
  \begin{tabular}{l|c|c|c|c}
  \toprule
     & CVRP & CVRPTW & CartPole & Ant \\
    \midrule
    5 Points & 30.2\% & 30.7\% & 28.2\% & 3.3\% \\
    10 Points & 1.9\% & 0.8\% & 8.3\% & 0.0\% \\
    \bottomrule
  \end{tabular}}
  \end{center}
  \vskip -0.1in
\end{wraptable}

First, \emph{the monotonic model remains effective even when the actual training trajectories are noisy and locally non-monotonic}. Across evaluated domains, the observed RL learning curves exhibit substantial fluctuations. We quantify this phenomenon by uniformly sampling either 5 or 10 points from each training trajectory and measuring its monotonicity (Table~\ref{table_monotonicity}). With 10 sampled points, nearly all trajectories exhibit local non-monotonic behavior. With only 5 sampled points, the measured monotonicity increases because coarse sampling filters out high-frequency noise. Despite these local deviations, TSFT consistently achieves strong improvements, indicating that a simple monotonic model is sufficient to capture the broader trend required for effective budget allocation.

Second, \emph{monotonicity provides a simple and robust inductive bias  that reduces overfitting under limited observations}. We compare three monotonic models (Exponential, Power-Law, and Logarithmic) with four non-monotonic models (Quadratic, Cubic, Quartic, and Piecewise) on the relatively monotonic CVRP domain and the more challenging, non-monotonic Ant domain. We evaluate each model in terms of fitting error on the observed data, prediction error on future training points, and the resulting coverage achieved by TSFT (Table~\ref{table_model_error}). Overall, the monotonic models perform consistently well across both domains. While some non-monotonic models remain competitive on Ant, they perform noticeably worse on CVRP. Moreover, although the more flexible models often achieve lower in-distribution (ID) fitting error, they generally incur substantially larger out-of-distribution (OOD) prediction error, suggesting that they overfit the observed noise and produce unstable extrapolations.

\begin{table*}[h]
  \caption{Comparison of different parametric models. We report the TSFT coverage and the average in-distribution (ID) and out-of-distribution (OOD) mean squared errors (MSEs).}
  \label{table_model_error}
  \centering
  \begin{small}
  \renewcommand{\arraystretch}{1.0}
  \setlength{\tabcolsep}{5.5pt}
  \resizebox{0.99\textwidth}{!}{
  \begin{tabular}{ll|ccc|ll|ccc}
    \toprule
    \multicolumn{5}{c|}{\textbf{CVRP 3/100}}
    & \multicolumn{5}{c}{\textbf{Ant 3/50}} \\
    \cmidrule(lr){1-5}\cmidrule(lr){6-10}
    \textbf{Type} & \textbf{Model}
    & \textbf{Coverage}
    & \textbf{Avg. ID MSE}
    & \textbf{Avg. OOD MSE}
    & \textbf{Type} & \textbf{Model}
    & \textbf{Coverage}
    & \textbf{Avg. ID MSE}
    & \textbf{Avg. OOD MSE} \\
    \midrule

    \multirow{3}{*}{Monotonic}
    & Exponential
    & 18.7\% & 0.027682 & 0.103628
    & \multirow{3}{*}{Monotonic}
    & Exponential
    & 88.4\% & 0.000694 & 0.001309 \\

    & Power-Law
    & 19.7\% & 0.065287 & 0.192049
    &
    & Power-Law
    & 88.4\% & 0.000702 & 0.001048 \\

    & Logarithmic
    & 19.0\% & 0.063002 & 0.179515
    &
    & Logarithmic
    & 90.7\% & 0.000760 & 0.001076 \\
    \midrule

    \multirow{4}{*}{Non-monotonic}
    & Quadratic
    & 8.1\% & 0.027072 & 0.565207
    & \multirow{4}{*}{Non-monotonic}
    & Quadratic
    & 85.2\% & 0.000625 & 0.012584 \\

    & Cubic
    & 7.1\% & 0.015860 & 0.585428
    &
    & Cubic
    & 89.7\% & 0.000628 & 1.969833 \\

    & Quartic
    & 7.4\% & 0.015324 & 11.138788
    &
    & Quartic
    & 81.1\% & 0.000475 & 11.286251 \\

    & Piecewise
    & 15.3\% & 0.015488 & 0.147162
    &
    & Piecewise
    & 70.5\% & 0.000514 & 0.010115 \\
    \bottomrule
  \end{tabular}}
  \end{small}
\end{table*}

Third, \emph{explicitly modeling high-frequency fluctuations  provides limited practical benefit}. Beyond purely parametric models, we also evaluate a hybrid parametric Gaussian process model (PFGP), in which the parametric component captures the global trend, while a non-parametric Gaussian process models residual fluctuations, such as sudden performance jumps.
We conduct experiments on the 3/100 CVRP setting by fitting models to data collected from policies trained on the source task set $X_{S_1}$. As shown in Fig.~\ref{fig_model}, our parametric model consistently achieves higher recall than GP across future budget units, indicating that the simple exponential form provides more reliable extrapolation of fine-tuning trajectories. In contrast, the GP baseline exhibits unstable long-horizon predictions and large uncertainty when extrapolating beyond the observed budget range, which leads to more false negatives in the predicted coverage set. 
PFGP improves over GP by using the parametric function as a global trend and modeling only the residual variation, and its recall is often close to our parametric model. However, it does not provide a clear advantage over the parametric model, while incurring a significantly higher computational cost due to GP fitting. Therefore, we do not adopt PFGP in our main experiments. The bottom panels further confirm our observation: our model recovers a substantially larger portion of the ground-truth coverage set than GP, whereas PFGP achieves comparable recall without delivering a meaningful improvement.

Based on the above evidence, TSFT adopts a simple parametric model to capture the global trend of fine-tuning. Nevertheless, we acknowledge that the exponential model is not universally applicable. When the underlying fine-tuning trajectory exhibits strongly non-monotonic global behavior, the model may produce suboptimal budget allocations. Promising directions for future work include uncertainty-aware exploration strategies and more expressive surrogate models (e.g., neural networks) with stronger extrapolation capabilities.

\begin{figure*}[t]
    \centering
    \includegraphics[width=0.322\columnwidth]{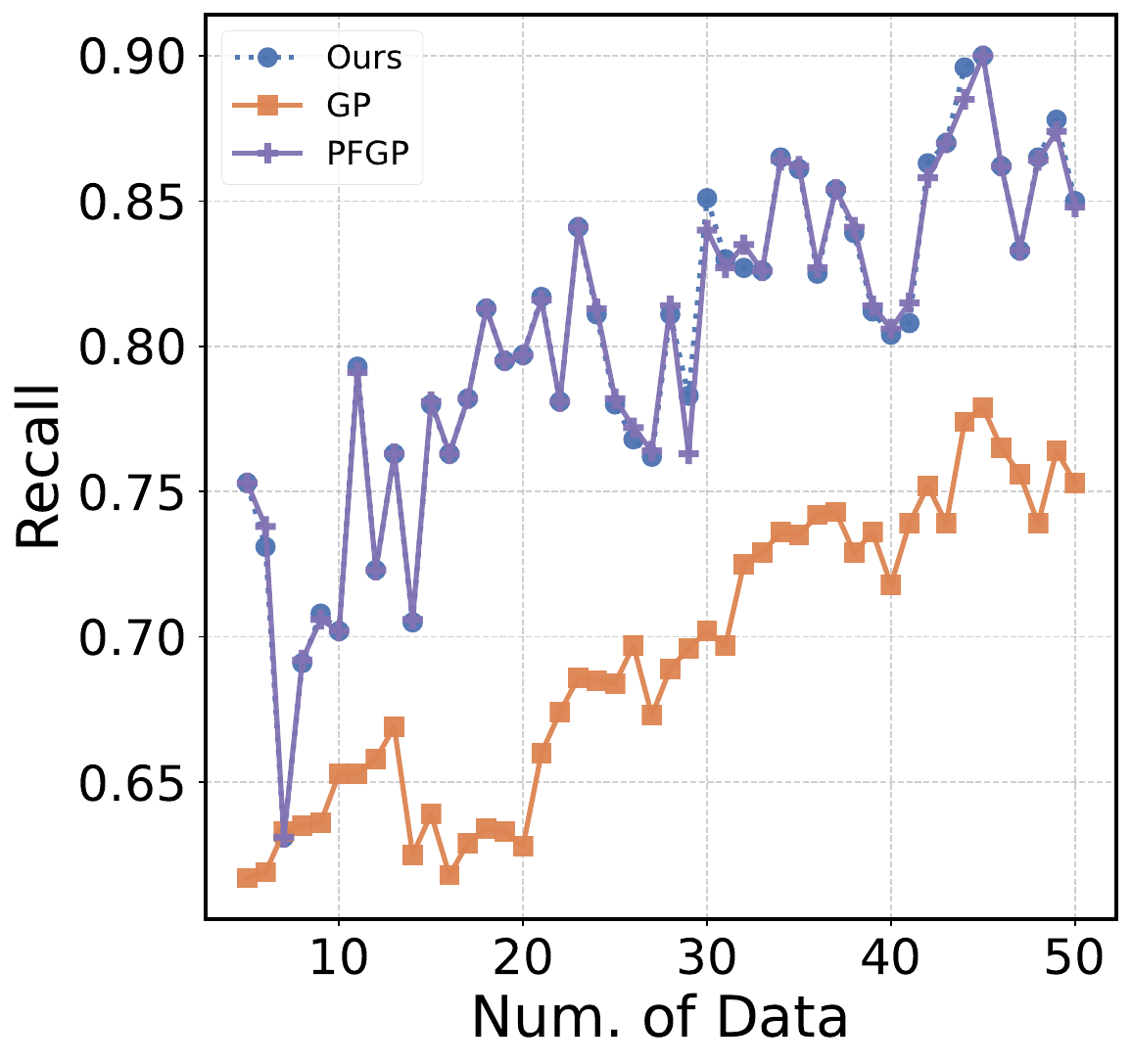}
    \includegraphics[width=0.3\columnwidth]{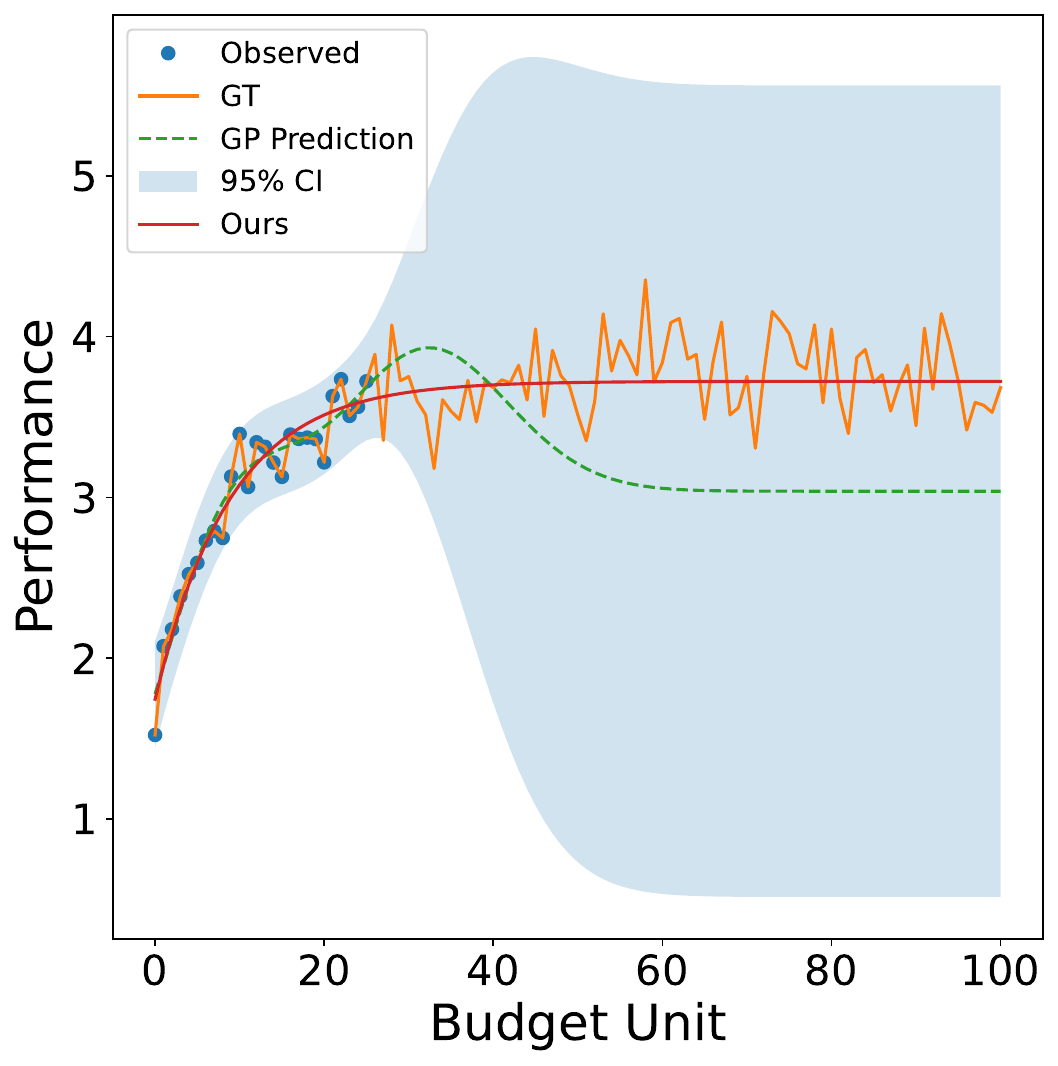}
    \includegraphics[width=0.3\columnwidth]{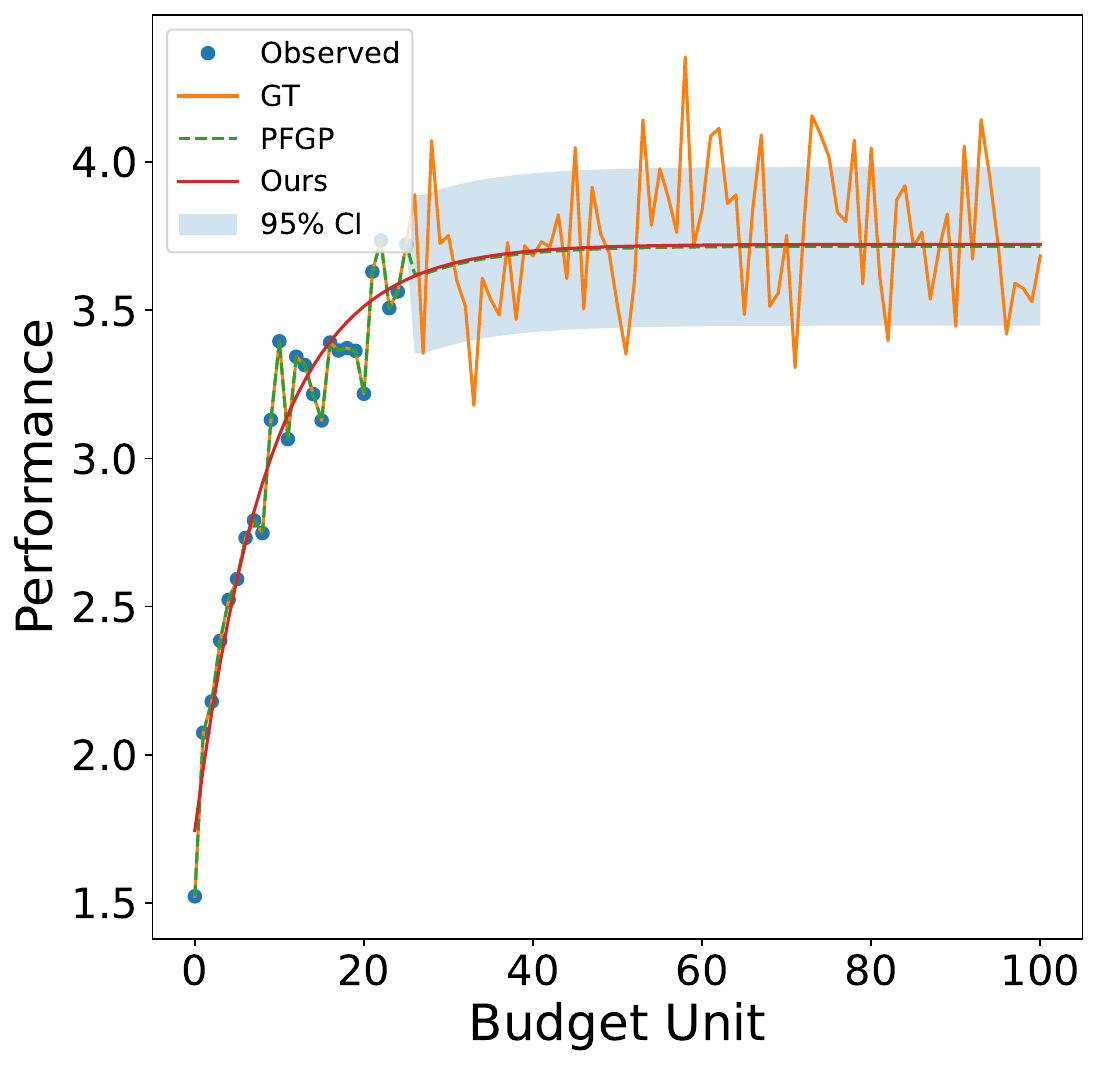}\\
    \includegraphics[width=0.3\columnwidth]{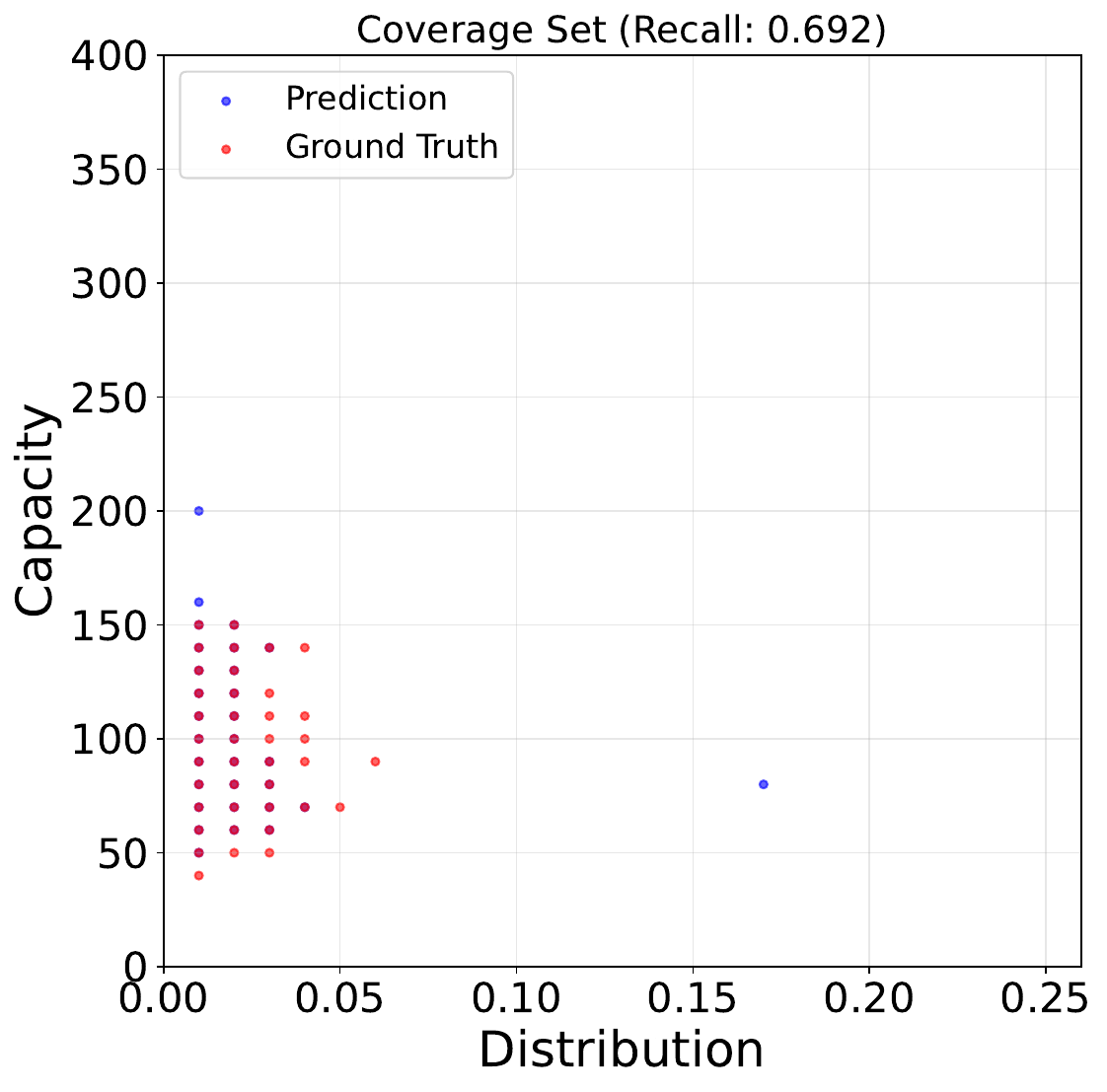}
    \includegraphics[width=0.3\columnwidth]{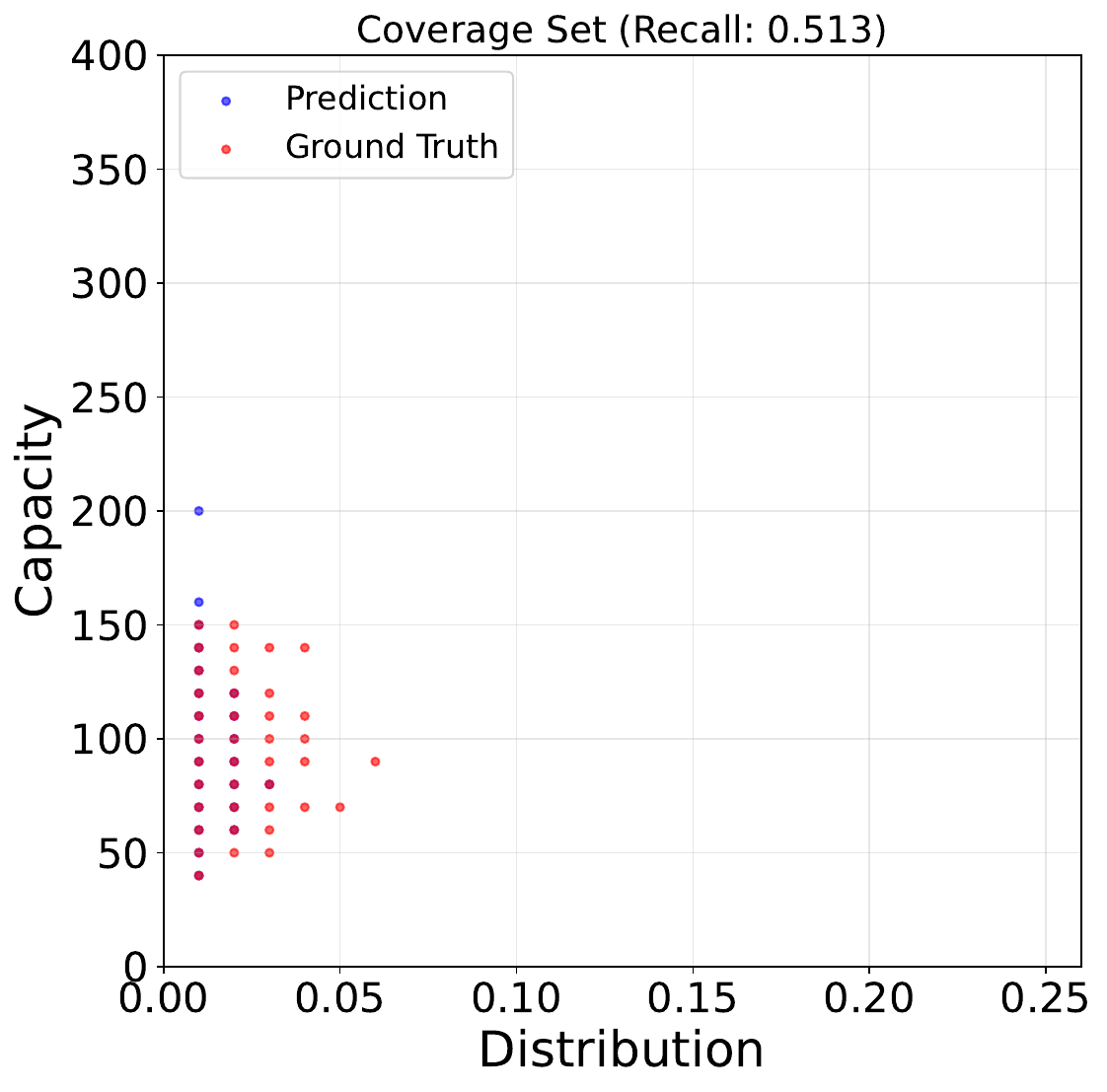}
    \includegraphics[width=0.3\columnwidth]{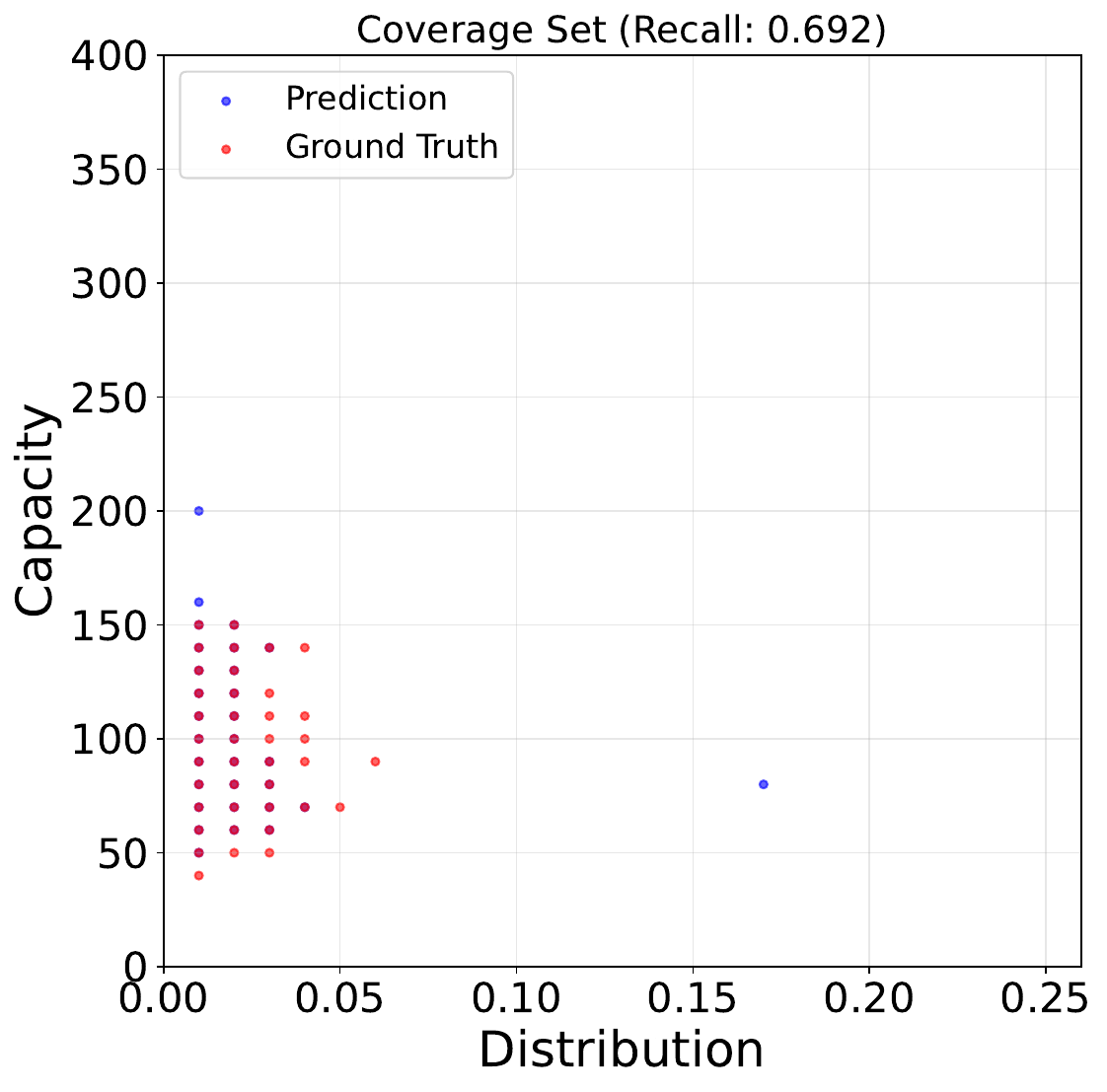}
    \caption{\emph{Top left panel:} Comparison of recall rates of coverage sets.
    \emph{Top right panels:} Comparison of fitted models on a single task.
    \emph{Bottom panels:} Predicted coverage sets from three modeling approaches (i.e., Ours, GP, and PFGP) after 50 additional budget units. \colorb{\textbf{Blue}}, \colorr{\textbf{red}}, and \textcolor[rgb]{0.55,0.0,0.0}{\textbf{dark red}} points denote false positives (FP), false negatives (FN), and true positives (TP), respectively.
    }
    \label{fig_model}
    \vskip -0.1in
\end{figure*}

\subsection{Sensitivity Analysis}
We tune several key hyperparameters, including the performance threshold $\epsilon$, the number of policies $N$, the warmup budget $W$, and the execution budget $E$. The tuning procedure is based on a grid search over predefined ranges selected according to computational feasibility and empirical robustness.



\begin{figure*}[!ht]
    \centering
    \includegraphics[width=0.4\columnwidth]{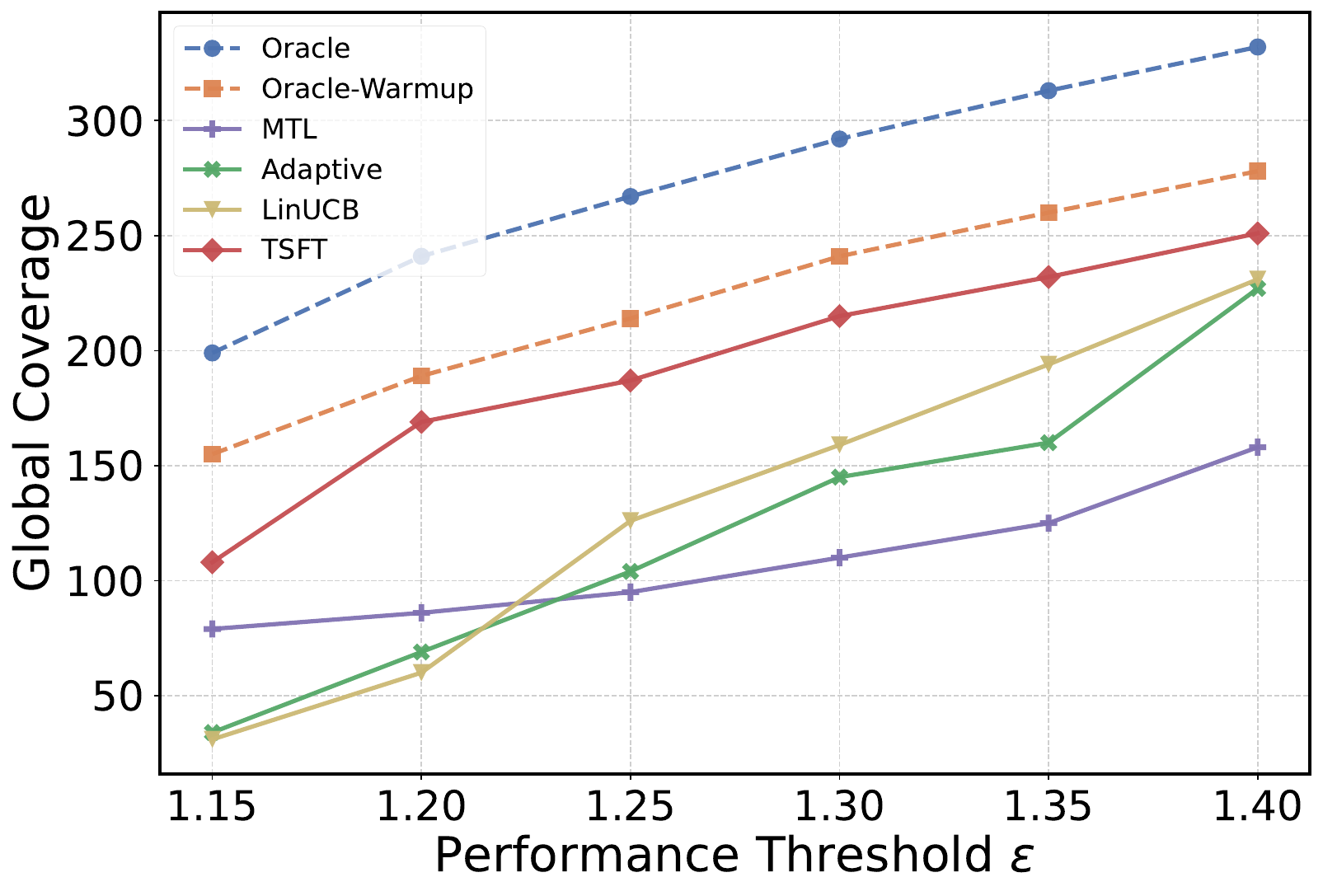}
    \includegraphics[width=0.27\columnwidth]{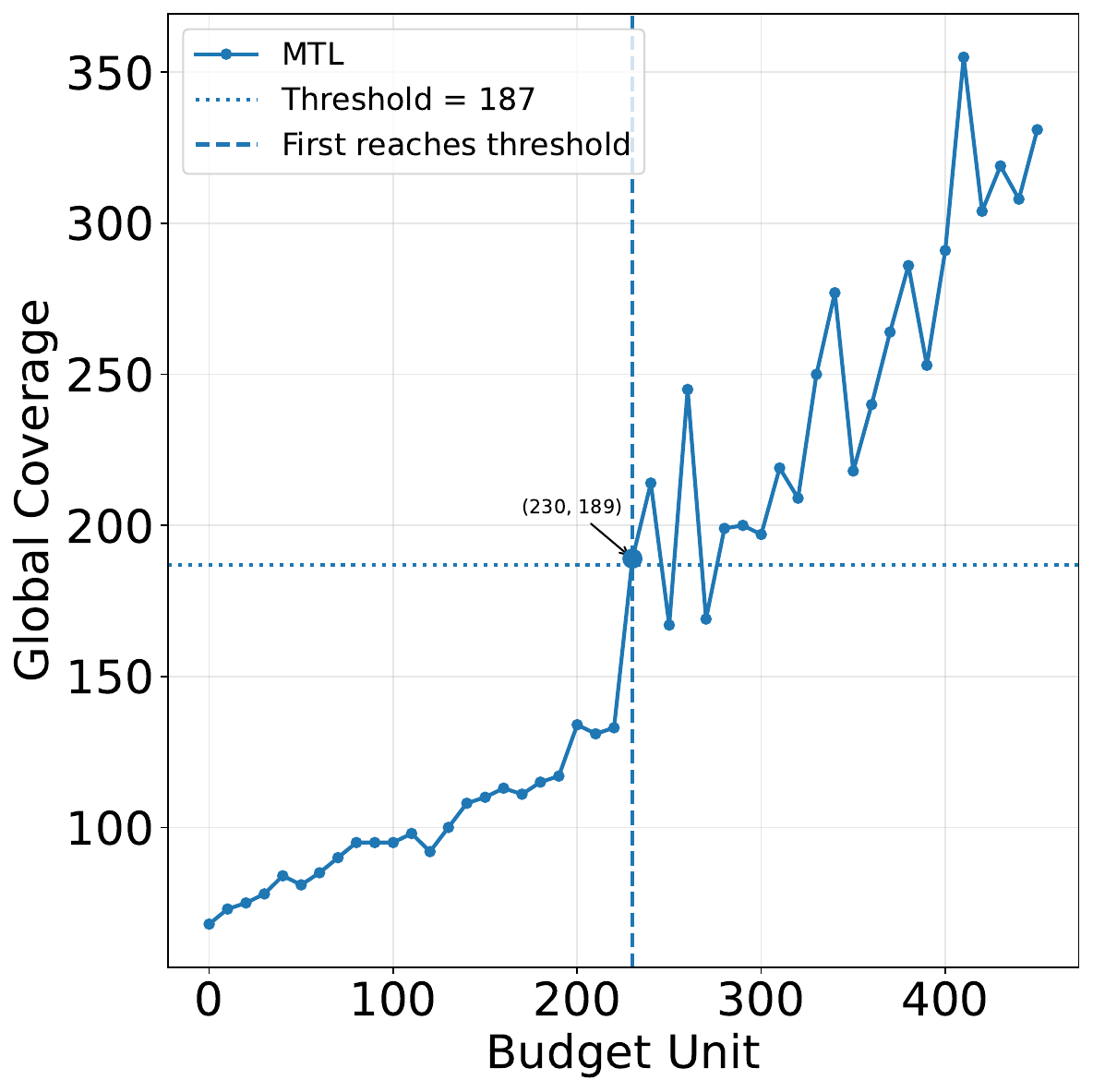}
    \includegraphics[width=0.27\columnwidth]{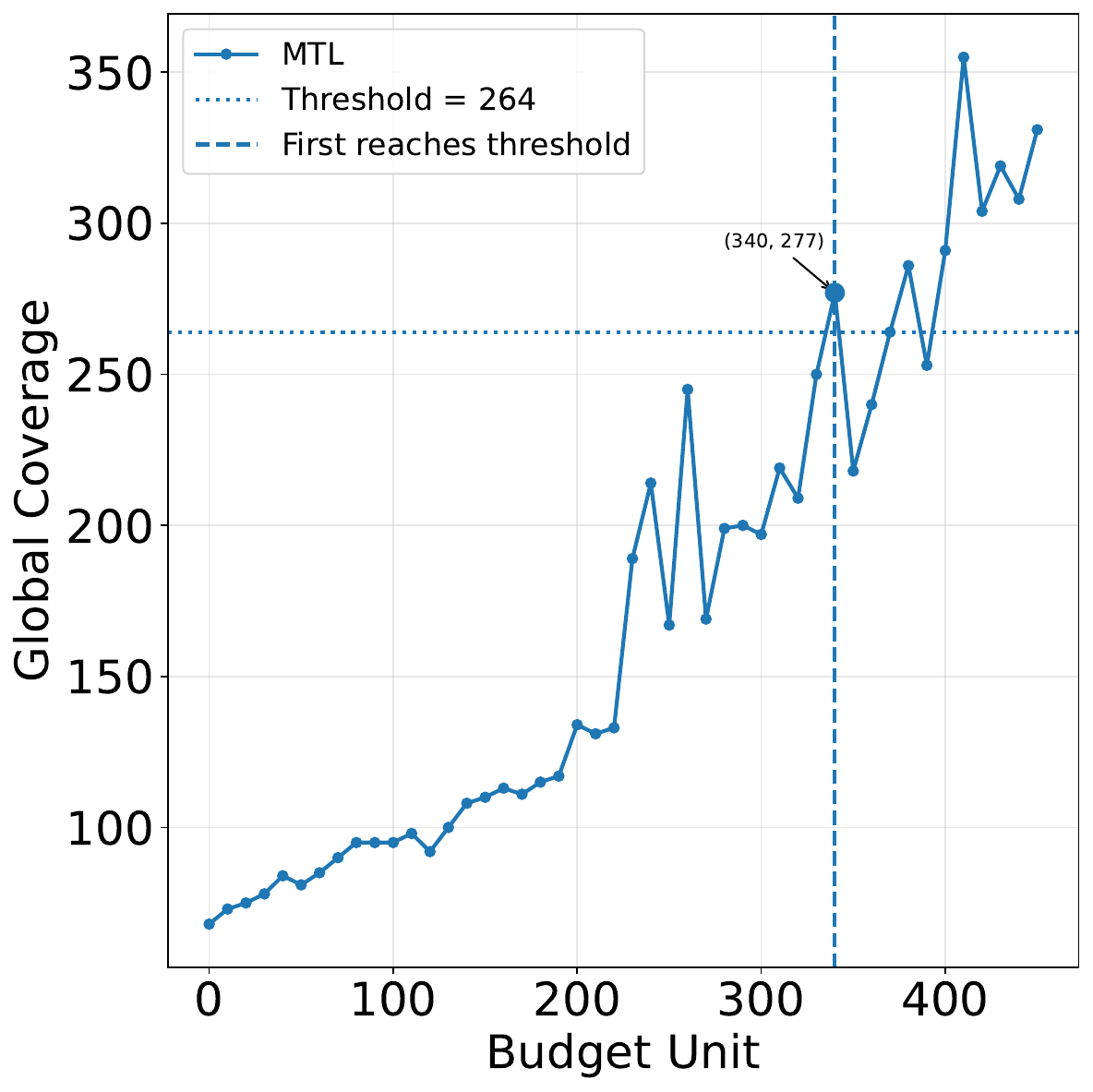}
    \caption{\emph{Left panel:} Sensitivity test of performance threshold $\epsilon$.
    \emph{Right panels:} Sample efficiency comparison between TSFT and MTL on CVRP. We use the best TSFT results reported in Table~\ref{table_nco_control} as thresholds, namely 187 and 264 for budgets of 100 and 150 units, respectively.}
    \label{app_hypara}
    \vskip -0.1in
\end{figure*}

\textbf{Performance Threshold.} We conduct a sensitivity analysis on the performance threshold $\epsilon$ in the 3/100 CVRP setting, varying $\epsilon$ from $1.15\%$ to $1.4\%$ with increments of $0.05\%$. As shown in Fig.~\ref{app_hypara}, the global coverage of all methods increases as $\epsilon$ becomes larger, since a looser threshold allows more tasks to be counted as covered. 
TSFT achieves competitive and stable performance over a wide range of $\epsilon$, consistently outperforming standard baselines.

\textbf{Number of Policies.} We conduct a scalability test on the CVRP setting with a fixed budget of $K=150$, while varying the number of policies from $N \in \{3,4,5,6,7\}$. The results in Fig.~\ref{app_hypara1} show that TSFT scales substantially better than the baseline methods as the number of policies increases. Most baselines exhibit a clear performance drop as $N$ becomes larger, suggesting that they struggle to effectively allocate the fixed budget across an increasing number of policies. Although LinUCB is competitive when $N$ is small, its performance deteriorates sharply for larger $N$. 
Overall, these results demonstrate that TSFT is more effective at managing the increased allocation complexity induced by a larger policy portfolio. At the same time, they also reveal a fundamental trade-off: when the total budget is fixed, an excessively large $N$ eventually reduces performance because each policy receives a smaller specialization budget (i.e., budget dilution). Consequently, TSFT would be expected to achieve higher coverage if the per-policy specialization budget were held constant.

\textbf{Warmup Budget.} We conduct a sensitivity analysis on the warmup budget $W$ in the 3/100 and 3/150 CVRP settings, varying the warmup budget per policy from 3 to 15. As shown in Fig.~\ref{app_hypara1}, TSFT generally achieves stronger performance with relatively small warmup budgets, whereas an excessively large warmup budget may reduce the budget available for subsequent adaptive specialization, potentially leading to degraded global coverage. Note that the warmup budget introduces suboptimality only when it allocates more budget units to some policy than required by the optimal allocation under the true objective.

\begin{figure*}[ht]
    \centering
    \includegraphics[width=0.32\columnwidth]{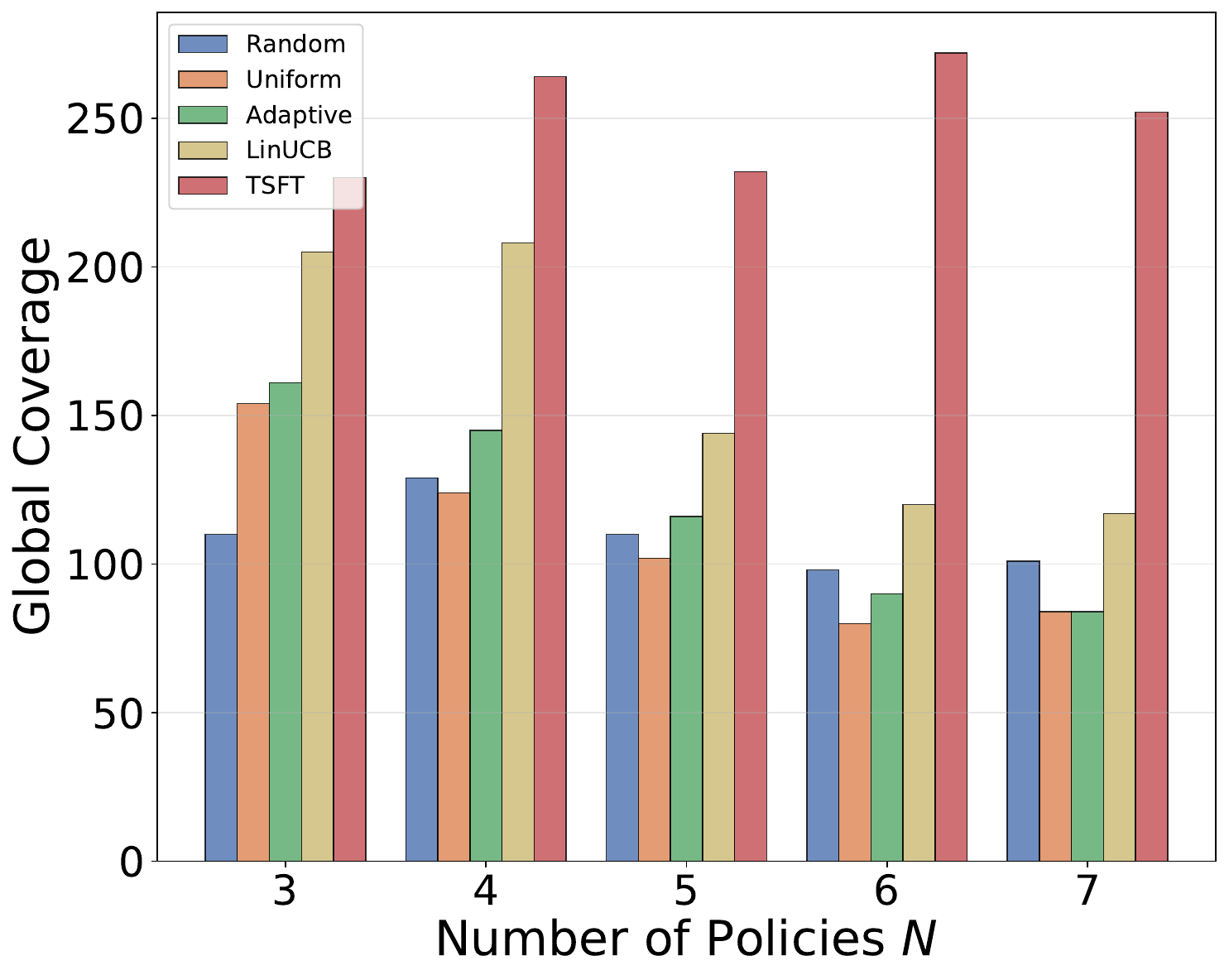}
    \includegraphics[width=0.32\columnwidth]{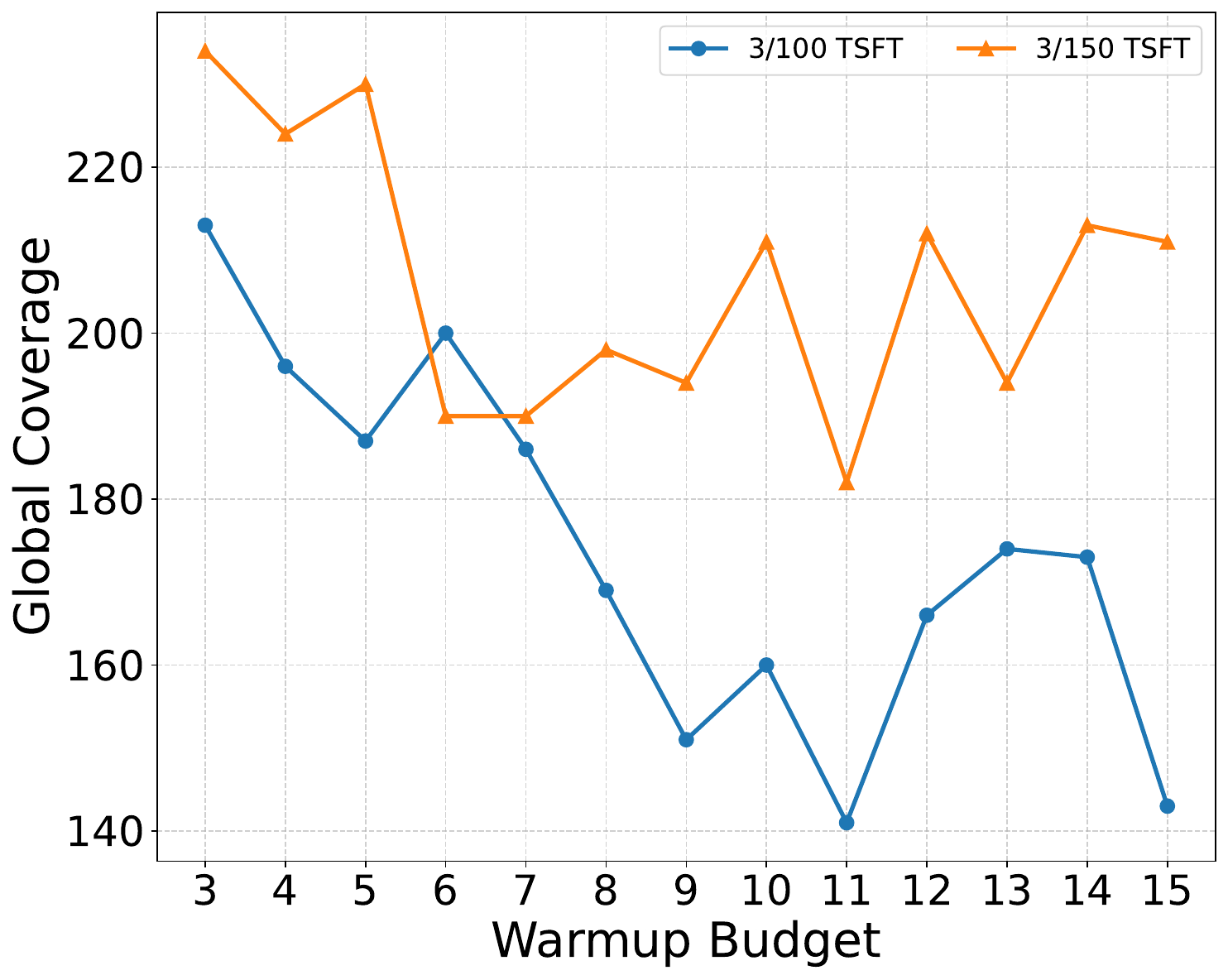}
    \includegraphics[width=0.32\columnwidth]{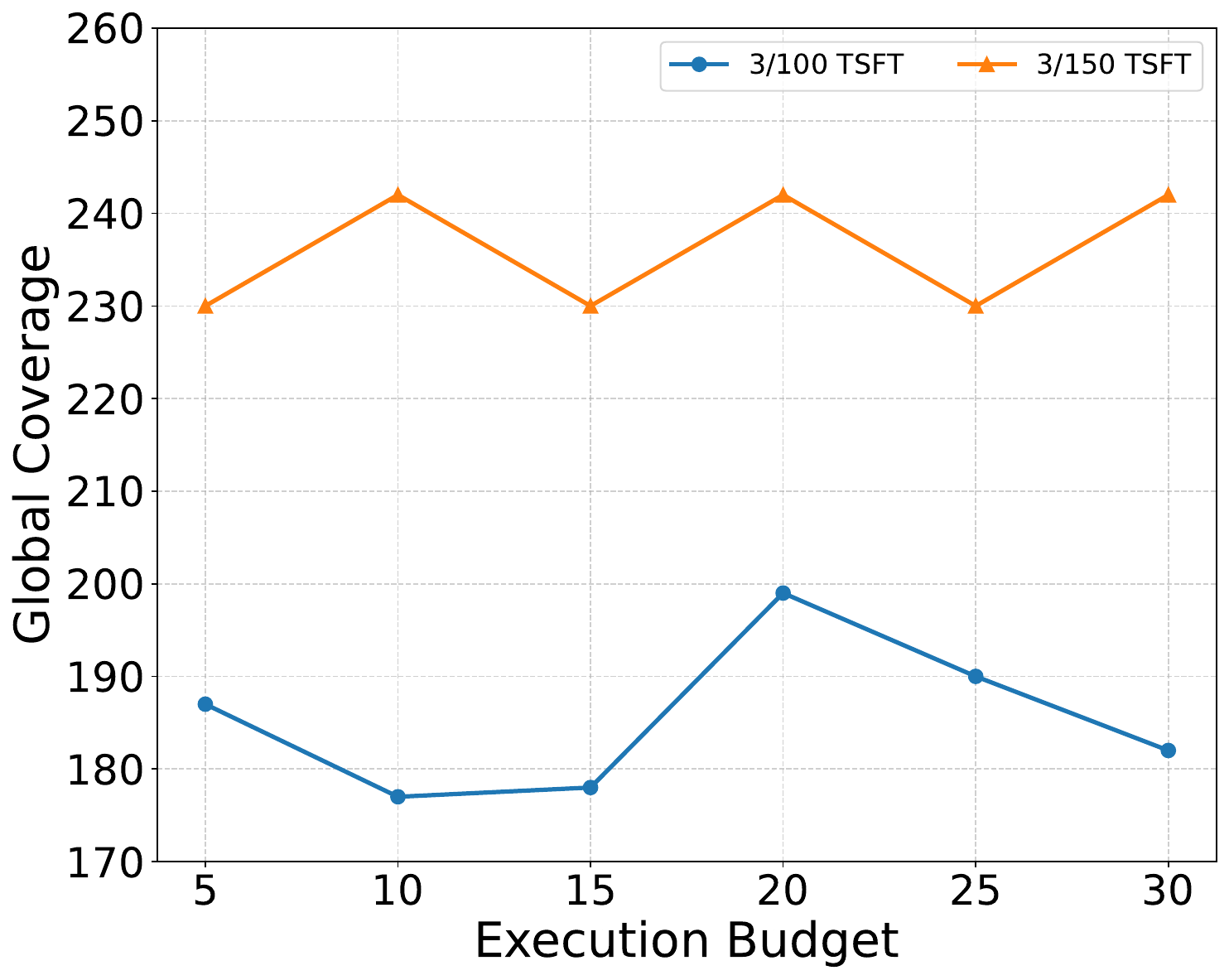}
    \caption{\emph{Left panel:} Performance comparison under different number of policies.
    \emph{Middle panel:} Sensitivity test of warmup budget.
    \emph{Right panel:} Sensitivity test of execution budget.
    }
    \label{app_hypara1}
\end{figure*}

\textbf{Execution Budget.} We conduct a sensitivity analysis on the execution budget $E$ in the 3/100 and 3/150 CVRP settings, varying the execution budget unit from 5 to 30. As shown in Fig.~\ref{app_hypara1}, TSFT remains relatively stable across different execution budget settings, indicating that the method is not overly sensitive once a reasonable execution budget is used. This also suggests that TSFT may further robustly benefit from adaptive budget allocation without requiring extensive tuning of $E$.

\begin{wraptable}{r}{0.52\textwidth}
  \caption{Sensitivity Test of Source Task Set.}
  \label{table_source_task}
  \begin{center}
  \vspace{-5pt}
  \renewcommand\arraystretch{1.0}  
  \resizebox{0.52\textwidth}{!}{ 
  \begin{tabular}{l|c|c|c}
  \toprule
    CVRP 3/100 & Adaptive & LinUCB & TSFT (DP) \\
    \midrule
    $(X_{S_1}, X_{S_3}, X_{S_5})$ & 10.4\% & 12.6\% & \textbf{18.7\%} \\
    $(X_{S_1}, X_{S_2}, X_{S_3})$ & 7.1\% & 8.1\% & \textbf{11.4\%} \\
    $(X_{S_2}, X_{S_3}, X_{S_4})$ & 6.4\% & 6.1\% & \textbf{12.3\%} \\
    $(X_{S_3}, X_{S_4}, X_{S_5})$ & 9.0\% & 5.7\% & \textbf{16.9\%} \\
    $(X_{S_1}, X_{S_2}, X_{S_5})$ & 10.5\% & 11.7\% & \textbf{21.8\%} \\
    $(X_{S_1}, X_{S_4}, X_{S_5})$ & 7.4\% & 7.3\% & \textbf{18.4\%} \\
    $(X_{S_2}, X_{S_3}, X_{S_5})$ & 6.3\% & 9.4\% & \textbf{17.0\%} \\
    $(X_{S_2}, X_{S_4}, X_{S_5})$ & 6.9\% & 9.8\% & \textbf{20.1\%} \\
    \bottomrule
  \end{tabular}}
  \end{center}
  \vskip -0.1in
\end{wraptable}
\textbf{Source Task Set.} A substantial body of prior CRL research focuses on where to train (e.g., how to select representative source tasks). Our work addresses a complementary question of how much to fine-tune each given region. Nevertheless, We conduct additional experiments in the 3/100 CVRP setting using different source-task configurations to investigate TSFT's sensitivity to the choice of source-task set. Specifically, we define the five candidate regions illustrated in Fig.~\ref{app_fig_context_space} and evaluate different selections of three regions. As shown in Table~\ref{table_source_task}, TSFT remains effective across the tested configurations and consistently identifies high-quality budget allocations, demonstrating that our TSFT is largely robust to different choices of source tasks.

\subsection{Computational Cost}
\label{app_computation}
We provide a detailed comparison of the computational costs of the ILP- and DP-based methods in Table~\ref{tab:computational_cost}.
The exact DP solver scales poorly with the number of policies N, quickly becoming impractical as N increases. Although RH-DP substantially reduces both runtime and memory consumption, it remains computationally prohibitive for large-scale planning over larger policy sets.
In contrast, ILP solves the problem exactly while remaining computationally efficient. Unlike DP, which explicitly enumerates and caches nearly all reachable allocation states, ILP provides a compact formulation that enables the solver to exploit LP relaxations and branch-and-bound pruning to eliminate large portions of the search space. 
Note that the slight non-monotonicity in ILP runtime arises because branch-and-bound complexity depends on instance-specific structure and pruning effectiveness rather than solely on \(N\) and \(K\).
Consequently, ILP is often faster in practice, though both formulations remain exponential in the worst case.

\begin{table*}[h]
\centering
\caption{Computational cost comparison of TSFT under different numbers of policies $N$ and total budgets $K$. Each entry reports runtime and peak memory usage per decision-making step.}
\label{tab:computational_cost}
\begin{small}
\renewcommand{\arraystretch}{1.5}
\setlength{\tabcolsep}{4pt}
\resizebox{\textwidth}{!}{
\begin{tabular}{llcccccc}
\toprule
Setting & Method & \multicolumn{6}{c}{Runtime / Memory} \\
\midrule
\multicolumn{8}{c}{\textbf{Varying the number of policies $N$ with fixed budget $K=100$}} \\
\midrule
$K=100$ &  & $N=2$ & $N=3$ & $N=4$ & $N=5$ & $N=6$ & $N=7$ \\
\cmidrule(lr){3-8}
 & TSFT (DP)    
 & 0.04s / 2.5MB 
 & 2.1s / 47.8MB 
 & 47.0s / 771.8MB 
 & 642.3s / 9476.2MB 
 & -- 
 & -- \\
 & TSFT (RH-DP) 
 & 0.02s / 1.0MB 
 & 0.4s / 8.4MB 
 & 3.3s / 65.2MB 
 & 25.2s / 406.4MB 
 & 136.8s / 2186.9MB
 & 717.2s / 10418.9MB \\
 & TSFT (ILP)
 & 0.08s / 2.2MB
 & 0.2s / 6.1MB
 & 0.5s / 8.7MB
 & 0.5s / 9.0MB
 & 0.5s / 10.3MB 
 & 0.4s / 11.0MB \\
\midrule
\multicolumn{8}{c}{\textbf{Varying the total budget $K$ with fixed number of policies $N=3$}} \\
\midrule
$N=3$ &  & $K=50$ & $K=100$ & $K=150$ & $K=200$ & $K=250$ & $K=300$ \\
\cmidrule(lr){3-8}
 & TSFT (DP)    
 & 0.11s / 2.8MB 
 & 2.1s / 47.8MB 
 & 11.1s / 228.3MB 
 & 34.3s / 673.1MB 
 & 78.7s / 1580.2MB 
 & 165.1s / 3197.5MB \\
 & TSFT (RH-DP) 
 & 0.08s / 2.3MB 
 & 0.4s / 8.4MB 
 & 0.6s / 13.7MB 
 & 0.9s / 20.0MB 
 & 1.2s / 25.0MB 
 & 1.5s / 32.6MB \\
 & TSFT (ILP)
 & 0.09s / 2.2MB
 & 0.2s / 6.1MB
 & 0.2s / 11.4MB
 & 0.5s / 11.4MB
 & 0.2s / 11.4MB
 & 0.2s / 11.4MB \\
\bottomrule
\end{tabular}
}
\end{small}
\vskip -0.1in
\end{table*}

\subsection{Cheap Evaluation}
\label{app:cheap_eval}
\begin{wraptable}{r}{0.52\textwidth}
  \caption{Results for Cheap Evaluation.}
  \label{table_cheap_eval}
  \begin{center}
  \vspace{-5pt}
  \renewcommand\arraystretch{1.0}  
  \resizebox{0.52\textwidth}{!}{ 
  \begin{tabular}{l|c|c}
  \toprule
    TSFT (DP) & CVRP 3/100 & CVRP 3/150 \\
    \midrule
    Original & 18.7\% & 23.0\% \\
    Reduced Evaluation (10x) & 18.9\% & 22.3\% \\
    \bottomrule
  \end{tabular}}
  \end{center}
\end{wraptable}

TSFT periodically evaluates policies across the context space. In our experiments, the total evaluation time is on the order of hours, which is substantially smaller than the days required for policy training. Nevertheless, exhaustive evaluation may be impractical in domains with expensive evaluations. To mitigate this, we may reduce the online evaluation cost required for model fitting by using fewer validation samples or evaluating less frequently with a larger execution interval. As shown in Table \ref{table_cheap_eval}, using one-tenth as many samples for online evaluation yields comparable coverage to the original setting. Final global coverage is still calculated using the full dataset. More advanced methods, such as approximating the evaluation results through selective evaluation, constitute an interesting direction for future research.

\subsection{Visualization}
\label{app_visual}
We visualize the performance heatmap of each policy to illustrate task specialization. Specifically, we consider the final policies in the 5/150 CVRP setting, where specialization is guided by TSFT.
Since TSFT allocates different budgets across policies, some remain at early stages of fine-tuning, whereas others (e.g., the last policy) progress to later stages.

As shown in Fig.~\ref{app_visua}, each policy achieves low optimality gaps primarily around the region associated with its corresponding source task set, while its performance generally degrades when moving farther away from that region. This indicates that TSFT encourages different policies to specialize in distinct subregions of the context space rather than forcing a single policy to perform uniformly well across all tasks. More importantly, the specialized regions are complementary across policies, suggesting that the final policy portfolio can provide broader global coverage through task-wise policy selection.
This effect is further illustrated in the bottom row of Fig.~\ref{app_visua}. Although the pretrained and MTL policies achieve reasonably good average performance across the context space, their global coverage remains limited because they lack sufficient specialization.

\begin{figure*}[!h]
    \centering
    \includegraphics[width=0.19\columnwidth]{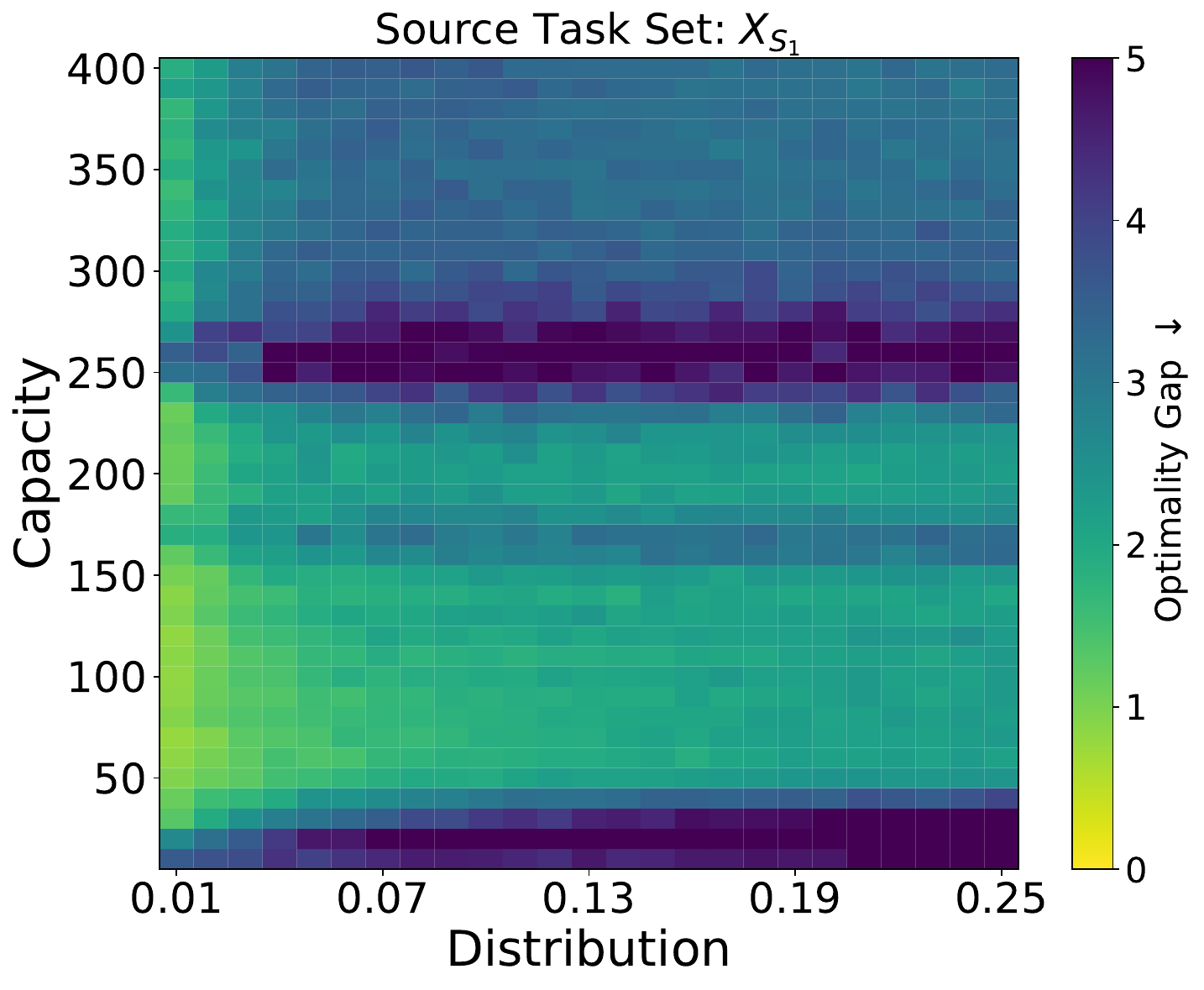}
    \includegraphics[width=0.19\columnwidth]{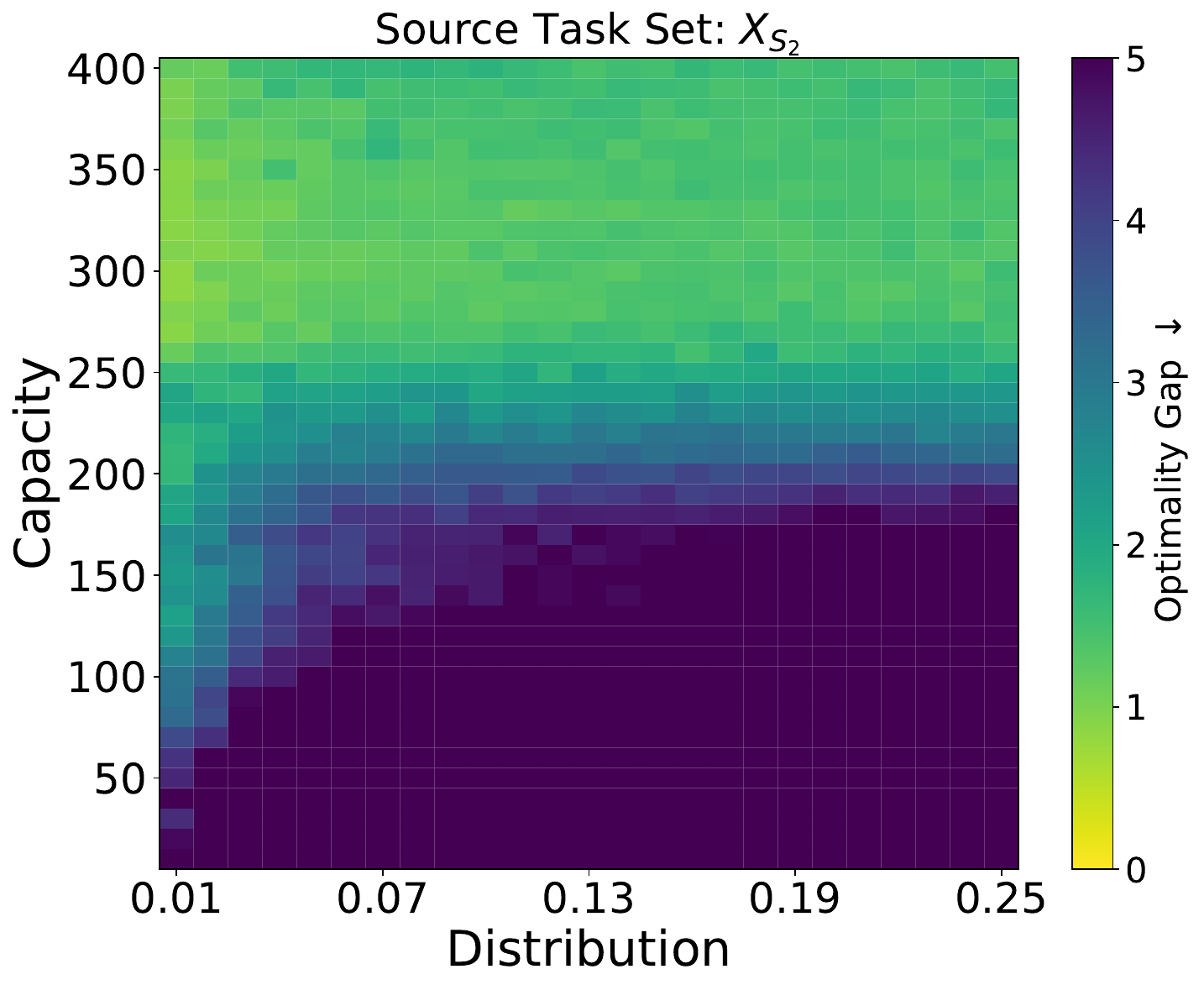}
    \includegraphics[width=0.19\columnwidth]{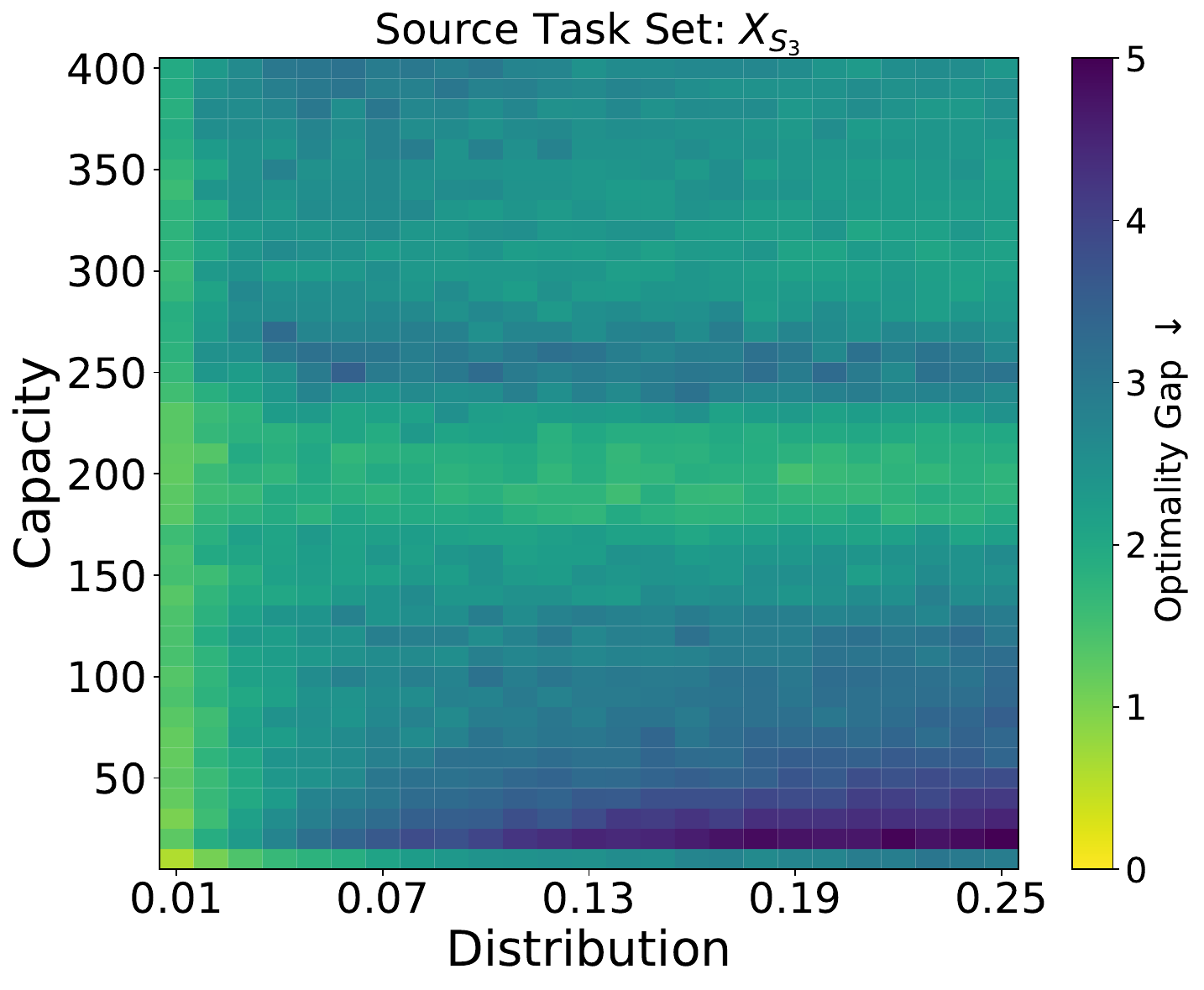}
    \includegraphics[width=0.19\columnwidth]{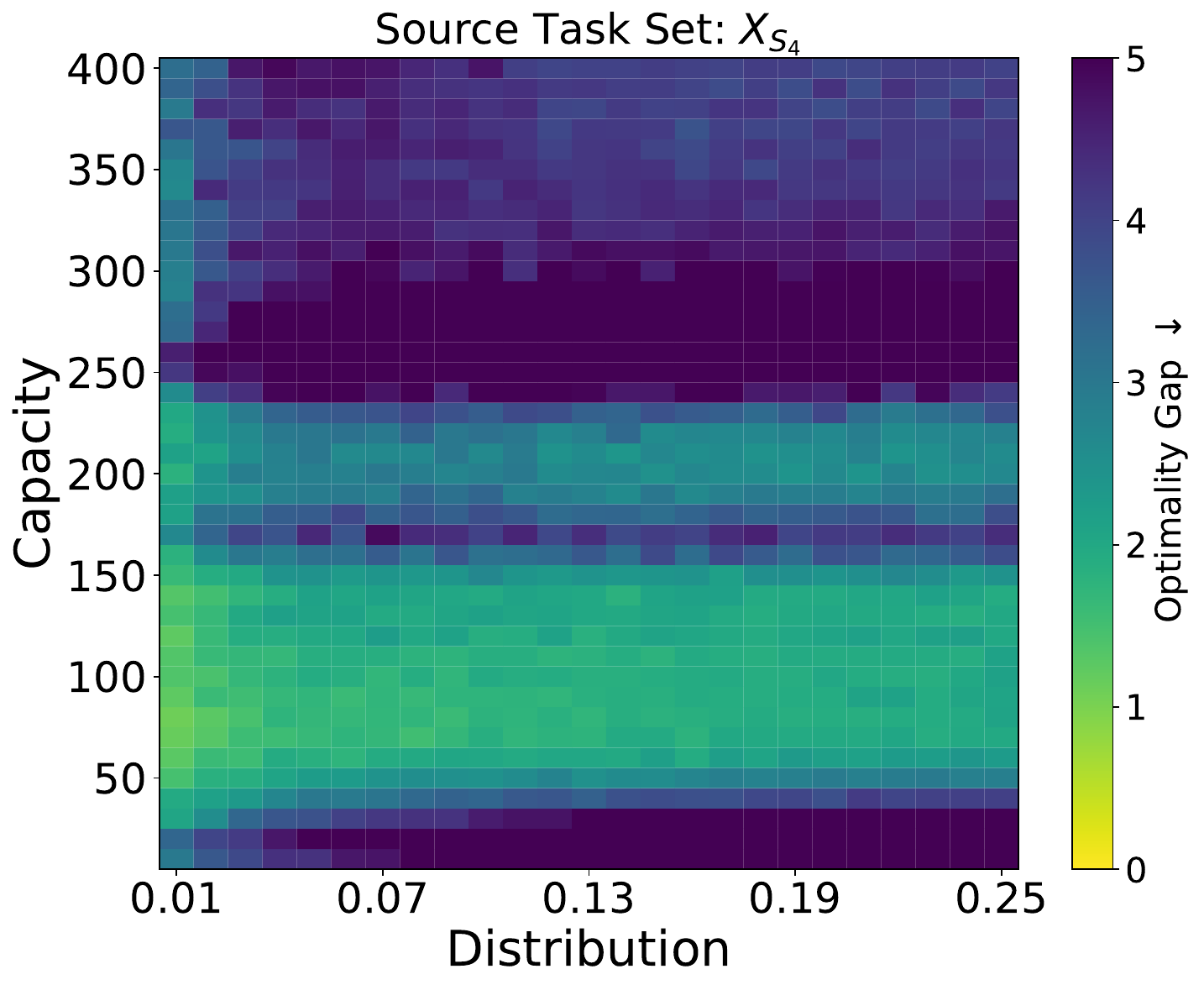}
    \includegraphics[width=0.19\columnwidth]{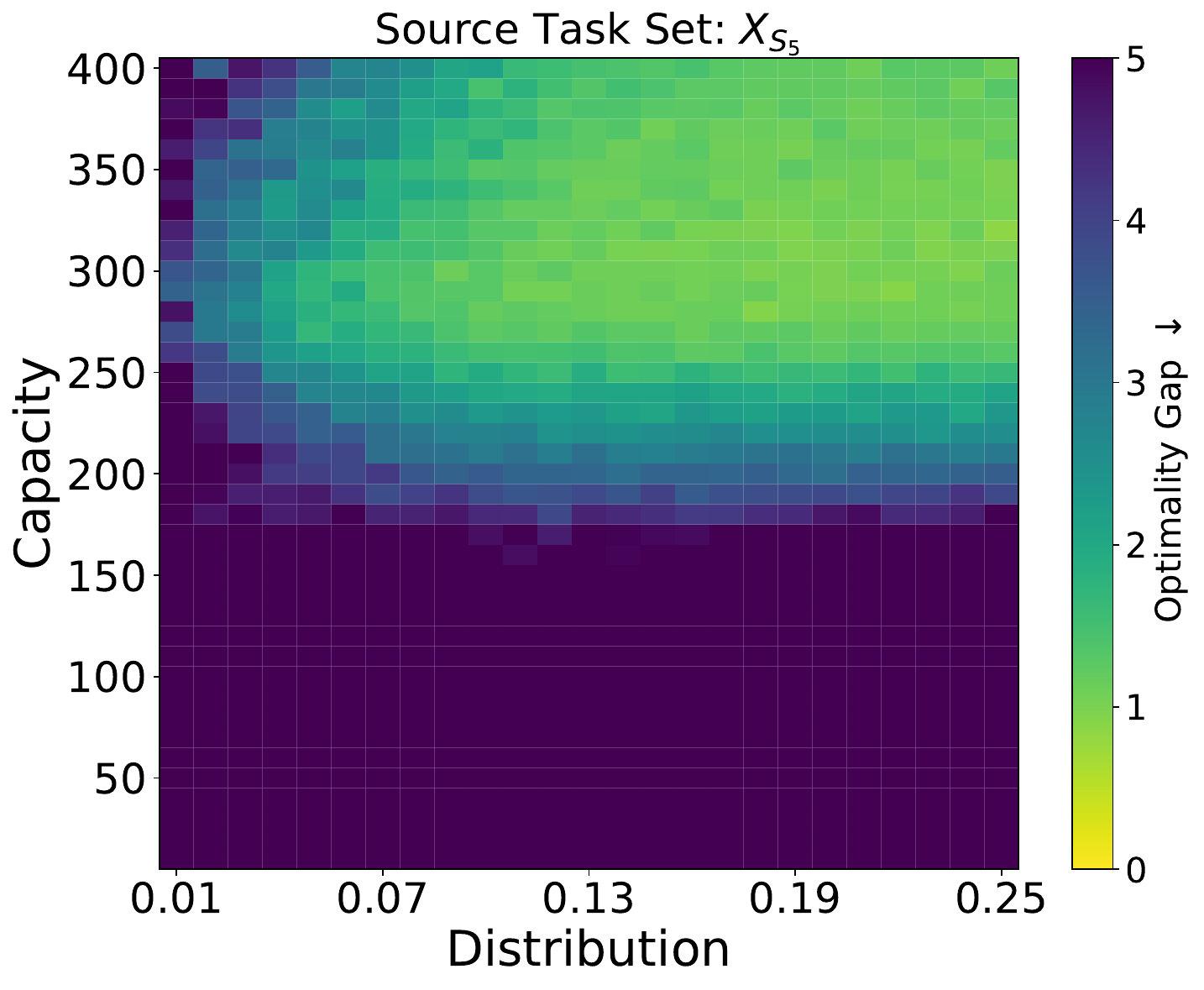} \\
    \includegraphics[width=0.19\columnwidth]{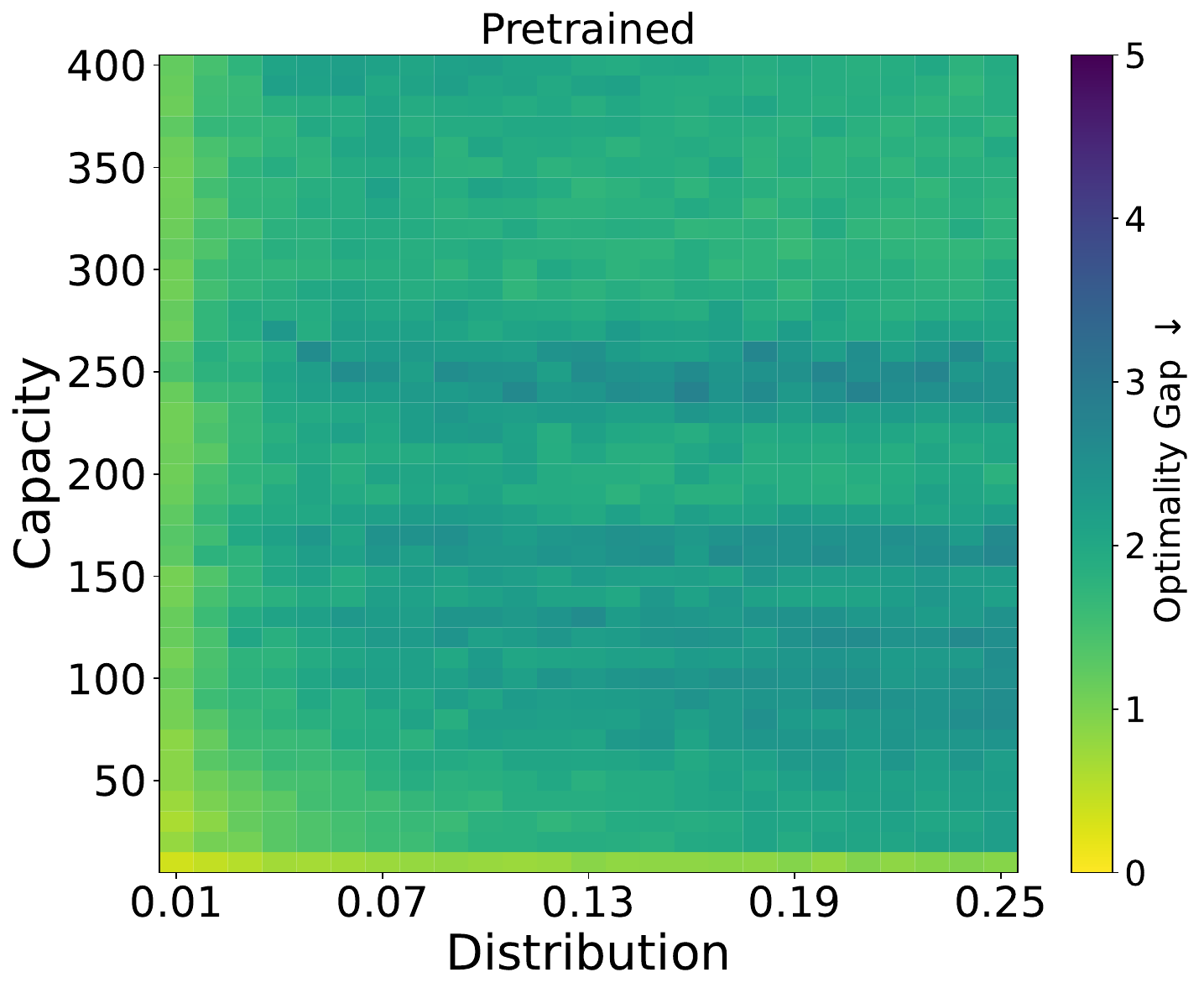}
    \includegraphics[width=0.19\columnwidth]{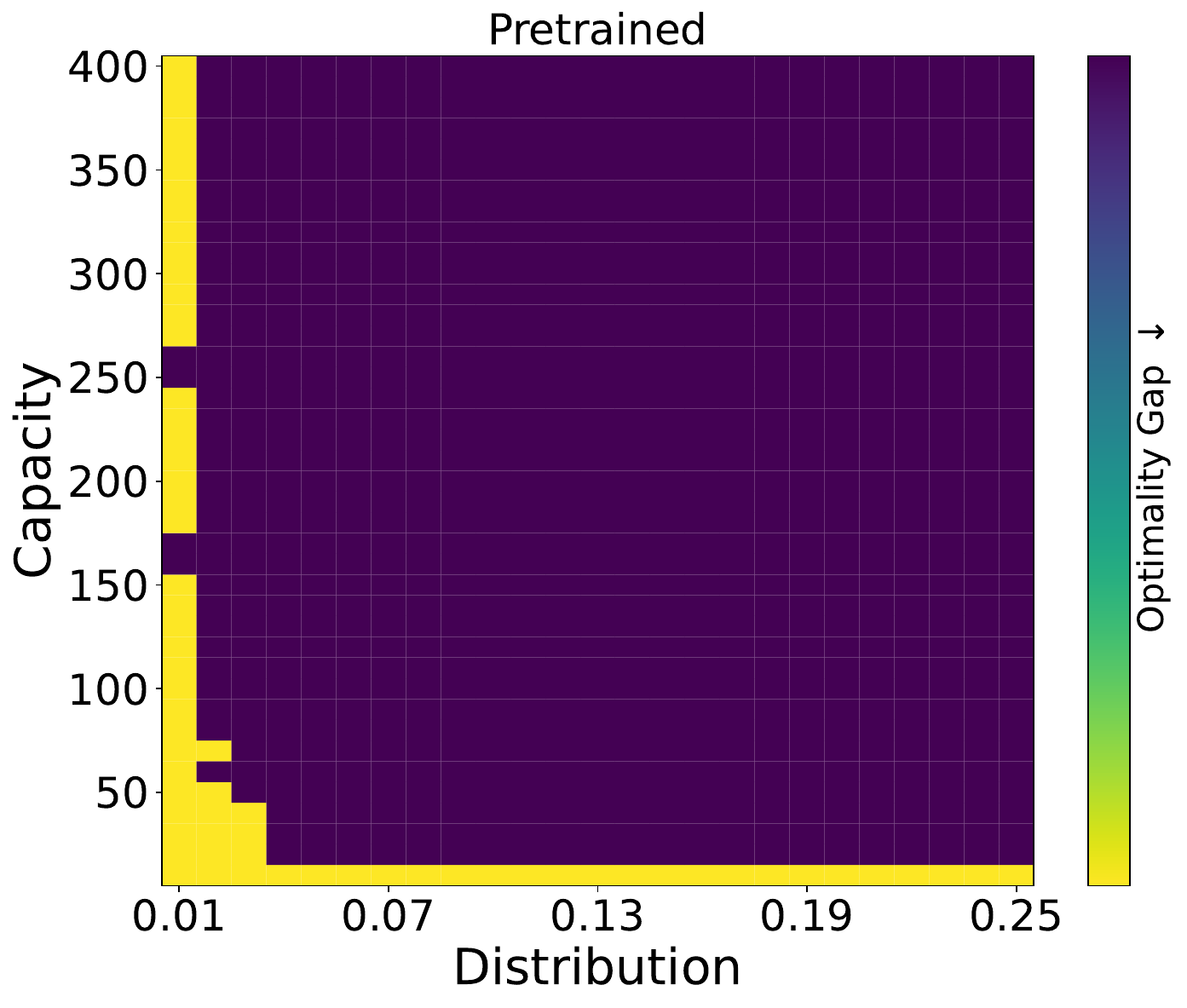}
    \includegraphics[width=0.19\columnwidth]{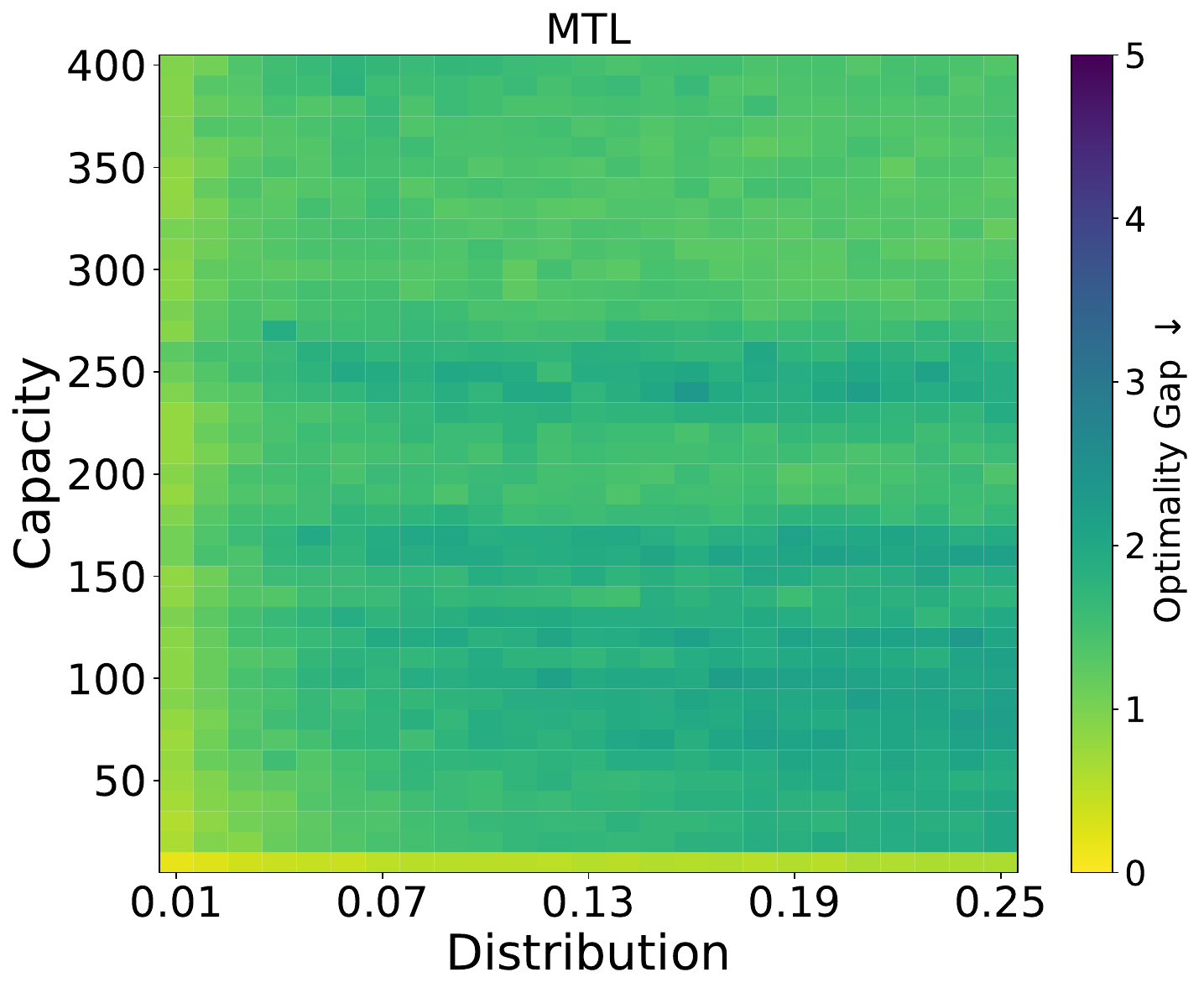}
    \includegraphics[width=0.19\columnwidth]{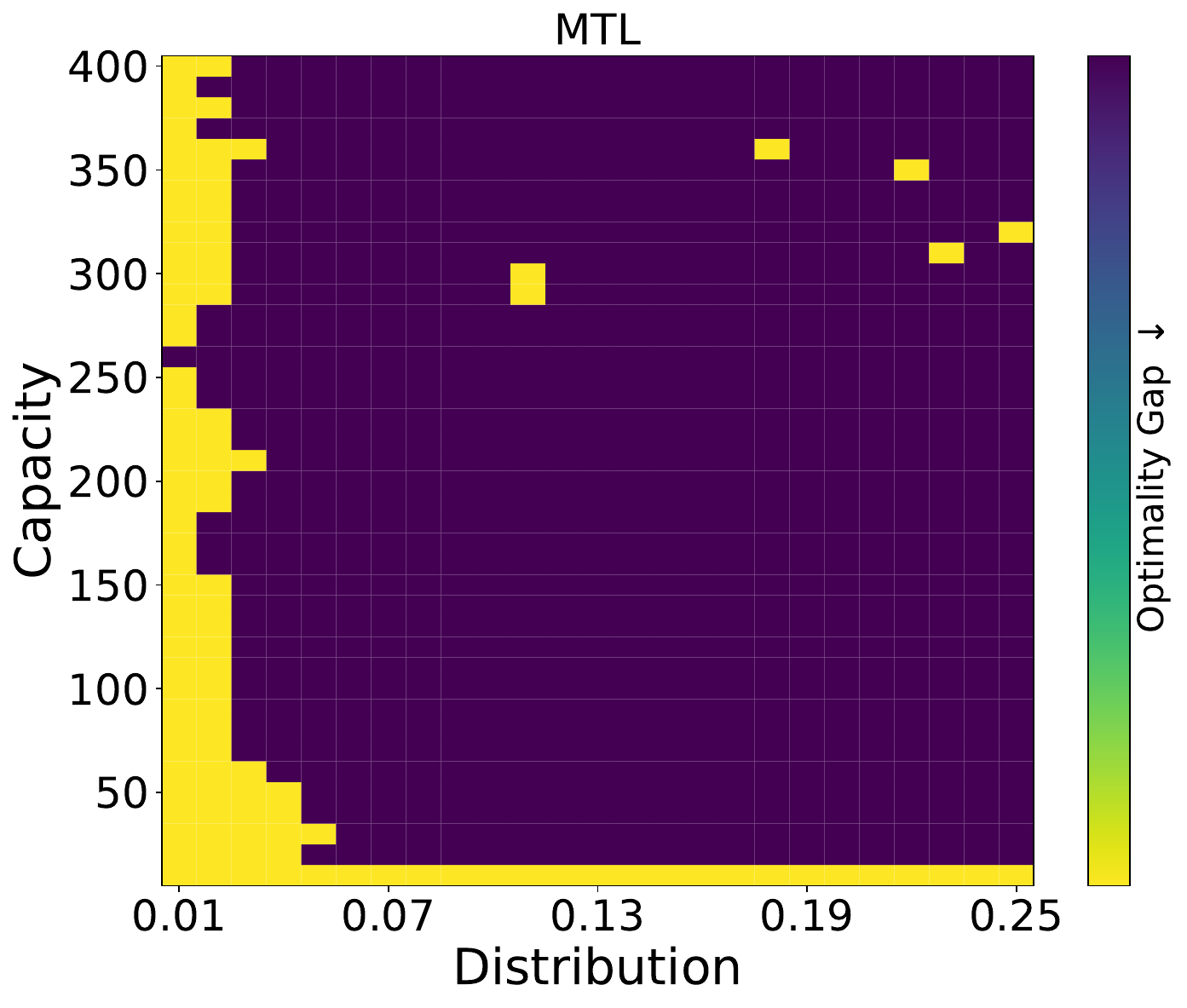}
    \includegraphics[width=0.19\columnwidth]{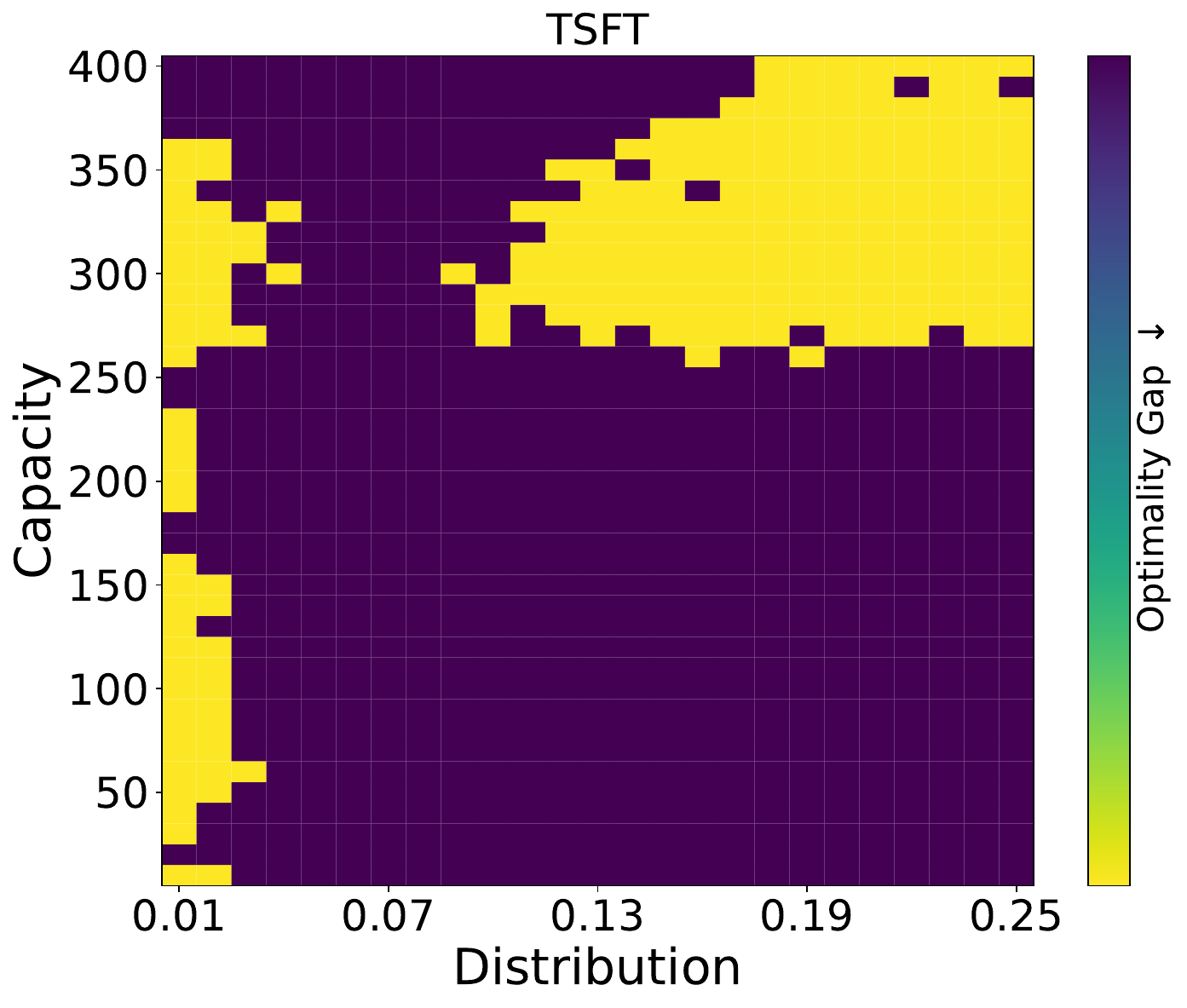}
    \caption{\emph{Top row:} Performance heatmaps of policies selected by TSFT. \emph{Bottom row:} Performance heatmaps and coverage sets (yellow regions) of the pretrained policy and the multi-task learning (MTL) policy under the same budget as TSFT, together with the final coverage set achieved by TSFT.}
    \label{app_visua}
    \vskip -0.1in
\end{figure*}

\section{Broader Impact}
\label{app_impact}
This work studies task specialization fine-tuning for contextual reinforcement learning. One positive societal impact is that it contributes directly to \emph{greener} AI practices: by improving sample efficiency, such methods may reduce the computational and energy costs required to adapt RL agents across diverse task conditions. This can make learning-based decision systems more accessible to researchers and practitioners with limited resources. In application domains such as logistics, transportation, robotics, and resource management, improved contextual generalization may also support more robust and efficient automated decision-making in changing environments. However, negative societal impacts may arise if these methods are deployed in high-stakes settings without sufficient validation, since improved sample efficiency alone does not guarantee safety, fairness, or reliability across all contexts. In addition, more efficient adaptation could accelerate the deployment of RL systems in domains where automation may affect labor, privacy, or human oversight. Therefore, practical deployment should be accompanied by careful evaluation, domain-specific safeguards, and appropriate human supervision.

\section{License}
\label{app_license}

The licenses and usage of the existing assets are listed in Table \ref{tab:asset}. Our source code and datasets will be publicly released under the MIT License upon publication.

\begin{table*}[h]
  \centering
  \caption{Licenses and usage for existing assets.}
  \begin{small}
    \begin{tabular}{c|l|l|l}
    \toprule
    \textbf{Type} & \textbf{Asset} & \textbf{License} & \textbf{Usage} \\
    \midrule
    \multirow{4}[2]{*}{Code}
     & POMO \cite{kwon2020pomo}                        & MIT License & Policy network (CVRP/CVRPTW) \\
     & Stable Baselines3 \cite{stable-baselines3}      & MIT License & PPO algorithm (CartPole/Ant) \\
     & MOORE \cite{hendawy2023multi}                   & MIT License & Meta-World multi-task RL \\
     & VeRL                                            & Apache 2.0  & LLM RFT framework \\
    \midrule
    Model
     & Qwen3-4B-Base \cite{yang2025qwen3}              & Apache 2.0  & LLM pretrained policy \\
    \midrule
    \multirow{4}[2]{*}{Datasets}
     & DAPO-17K \cite{yu2025dapo}                      & Apache 2.0  & LLM training / evaluation \\
     & MATH \cite{hendrycks2021measuring}              & MIT License & LLM training / evaluation (MATH-500) \\
     & GSM8K \cite{cobbe2021training}                  & MIT License & LLM training / evaluation \\
     & CodeContests+ \cite{wang2025codecontests}       & Apache 2.0  & LLM training / evaluation \\
    \midrule
    \multirow{7}[2]{*}{Benchmark}
     & CARL \cite{benjaminscontextualize}              & Apache 2.0  & Continuous control CRL benchmark \\
     & Meta-World \cite{yu2020meta}                    & MIT License & Multi-task robotic manipulation benchmark \\
     & AIME 2024                       & Public      & LLM evaluation \\
     & AIME 2025                       & Public      & LLM evaluation \\
     & Minerva Math \cite{lewkowycz2022solving}        & MIT License & LLM evaluation \\
     & MBPP \cite{austin2021program}                   & CC-BY-4.0   & LLM evaluation \\
     & BigCodeBench \cite{zhuo2024bigcodebench}        & Apache 2.0  & LLM evaluation \\
    \bottomrule
    \end{tabular}
    \end{small}
  \label{tab:asset}
\end{table*}



\end{document}